%% file: 0-main-arxiv.tex
\documentclass[letterpaper,11pt]{article}
\usepackage{fullpage}

\usepackage[parfill]{parskip}
\usepackage[letterpaper, left=1in, right=1in, top=1in, bottom=1in]{geometry}

\usepackage{natbib}
 \bibpunct[, ]{(}{)}{,}{a}{}{,}%
\renewcommand{\arraystretch}{1.5}
\DeclareRobustCommand{\Crefrobust}[1]{\texorpdfstring{\Cref{#1}}{}}

\usepackage{xcolor}
\usepackage{hyperref}
\usepackage{amssymb,amsthm}
\hypersetup{
	colorlinks = true,
	linkcolor=blue!70!black,
	citecolor=blue!70!black,
	linkbordercolor = {white}
}
\input{informs-macro}

\input{math}

\ifodd 0
\newcommand{\wt}[1]{{\color{blue}(WT: #1)}}
\else
\newcommand{\wt}[1]{}
\fi

\newcommand{\wtr}[1]{{\color{black}  #1}}
\newcommand{\ignore}[1]{}

\newcommand{\xhdr}[1]{\vspace{2mm} \noindent{\bf #1.}}

\newcommand{\argmax}{\mathop{\mathrm{arg\,max}}}

\title{Dynamic Pricing and Learning with Long-term Reference Effects}
\author{Shipra Agrawal\thanks{Columbia University. Email: \texttt{sa3305@columbia.edu}.} \and 
Wei Tang\thanks{Columbia University. Email: \texttt{wt2359@columbia.edu}.}}
\date{}
\begin{document}

\maketitle

\begin{abstract}
\input{algo-design/0-abstract}

\end{abstract}

\newpage
\section{Introduction}

\input{algo-design/0-intro}

\subsection{Related Work}

\input{algo-design/1-related-work}

\section{Problem Formulation} 
\label{sec:model}
\input{algo-design/2-model}

\section{Characterizing (Near-)Optimal Pricing Policy} 
\label{sec:opt}

\input{algo-design/3-full-infor-setting}

\section{Learning and optimization under demand uncertainty} 
\label{sec:algo}
\label{subsec:algo equal ref}

\input{algo-design/4-algo-equal-ref-new}


\section{Regret Analysis} 
\label{sec:regret analysis}

\input{algo-design/5-analysis-equal-ref}


\section{Conclusions and Future Directions}
\label{sec:conclusions}

\input{algo-design/7-conclusions}

\section*{Acknowledgement}
The authors would like to thank the reviewers of EC'24 for helpful comments.
This work was supported in part by NSF 2147361, NSF 2040971 and NSF CAREER 1846792.

\bibliographystyle{apalike}
\bibliography{mybib.bib}

\appendix

\newpage
\section{Missing Proofs in Section \ref{sec:opt}}
\input{algo-design/apx-proof-opt}

\section{Missing Algorithms in Section \ref{subsec:algo equal ref}}
\label{apx:missing algo}

\input{algo-design/apx-missing-algo}

\section{Missing Proofs in Section \ref{sec:regret analysis}}
\label{apx:missing proofs regret}

\input{algo-design/apx-LB-proof}

\subsection{Missing Proofs in Step 1}
\label{apx:proof for para estimation}
\input{algo-design/apx-proof-estimation}

\subsection{Missing Proofs in Step 2}

\input{algo-design/apx-proof-smoothing}

\subsection{Missing Proofs in Step 3}
\label{apx:proof lipschitz}

\input{algo-design/apx-proof-Lipschitz}

\end{document}

%% file: informs-macro.tex
\usepackage{pgfplots}
\usepackage{tikz}
\usepackage[caption=false]{subfig}
\usetikzlibrary{topaths,calc}

\usepackage{mathtools, bm, amsmath}
\usepackage{xcolor}
\usepackage[normalem]{ulem}
\usepackage{thm-restate}
\usepackage[utf8]{inputenc} 
\usepackage[T1]{fontenc}    

\usepackage{comment}
\usepackage{nicefrac}       
\usepackage{float}
\usepackage{xspace}
\usepackage{xfrac}
\usepackage{enumitem}
\usepackage{graphicx}

\usepackage{algpseudocode}
\usepackage{algorithm}
\usepackage{setspace}

\usepackage{cleveref}

\newtheorem{theorem}{Theorem}

\newtheorem{remark}{Remark}

\newtheorem{lemma}{Lemma}[section]
\newtheorem{proposition}{Proposition}

\newtheorem{corollary}{Corollary}[section]
\newtheorem{assumption}{Assumption}[section]
\newtheorem{definition}{Definition}[section]

\newcommand{\squishlist}{
\begin{list}{{{\small{$\bullet$}}}}
{\setlength{\itemsep}{3pt}      \setlength{\parsep}{1pt}
\setlength{\topsep}{1pt}       \setlength{\partopsep}{0pt}
\setlength{\leftmargin}{1em} \setlength{\labelwidth}{1em}
\setlength{\labelsep}{0.5em} } }
\newcommand{\squishend}{  \end{list}}

\newcommand{\Reg}[2][]{\text{\bf REG}\ifthenelse{\not\equal{}{#1}}{^{#1}}{}\!\left[{\def\givenn{\middle|}#2}\right]}

\newcommand{\shock}{\varepsilon}
\newcommand{\shockUB}{\bar{\shock}}

\newcommand{\rePrice}{r}

\newcommand{\baseDemandPara}{b}
\newcommand{\baseDemandParaB}{\baseDemandPara}
\newcommand{\baseDemandParaA}{a}
\newcommand{\priceLB}{\underline{p}}
\newcommand{\priceUB}{\bar{p}}

\newcommand{\basePara}{I}
\newcommand{\baseParaSpace}{\mathcal{I}}

\newcommand{\staticRev}[2][]{\cc{Rev}\ifthenelse{\not\equal{}{#1}}{_{#1}}{}\!\left({\def\givenn{\middle|}#2}\right)}
\newcommand{\realizedDemand}{D}
\newcommand{\para}{\theta}
\newcommand{\truePara}{\para^*}
\newcommand{\paraSpace}{\Theta}
\newcommand{\paraCOne}{C_1}
\newcommand{\paraCTwo}{C_2}
\newcommand{\paraEst}{\hat{\para}}

\newcommand{\trueCOne}{C_1^*}
\newcommand{\trueCTwo}{C_2^*}

\newcommand{\estCOne}{\hat{C}_1}
\newcommand{\estCTwo}{\hat{C}_2}

\newcommand{\estCOneErr}{\varepsilon_1}
\newcommand{\estCTwoErr}{\varepsilon_2}
\newcommand{\maxEstError}{\bar{\varepsilon}}

\newcommand{\zeroVec}{\boldsymbol{0}}

\newcommand{\identity}{\boldsymbol{I}}

\newcommand{\baseDemandParaAUB}{\bar{\baseDemandParaA}}
\newcommand{\baseDemandParaALB}{\baseDemandParaA_{\min}}
\newcommand{\baseDemandParaBUB}{\bar{\baseDemandParaB}}

\newcommand{\maxRef}{\bar{\eta}}

\newcommand{\revUB}{\bar{R}}

\newcommand{\Hessian}{\boldsymbol{H}}
\newcommand{\eigenVal}{\lambda}
\newcommand{\Jacobian}{\boldsymbol{J}}
\newcommand{\solvedP}{p^\dagger}

\newcommand{\fixed}{\text{fixed}}

\newcommand{\GR}{\cc{GR}}

\newcommand{\estGreedy}{\hat{p}^{\GR}}
\newcommand{\randomDirec}{\kappa}
\newcommand{\distance}{d}
\newcommand{\gradient}{g}
\newcommand{\Proj}{\cc{PROJ}}

\newcommand{\lipConst}{L}
\newcommand{\lipConstEqualRef}{\lipConst}
\newcommand{\adjustedGrad}{g}
\newcommand{\adjustedGradUB}{G}
\newcommand{\newadjustedGradUB}{\bar{\adjustedGradUB}}
\newcommand{\chosenPrice}{\widetilde{p}}

\newcommand{\gradientDiff}{\delta^{(\gradient)}}

\newcommand{\ESM}{\cc{ESM}}
\newcommand{\ARM}{\cc{ARM}}
\newcommand{\ESMPara}{\zeta}

\newcommand{\approxMD}{\tilde{\pvec}}

%% file: math.tex
\newcommand{\omt}[1]{}

\newcommand{\cc}[1]{\ensuremath{\mathsf{#1}}} 

\usepackage{stackengine}

\newcommand{\R}{\mathbb{R}}
\newcommand{\N}{\mathbb{N}}

\newcommand{\prob}[2][]{\text{Pr}\ifthenelse{\not\equal{}{#1}}{_{#1}}{}\!\left[{\def\givenn{\middle|}#2}\right]}
\newcommand{\expect}[2][]{\mathbb{E}\ifthenelse{\not\equal{}{#1}}{_{#1}}{}\!\left[{\def\givenn{\middle|}#2}\right]}
\newcommand{\indicator}[2][]{\mathbf{1}\ifthenelse{\not\equal{}{#1}}{_{#1}}{}\!\left\{{\def\givenn{\middle|}#2}\right\}}
\newcommand{\Variance}[2][]{\text{Var}\ifthenelse{\not\equal{}{#1}}{_{#1}}{}\!\left[{\def\givenn{\middle|}#2}\right]}

\newcommand{\newexpect}[3]{
    \mathbb{E}
    \ifthenelse{\equal{}{#1}}{}{^{#1}}
    \ifthenelse{\equal{}{#2}}{}{_{#2}}
    \!\left[{#3}\right]
    }

\newcommand{\bvec}{\mathbf{b}}
\newcommand{\pvec}{\mathbf{p}}
\newcommand{\Amtrx}{\bm{A}}

\newcommand{\demand}{D}
\newcommand{\baseDemand}{H}

\newcommand{\positiveRef}{\eta^+}
\newcommand{\negativeRef}{\eta^-}

%% file: algo-design/0-abstract.tex
We consider a dynamic pricing problem where customer response to the current price is impacted by the customer price expectation, aka reference price. We study a simple and novel reference price mechanism where reference price is the average of the past prices offered by the seller. As opposed to the more commonly studied exponential smoothing mechanism, in our reference price mechanism the prices offered by seller have a longer term effect on the future customer expectations. 

We show that under this mechanism, a markdown policy is near-optimal irrespective of the parameters of the model. This matches the common intuition that a seller may be better off by starting with a higher price and then decreasing it, as the customers feel like they are getting bargains on items that are ordinarily more expensive. For linear demand models, we also provide a detailed characterization of the near-optimal markdown policy along with an efficient way of computing it.

We then consider a more challenging dynamic pricing and learning problem, where the demand model parameters are apriori unknown, and the seller needs to learn them online from the customers' responses to the offered prices while simultaneously optimizing revenue.  
The objective is to minimize regret, i.e., the $T$-round revenue loss compared to a clairvoyant optimal policy. This task essentially amounts to learning a non-stationary optimal policy in a time-variant Markov Decision Process (MDP). For linear demand models, we provide an efficient learning algorithm with an optimal
$\widetilde{O}(\sqrt{T})$ regret upper bound.\footnote{A one-page abstract has been accepted at the 25th ACM Conference on Economics
and Computation (EC'24).}

%% file: algo-design/0-intro.tex
In modern marketplaces, 
the demand for a product is influenced  not only
by the current selling price but also by customers' price expectation,
aka reference price. Intuitively, the reference price is what customers perceive as the ``normal" price for a product based on the historic market prices. 
When the current price is lower than the reference price, customers are likely to 
perceive a bargain or ``gain", which could lead to an increased purchase. And 
when the selling price is significantly higher than the reference price, customers are likely to perceive a rip-off or ``loss",  potentially leading to a negative effect on demand.
The reference effects are often asymmetric, owing to customers' different attitudes towards a perceived loss v.s.\ a perceived gain \citep{LB-89,RT-94}.

In  the existing literature, 
a common way to model the reference price mechanism in markets with 
repeated-purchases is to consider reference prices formed {\em endogenously} based on the past prices of a product (see, e.g., \citealp{FGL-03,PW-07,NP-11,CHH-17,dK-21}). 
The motivation for this endogenous mechanism stems from the observation that the returning customers in these markets may remember the historical prices of 
the product. 

A commonly used model for endogenously formed reference prices
is the exponential smoothing mechanism (henceforth \ESM). In \ESM,  
customers' reference prices are formed as a weighted average that puts exponentially decreasing weights over time on past prices \citep{MRS-05}
in order to reflect the psychological intuition that customers have a fast-diminishing memory over the past prices.
In doing so, \ESM\ essentially captures a \emph{short-term reference effect} dominated by the recent prices. The short-term nature of reference price effects in the \ESM\ is also evident from the observation made by the previous works that in this setting, a fixed-price policy is near-optimal.
More precisely, the revenue on repeatedly offering the optimal fixed-price is within a constant (independent of the sales horizon) of the optimal revenue  assuming loss-averse customers\footnote
{Formal definitions of ``loss-averse" vs ``gain-seeking" customers are provided in Section \ref{sec:model}.} and linear base demand model \citep{FGL-03,PW-07,dK-21}. Intuitively, the result follows from the insight that due to exponential weighting in \ESM, only a constant number of past prices significantly influence the current reference price and thereby the customer response. 

However, the proliferation of online platforms and marketplaces has led to a paradigm shift in customer expectation and response behavior. The reference effects are becoming more pronounced and longer term, since the information about historic prices is widely and easily available.  It is common for platforms like 
Google Hotel to put labels 
``Great Deal, 28\% less than usual'' or ``Deal, 16\% less than usual'' 
to 
signal how good the current hotel price is compared to the average price of this hotel over the past year.\footnote{See 
\href{https://support.google.com/travel/answer/6276008?hl=en\#zippy=\%2Chotel-prices\%2Chotel-booking-links\%2Chotel-deals}{Google hotel deals policy}.
}
Certain products in Amazon have a ``Was Price'' displayed, which is determined using recent price history of the product on Amazon. A number of third party websites, like CamelCamelCamel, PriSync, Wiser Solutions, Price2Spy,  are dedicated to price tracking and monitoring, and are increasingly popular among customers for comparing product prices to historic prices before making a purchase.
Thus the customers no longer need to rely on their individual memory in order to look back at historic prices over a week or a month or even year(s). The need for studying new reference price mechanisms with longer memory is therefore apparent, which was interestingly also mentioned as an important avenue for future research in a recent work on \ESM~\citep{dK-21}.


Motivated by these observations, in this work, we consider a novel and simple reference price mechanism aimed at capturing longer-term reference effects. 
In our mechanism that we refer to as {\em averaging reference mechanism} (henceforce \ARM), the reference price is formed as a simple average (i.e., with equal weights) of the past prices. See formal definition in \Cref{eq:ref dynamics}. 

We demonstrate that this simple modification to the reference price mechanism significantly changes the nature of the dynamic pricing problem. Specifically, as opposed to \ESM, the revenue of any fixed-price policy 
under \ARM~can be highly suboptimal compared to the optimal revenue: 
the difference can be as large as 
linear in sales horizon (see \Cref{prop:linear regret fixed}). Furthermore, we show that a markdown policy is always near-optimal under \ARM. This matches the practical intuition that under significant and long-term reference effects, a seller may be able to set customer expectations better by starting with a higher price and then decreasing it in order to give a perception of a bargain to the customer (or, avoid the perception of a rip-off). This result also provides a technical justification of the widespread use of markdown pricing in industry (see, e.g., \citealp{IBM,Google}) and in academic literature \citep{L-86, WLZ-15, JLR-21, BCK-23}.

Next, we summarize our contributions. Besides modeling contributions, our results include a detailed characterization of the optimal dynamic pricing policy under the \ARM\ model, and algorithm design and analysis for learning and regret minimization when the demand model parameters are unknown.

\subsection{Our Contributions and Results}
A main {\it modeling contribution} of our work is to introduce and pioneer the study of the average reference model (i.e., \ARM) in dynamic pricing and learning literature which has thus far been dominated by the exponentially weighted average (i.e., \ESM) model for reference effects. We contend that \ARM\ is able to capture longer-term reference effects on customer behavior and matches the practice much more closely in its implications on the seller's optimal pricing strategies. We provide several rigorous technical results to support this claim. Further, for \ARM\ with linear demand models, we provide algorithmic and computational insights towards efficiently implementable near-optimal dynamic pricing strategies, in both full information and learning settings.

Our main {\it technical contributions} are summarized as follows.

\xhdr{Suboptimality of fixed-price and (near-)optimality of markdown pricing}
In our first result that sets the \ARM\ model distinctively apart from \ESM, we demonstrate that 
under \ARM, the revenue of any fixed-price policy 
can be highly sub-optimal compared to that of the optimal pricing policy. Specifically, in \Cref{prop:linear regret fixed}, we show that the difference between these two revenues can grow linearly in sales horizon $T$. 

Next, in Theorem \ref{prop:approx opt of markdown}, we establish the (near-)optimality of markdown pricing in \ARM\ models. Here, markdown pricing policy  is defined as any policy under which the sequence of prices offered are always decreasing (or non-increasing) over time. In Theorem \ref{gain-seeking}, we show that 
under \ARM~with gain-seeking customers, markdown pricing is optimal. And further, in Theorem \ref{risk-averse}, we show that for loss-averse customers, markdown pricing is near-optimal:  specifically, there always exists a markdown policy that achieves a revenue within $O(\log(T))$ of the optimal revenue in $T$ rounds.

Our results complement the existing literature supporting markdown strategies in dynamic pricing problems
\citep{L-86, WLZ-15, JLR-21, BCK-23,JLR-22}, and provide a technical justification for the widespread practice of such strategies \citep{Google, IBM}.

\xhdr{Characterization and implementation of near-optimal pricing for linear demand models}
Having established the near-optimality of 
the markdown pricing policy, we then explore 
how to characterize and compute such a desired markdown pricing policy.
Notice that 
reference price effects make the dynamic pricing problem substantially more involved. 
In particular, as we elaborate later, with \ARM, under any pricing policy, demand evolves as a non-stationary stochastic process. Finding an optimal pricing policy then amounts to solving a time-variant Markov decision process (MDP) with continuous action/state space. 
Moreover, the reward function 
(i.e., the seller's revenue function) is nonsmooth under asymmetric reference price effects. 
For a more tractable setting, 
we follow the existing literature \citep{FGL-03,dK-21,JYS-23} 
and consider a linear base demand model. 
Note that in this model, while the demand is linear in the offered price, there still may be some non-linearity in demand stemming from the asymmetry in reference price effects. For details, see Section \ref{sec:model}.

Under the linear base demand model, firstly, in \Cref{prop:approx markdown structure}, 
we provide a detailed structural characterizations for a (nearly) optimal markdown pricing policy for both
gain-seeking and loss-averse customers. Next, in \Cref{cor:computational complexity approx opt}, we utilize this markdown structure to provide an efficient algorithm for computing such
policies. 




\xhdr{Learning and regret minimization under linear demand models}
We also study the dynamic pricing problem with reference effects in the presence of demand model uncertainty. This problem is motivated by the practical consideration that the demand model parameters may not be fully known to the seller {\em a priori}. Instead, the seller may have to learn these (implicitly or explicitly) from the observed customer response through sales data, in order to compute/improve the pricing policy.

Following the previous literature, again we consider a linear base demand model. 
We assume that seller has no apriori knowledge of the base demand model parameters or the parameters of the \ARM~model. In line with the online learning and multi-armed bandits literature, 
we focus on evaluating the algorithm performance via the {\em cumulative regret}, defined as the total expected revenue loss compared to a clairvoyant policy. 

Our main contribution in this setting is a learning algorithm that achieves a regret upper bound of $\widetilde{O}(\priceUB^{3.5}\sqrt{T})$ (see \Cref{thm:regret upper bound}). A main insight that we derive and utilize for this algorithm is that although under \ARM\ the near-optimal markdown policies are highly non-stationary, they can be parameterized by a two-dimensional parameter vector which also fully characterizes the local greedy price (namely the price that maximizes a single round revenue)  under a certain range of reference prices. This insight allows us to offer estimated greedy prices in the learning phase and use the demand response from those prices to learn the desired markdown policy. 

Finally, in  \Cref{thm:LB}, we derive a regret lower bound that closely matches our upper bound. In particular, building upon the proof of the existing lower bounds for dynamic pricing problem, we show that even when there are no reference effects, there exists a set of instances satisfying our assumptions such that any learning algorithm must incur a worst-case (over this set) expected  regret of at least $\Omega(\priceUB \sqrt{T})$.

%% file: algo-design/1-related-work.tex
On the research related to dynamic pricing,
our work mainly connects to two streams of works: 
(i) dynamic pricing with reference effects, 
(ii) dynamic pricing and learning (with reference effects).
As we elaborate later, the seller's learning problem is essentially that of learning in a time-variant MDP. Thus, our work also relates to 
works on reinforcement learning for non-stationary MDP.

\xhdr{Dynamic pricing with reference effect}
Dynamic pricing with reference effects has been extensively studied in both revenue management and marketing  literature over the past decades (see \citealp{MRS-05} and \citealp{AK-10} for a review).
Below we mention mostly related works.
As previously mentioned, \ESM\ is one commonly-used model for reference price in the current literature.  
Under this mechanism, 
\citet{PW-07} show that optimal pricing policy will eventually converge to a constant price when customers are loss averse, similar message is also delivered in \citet{dK-21}.
In addition, \citet{CHH-17} develop strongly polynomial time algorithms to compute the optimal prices under certain conditions.
Other works that focus on \ESM\ include \citet{FGL-03,WLZ-15,HCH-16,CZW-19,GJL-20}, just to name a few.
As mentioned, two key features of \ESM\ is that it is a stationary process and customers have diminishing memories over the past prices. 
These two features also appear in other reference price mechanism like peak-end anchoring \citep{NP-11}.
Our work differs with these works significantly as 
the considered \ARM\ is a different reference price model
that aims to capture the long-term reference effects. 

\xhdr{Dynamic pricing and learning (without and with reference effect)}
The learning part of our work also relates to the expanding literature on learning and earning in dynamic pricing problems, see, e.g., 
\citet{KL-03,BZ-09,AC-09,BR-12,KZ-14,BZ-15,AFT-23,CSWZ-22,CWZ-24}, 
just to name a few. 
We refer the interested readers to the survey by 
\citet{D-15}.
These works are focusing on learning without reference effects. 
The mostly related works are \citet{dK-21,JYS-23},
similar to our work, they also consider learning with reference effects. 
In particular, \citet{dK-21} study 
dynamic pricing with learning with 
linear demand model under \ESM, 
and \citet{JYS-23} focus on multi-product selling setting where the reference effects are formed by the comparison between the prices of all products. 
Notice that in these two works, a fixed-price policy is already near-optimal. 
Thus, the high-level idea behind the algorithms in both works is an epoch-based learning-and-earning algorithm that integrates
least squares estimation to estimate the underlying demand parameters and compute a good estimate of the optimal fixed price. 
Our work differs from these two works significantly as in our problem, the seller needs to learn an optimal non-stationary policy, making their algorithms not applicable to our setting. 


\xhdr{Learning in non-stationary MDP}
The underlying demand function with our
reference price mechanism essentially evolves according to a time-variant MDP.
There are also works studying dynamic pricing and learning in a non-stationary environment, e.g., 
\citet{BZ-11,D-15b,CWX-19,KZ-17,AYZ-21}.
However, these works are either consider a very broad non-stationarity or focusing a particular structured model which is fundamentally different from ours. 
There is also much recent work on learning and regret minimization 
in stateful models using
general MDP and reinforcement learning frameworks (e.g., 
\citealp{AJO-08,AJ-17,CSZ-20,FYWX-20}). 
However, these results typically rely on an assumption that
each state can be visited many times over the learning process. 
This is ensured either because of an episodic MDP setting (e.g., \citealp{FYWX-20}), 
or through
an assumption of communicating or weakly communicating MDP with bounded diameter, i.e., a
bound on the number of steps to visit any state from any starting state under the optimal policy
(e.g., \citealp{CSZ-20,AJ-17}). 
Our setting is non-episodic, and under \ARM, 
the state (i.e., the reference price) can take steps that linear in current time to reach other state. 
Therefore, the results in the above papers on learning general MDPs are not applicable to our learning problem.


%% file: algo-design/2-model.tex
\newcommand{\randomVar}{\mathcal{U}}
\newcommand{\val}{V}
\newcommand{\optVal}{\val^*}
\newcommand{\Qfunction}{Q}
\newcommand{\policy}{\pvec}
\newcommand{\optPolicy}{\policy^*}
\newcommand{\optPrice}{p^*}
\newcommand{\rePriceFn}{q}

\newcommand{\strongConcavPara}{\lambda}
\newcommand{\filtering}{\mathcal{F}}

\newcommand{\pricingAlgo}{\pi_{\cc{ALG}}}

We consider a dynamic pricing problem 
with an unlimited supply of durable products
of a single type to be sold in a market
across $T$ time periods.
At the beginning of each time round $t\in[T]$,
the seller sets a price $p_t \in [\priceLB,\priceUB]$ for her product, 
where $\priceLB$ and $\priceUB$ are the 
pre-determined 
smallest and largest feasible prices satisfying $0 \le \priceLB < \priceUB < \infty$. 
Given the price $p_t$, the seller observes 
the demand $\realizedDemand_t$ and 
collects revenue $p_t\realizedDemand_t$. 
The observed demand $\realizedDemand_t$
is influenced by: 
(i) the current selling price $p_t$, 
(ii) a reference price $r_t$, 
and (iii) the unobservable demand shocks in the following manner:
\begin{align*}
    \realizedDemand_t = \demand(p_t,\rePrice_t) + \shock_t~.
\end{align*}
Here, $\rePrice_t\in[\priceLB, \priceUB]$ denotes the reference price at the beginning of round $t$; 
$\shock_t \le \shockUB, t\in [T]$  
denote unobservable demand shocks which are independently and identically distributed 
random variables with zero mean, 
and are all upper bounded by  $\shockUB \ge 0$. 
And, the demand function $\demand(\cdot,\cdot)$ is given by \footnote{Here operation $(x)^+$ denotes $\max\{x, 0\}$.}
\begin{align}
    \label{eq:demand}
    \demand(p_t,\rePrice_t) = 
    \baseDemand(p_t)
    + \positiveRef \cdot (\rePrice_t - p_t)^+ - \negativeRef \cdot (p_t -\rePrice_t)^+  ~,
\end{align}
where $\baseDemand(p)$ represents the non-decreasing
base demand in price with the absence of reference effects,
parameters $\positiveRef \ge 0$ and $\negativeRef \ge 0$
control the impact of reference price on demand: 
if the current selling price $p$ is less than the reference price $\rePrice$ then the expected demand increases by $\positiveRef(\rePrice - p)$, 
but if $p$ exceeds $\rePrice$ then the expected demand decreases by $\negativeRef (p - \rePrice)$. 
The aggregate-level demands are classified as loss averse, loss/gain neutral, or gain seeking depending on whether $\positiveRef < \negativeRef$, 
$\positiveRef = \negativeRef$, 
or $\positiveRef > \negativeRef$, respectively. 
For presentation simplicity, we assume $\priceLB \equiv 0$.
\wtr{
Motivated by the pricing application that we consider, we also assume the non-negative demand:\footnote{The assumption on demand non-negativity is commonly made in dynamic pricing literature, see, e.g., \cite{BZ-15,dZ-14,dK-21,CSWZ-22,CWZ-24}.}
\begin{assumption}
\label{assump:demand non-negativity}
To avoid negative demand, we assume that 
$\demand(p, \rePrice) \ge 0$ for all $p, \rePrice\in[0, \priceUB]$.
\end{assumption}
}

\xhdr{The reference-price dynamics}
In this paper, we consider the following dynamics of evolution of the reference-price  $(\rePrice_t)_{t\in[T]}$.
\begin{definition}[\ARM]
\label{eq:ref dynamics}
A reference-price dynamic is said to be an averaging reference mechanism (\ARM) if 
\begin{align}
    \rePrice_{t+1} 
    = \frac{\rePrice_1 + \sum_{s=1}^t p_t}{t+1}, \quad 
    \text{for } t \in [T-1]~ ,
\end{align}
where $\rePrice_1$ is the starting reference price at $t=1$. 
\end{definition}
Another equivalent way to describe \ARM\ is through the running average: given reference price $r_t$ and selling price $p_t$ at time $t$, the next reference price $r_{t+1}$ is given by:
\begin{align}
    \label{eq:ref dynamics via iterative}
    \rePrice_{t+1} 
    = 
    \frac{t\rePrice_t}{t+1} + \frac{p_t}{t+1}
    = \ESMPara_{t} \rePrice_t + (1 - \ESMPara_{t}) p_t, \quad 
    \text{for } t \in [T-1]~,
\end{align}
where $1-\ESMPara_t = \sfrac{1}{(t+1)}$ for all $t\in[T-1]$ can be seen as the {\em averaging factor} that captures how 
the current price $p_t$ affects the reference price in next round. Intuitively, in this mechanism, once a product has been around for a long time, the customer price expectations may not be impacted much by a few rounds of price variations.
\begin{remark}
\label{rmk:non-stationary}
We emphasize that our \ARM\
is a non-stationary mechanism as the averaging factor $\ESMPara_t$
is time-dependent. 
This stands contrast to \ESM\ where
$\rePrice_{t+1} 
= \ESMPara\rePrice_t + (1 - \ESMPara) p_t$
where the averaging factor $\ESMPara$ 
is assumed to be a constant across the whole time horizon. 
\end{remark}
Given a price $p\in [0, \priceUB]$, a reference price $\rePrice\in [0, \priceUB]$, \wtr{and a demand function $\demand(p, \rePrice)$ defined in \eqref{eq:demand}}, we denote the expected single-round revenue by function $\staticRev{p, \rePrice}$:
\begin{align*}
    \staticRev{p, \rePrice}
    = p \cdot \demand(p,\rePrice)
    = p\cdot \left(
    \baseDemand(p)
    - \negativeRef (p-\rePrice)^+ 
    + \positiveRef (\rePrice - p)^+ 
    \right).
\end{align*}
Given a starting reference price $\rePrice$ at time $t$, 
and the sales horizon $T$, 
the optimal expected cumulative revenue
starting from time $t$
is given by 
\begin{align}
    \label{eq:opt}
    \tag{$\mathcal{P}_{\cc{OPT}}$}
    \optVal(\rePrice, t)
    & \triangleq 
    \sup_{p_t, \ldots, p_T}\sum_{s = t}^{T} 
    \staticRev{p_s, \rePrice_s}, \quad
    \text{s.t. }~ \rePrice_{s+1} = \frac{s\rePrice_s  + p_s}{s+1}, 
    ~~~ s\in[t, T-1]~.
\end{align}
We use $\val^{\policy}(\rePrice, t)$ to denote the seller's revenue on applying a price sequence 
 $\policy = (p_s)_{s\ge t}$,
starting with a reference price $\rePrice$ at time $t$.
For notation simplicity, when starting time $t=1$, we use the notation 
$\optVal(\rePrice) \equiv \optVal(\rePrice, 1), \val^{\policy}(\rePrice) \equiv \val^{\policy}(\rePrice, 1)$.
We use $\optPolicy(\rePrice, t) \triangleq (\optPrice_s)_{s\ge t}$ to denote an
optimal pricing policy that maximizes $\val^{\policy}(\rePrice, t)$. 

\xhdr{Dynamic pricing and learning under unknown model parameters}
We also consider a partial-information setting where demand function
is initially unknown to the seller and has to be inferred 
from the observed demand response to offered prices. 
In this part, to make the problem more tractable, we focus on linear base demand models, i.e., $\baseDemand(p) \triangleq \baseDemandParaB - \baseDemandParaA p$.
Here $\baseDemandParaA \in \R$ and $\baseDemandParaB\in \R$ are the base 
demand model parameters which captures the customers' price-sensitivity and the market size, respectively. 

Then, in the learning setting, we assume that the base demand parameters $(\baseDemandParaA, \baseDemandParaB)$ 
and the reference effect parameters $(\positiveRef, \negativeRef)$,
are all apriori unknown to the seller. At any time $t$, given reference price $r_t$, on offering a price $p_t$, the seller can observe a noisy outcome from demand response function $D(p_t,r_t)=b-ap_t
+ \positiveRef (\rePrice_t - p_t)^+ - \negativeRef (p_t -\rePrice_t)^+$, which may be used to infer the demand model and thereby the optimal pricing policy. To aid this inference, we make the following assumption on the base demand model $H(p)$; similar assumptions are commonly made in the existing literature on dynamic pricing with linear demand models (see, e.g., \citealp{KZ-14, dK-21,JYS-23,dZ-14}).
\begin{assumption}
\label{assump:maximizer is in feasible set}
Assume that 
the maximizer  of the function $p\baseDemand(p)$ with linear demand model $H(p)=b-ap$, lies in the interior of the feasible price range $[0, \priceUB]$. In particular, assume $\frac{\baseDemandParaB}{2\baseDemandParaA} \le \priceUB -  \Omega(1)$.\footnote{
\wtr{We would like to note that this assumption is not contradicting to \Cref{assump:demand non-negativity}. In fact, there exists a rich range of model parameters that satisfy both assumptions, e.g., under linear demand, any $\baseDemandParaB\in[(\baseDemandParaA+\negativeRef)\priceUB, 2\baseDemandParaA\priceUB)$ with $\negativeRef<\baseDemandParaA$ would satisfy these commonly-made assumptions.}
}
\end{assumption}
We use vector $\basePara = (\baseDemandParaA, \baseDemandParaB, \positiveRef, \negativeRef)$ to denote the parameters of a given problem instance. We assume that all parameters are normalized so that $\basePara \in [0, 1]^4$. 
We use $\baseParaSpace \subseteq [0, 1]^4$ to denote the parameter set of all feasible instances satisfying 
\Cref{assump:maximizer is in feasible set}.
An online dynamic pricing algorithm $\pricingAlgo$ maps all the historical information collected up to this time round into a price in $[0, \priceUB]$.
Given a starting  reference price $\rePrice_1\in[0, \priceUB]$, 
we measure the performance of a pricing algorithm 
$\pricingAlgo$ via the $T$-period cumulative regret, or regret, which equals to the total expected revenue loss incurred by implementing the algorithm
$\pricingAlgo$ instead of a clairvoyant optimal pricing policy for the problem instance with parameter $\basePara$, i.e.,
\begin{align*}
    \Reg[\pricingAlgo]{T, \basePara} \triangleq
    \optVal(\rePrice_1)  - 
    \expect[\policy\sim \pricingAlgo]{\val^{\policy}(\rePrice_1)}~,
\end{align*}
where the expectation $\expect[\policy\sim\pricingAlgo]{\cdot}$ is over the sequence of prices
 induced by the algorithm $\pricingAlgo$ when run starting from reference price $r_1$ on a problem instance with parameter $\basePara$. 



%% file: algo-design/3-full-infor-setting.tex
In this section, we  derive our main results on the structure and computation of optimal and near-optimal dynamic pricing policies under \ARM. 

First in \Cref{subsec:linear regret simple}, we show that a fixed (aka static) price policy, which is known to be near-optimal under \ESM, can indeed be highly suboptimal in \ARM~in the sense that it can suffer with a linear revenue loss compared to the optimal policy. 

Next, in \Cref{subsec:approx opt markdown}, we demonstrate that a markdown policy is optimal for \ARM\ when customers are gain-seeking, and near-optimal (within $O(\log(T))$ revenue) when customers are loss-averse. 

Finally, in \Cref{subsec:characterize approx opt markdown}, we provide a detailed characterization and computational insights for this markdown policy for the case when the base demand model $H(p)$ is linear. This characterization will be later utilized in \Cref{sec:algo} for learning and regret minimization under demand uncertainty. 

\subsection{The Sub-Optimality of Fixed Price}
\label{subsec:linear regret simple}
\input{algo-design/fixed-price-policy}

\subsection{The (Near) Optimality of Markdown Pricing}
\label{subsec:approx opt markdown}
\input{algo-design/approx-opt-markdown}

\subsection{Characterizing Near-optimal Markdown Price Curve}
\label{subsec:characterize approx opt markdown}
\input{algo-design/characterize-approx-opt-markdown}

%% file: algo-design/fixed-price-policy.tex
\newcommand{\greedyPrice}{p^{\GR}}

We refer to $\policy$ as a fixed price policy if $\policy = (p, \ldots, p)$ for some $p\in [0,\priceUB]$. We show that under \ARM, any fixed-price policy  
can have linear (in number of rounds $T$) revenue loss compared to an optimal pricing policy, even  if we restrict to the linear base demand model and loss-neutral customers. 
\begin{restatable}{proposition}{linearregretfixed}
\label{prop:linear regret fixed}
There exists an \ARM\ problem instance with linear base demand model, i.e., $\baseDemand(p) = \baseDemandParaB-\baseDemandParaA p$ and loss-neutral customers
(i.e., $\positiveRef = \negativeRef $), and an initial reference price $\rePrice_1$
such that for any fixed-price policy $\policy$, we have
$\optVal(\rePrice_1) - \val^{\policy}(\rePrice_1) = \Omega(T)$. 
\end{restatable}
The above results highlight a fundamental difference between our 
\ARM\ model and the previously well-studied \ESM\ model for reference effects.
In particular, for the base linear demand, 
when the reference-price dynamics follow \ESM\
and when the customers are loss-neutral,
previous work \citep{dK-21} have shown that 
the seller can safely ignore the reference effect (notice that 
when the reference effects are ignored,
the single-round expected revenue would equal 
$p(\baseDemandParaB-\baseDemandParaA p)$ and 
thus the optimal fixed-price would be $\sfrac{\baseDemandParaB}{2\baseDemandParaA}$).
They show that a fixed-price policy with the selling price at $\sfrac{\baseDemandParaB}{2\baseDemandParaA}$ 
yields a revenue very close to the revenue under the optimal pricing policy (whenever $\positiveRef\le \negativeRef$): 
the difference between these two revenues is bounded 
by a constant independent of the sales horizon $T$.
However, the above \Cref{prop:linear regret fixed} shows that 
the performance of such fixed-price policy can be arbitrarily bad under \ARM. \wtr{One intuition behind this performance dichotomy (compared to the optimal total revenue) of fixed-price policy between \ESM\ and \ARM\ is as follows: Recall that the \ESM\ model 
consider constant averaging factor $\ESMPara \equiv \ESMPara_t$ for all $t$, in doing so, the customers essentially put exponentially decreasing weights on the past prices to form the reference price. This is similar to a setting where the reference price at time $t$ equals to the average of a constant window size (independent of sales horizon $T$) of past prices.
While under \ESM, this window size scales linearly in time $t$.
} 
\begin{remark}
We would like to note that a fixed-price policy has the (nearly) same revenue under both $\ESM$ and our $\ARM$.
Together with the results
in \Cref{prop:linear regret fixed}, this demonstrates that under $\ARM$, the optimal revenue can be significantly higher than the optimal revenue under \ESM. 
\end{remark}

We prove \Cref{prop:linear regret fixed} 
by showing that when customers are loss-neutral, 
there exists a simple two-price policy 
with $\Omega(T)$-larger revenue compared to the revenue obtained under any fixed-price policy. The policy starts with the higher price, and then at some round, switches to the lower price which is offered for the remaining rounds.\footnote{\wtr{We construct an instance with initial reference price $\rePrice_1 = 0$ to prove \Cref{prop:linear regret fixed}. However, we note that the proof can be easily extended to consider arbitrary $\rePrice_1$.}}

%% file: algo-design/approx-opt-markdown.tex
In this section, we show that for any \ARM~problem instance $I$, there always exists a markdown pricing policy 
that is near-optimal. 
We first define the markdown pricing policy.
\begin{definition}[Markdown pricing policy]
A markdown pricing policy is defined as any pricing policy which, when applied starting at any $t_1\in[T]$ and reference price $r$, generates a price curve $(p_t)_{t\ge t_1}$ 
satisfying 
$p_t \ge p_{t+1}$ for all $t\ge t_1$.
\end{definition}
The main results of this subsection are summarized as follows:
\begin{restatable}[Near optimality of markdown pricing policy]{theorem}{approxoptofmarkdown}
\label{prop:approx opt of markdown}
Fix any starting reference price 
$\rePrice_1\in[0, \priceUB]$ at time $t = 1$, 
\begin{enumerate}[leftmargin=*, label=\textbf{1\alph*}, topsep=0pt, itemsep = 2pt]
    \item
    \label{gain-seeking}
    when $\positiveRef \ge \negativeRef $, i.e., when customers are gain-seeking,
    optimal pricing policy
    is a markdown policy;
    \item
    \label{risk-averse}
    when $ \positiveRef < \negativeRef$, 
    i.e., when customers are 
    loss-averse,
    there exists a 
    markdown policy $\policy$ that 
    is near-optimal, namely,
    $\optVal(\rePrice_1) - \val^{\policy}(\rePrice_1) = 
    O\left(
    \priceUB(\priceUB-\rePrice_1) (\negativeRef+\positiveRef)\ln T
    \right)$.
\end{enumerate}
\end{restatable}
From the above results, we can see that 
the optimality of the markdown pricing policy depends on the 
relative values of reference effects $\positiveRef$ and $\negativeRef$. 
When the customers are gain-seeking (i.e., $\positiveRef\ge \negativeRef$),
the optimal pricing policy is a markdown policy. 
This characterization shares certain similarity with previous results 
on the optimal pricing policy
when the reference-price dynamics follow the \ESM. 
In particular, \citet{HCH-16} have shown that when 
the customers are insensitive to the loss (i.e., $\positiveRef > 0, \negativeRef = 0$)
and they only remember the most recent price (i.e., the averaging factor $\ESMPara$ in \ESM\ equals to $0$), 
the optimal pricing policy is a {\em cyclic markdown pricing policy}.
While in our setting with \ARM, 
no matter how customers are sensitive to the losses, 
as long as customers value more on the gains, the optimal pricing 
policy is always a markdown policy. 
\wt{discuss in more detail the comparison with the cyclic Markdown policy which is shown in \citet{HCH-16} to be optimal under some conditions}

On the other hand, when customers are loss-averse  (i.e., $\positiveRef < \negativeRef$),
the markdown pricing policy may not be optimal.
However, it is guaranteed that there exists a markdown pricing policy 
that is near-optimal, namely, its total 
revenue is within  $O\left(\ln T \right)$ of the optimal revenue. 



To prove \Cref{prop:approx opt of markdown}, we first show the 
optimality of markdown pricing policy when the customers are gain-seeking. 
This is summarized in the following lemma.  
\begin{restatable}{lemma}{optmarkdownwithlargerposiRef}
\label{lem:opt markdown with larger posiRef}
Fix any starting time $t_1\in[T]$ and a starting reference price $\rePrice_{t_1} = \rePrice$, 
when $\positiveRef \ge \negativeRef$, 
the optimal pricing policy 
starting from time $t_1$ 
is a markdown pricing policy. 
\end{restatable}
The main intuition behind the above markdown optimality is that: 
under \ARM, the high price and low price contribute the same 
to the reference prices in later rounds. 
Since customers are gain-seeking, 
i.e., the impact of a perceived gain (how low is price compared to reference price) is more than the impact of the same amount of perceived loss, it is always better for the seller to charge the high price before the low price. 
Indeed, we prove the above results by showing that whenever under a pricing policy 
$\policy = (p_t)_{t\in[t_1, T]}$ there is an increase in price, i.e., $p_k < p_{k+1}$ for some time rounds $k$ and $k+1$,
then we can obtain a new pricing policy $\policy'$ with a higher payoff by simply 
switching the prices $p_k$ and $p_{k+1}$ at time round $k$ and $k+1$. 

To analyze performance of the markdown pricing policy when the customers are loss-averse, we first show 
that when the starting reference price is the highest possible price $\priceUB$, 
then even in this case, the optimal policy is a markdown pricing policy. Intuitively, in this case, the markdown policy will never offer a price $p_t$ that is above the current reference price $r_t$, so there is never a perceived loss. 
\begin{restatable}{lemma}{markdownforpriceUPasinitial}
\label{lem:markdown for priceUP as initial}
Fix any starting time $t_1\in[T]$ and a starting reference price $\rePrice_{t_1} = \priceUB$, 
if $\positiveRef <  \negativeRef $, 
then the optimal pricing policy 
starting from time $t_1$ 
is a markdown pricing policy.
\end{restatable}

We then show the following Lipschitz property on how the optimal revenue $\optVal(\rePrice, t_1)$ depends on the starting reference price $\rePrice$.
Notably, this property holds for any reference effects. 
\begin{restatable}[Optimal revenue gap w.r.t.\ different starting reference price]{lemma}{optrevdiffdiffref}
\label{lem:opt rev diff w.r.t diff ref}
Fix any starting time $t_1\in[T]$, 
the optimal revenue function $\optVal(\rePrice, t_1)$ is increasing w.r.t.\ reference price $\rePrice$.
Moreover, for any $(\positiveRef, \negativeRef)$, we have
$\optVal(\rePrice', t_1) - \optVal(\rePrice, t_1) \le 
O\left(
\priceUB t_1(\rePrice'-\rePrice) (\negativeRef+\positiveRef)\ln \sfrac{T}{t_1}
\right)$ for any $\rePrice' \ge \rePrice$.
\end{restatable}

The near optimality of markdown pricing for loss-averse  customers
immediately follows
by combining \Cref{lem:markdown for priceUP as initial} and 
\Cref{lem:opt rev diff w.r.t diff ref}. 
Namely, fix an arbitrary starting reference price, 
one can simply implement the price sequence 
from optimal pricing policy (which is a markdown pricing policy) for starting reference price $\priceUB$, 
then one can guarantee near-optimal revenue.
The missing proofs for the above three lemmas and the 
detailed proof of \Cref{prop:approx opt of markdown} are deferred to \Cref{apx-proof-approx-markdown}.

%% file: algo-design/characterize-approx-opt-markdown.tex
\newcommand{\criticalTime}{t^\dagger}
\newcommand{\criticalPrice}{p^\dagger}


In preceding discussions, we 
show the near optimality of markdown pricing policy. 
In this section, focusing on base linear demand,
namely, $\baseDemand(p) = \baseDemandParaB - \baseDemandParaA p$,
we provide detailed structural characterizations for such near-optimal markdown policy.
We also provide computational results for computing such  policy.

\xhdr{The structure of the near-optimal markdown price curve}
When the base demand 
is $\baseDemand(p) = \baseDemandParaB - \baseDemandParaA p$,
we can characterize the following structure of a (near-)optimal markdown policy:
\begin{restatable}{proposition}{approxmarkdownstructure}
\label{prop:approx markdown structure}
\label{defn:policy class}
Given a problem instance $\basePara = (\baseDemandParaA, \baseDemandParaB, \positiveRef, \negativeRef)$ with linear base demand model, 
and a starting reference price $r$, we define price curve $\approxMD (\rePrice)\triangleq (p_t)_{t\in[T]}$ as:
\begin{equation*}
    p_t =
    \begin{cases}
    \begin{alignedat}{3}
    &\priceUB, &&  t\in[\criticalTime-1] && \\
    &\criticalPrice, && t = \criticalTime && \\
    &p_{t-1} - \frac{\positiveRef \rePrice_{t-1}}{2(\baseDemandParaA+\positiveRef)t + \positiveRef}, ~~&& t\in[\criticalTime+1, T-1], 
    ~~
    && \\
    &\frac{\positiveRef\rePrice_T + \baseDemandParaB}{2(\baseDemandParaA+ \positiveRef)}, && t = T, 
    &
    \end{alignedat}
    \end{cases}
\end{equation*}
where $\criticalPrice$ and $\criticalTime$ are some deterministic functions 
of $(\baseDemandParaA, \baseDemandParaB, \positiveRef, \rePrice)$, and $\rePrice_{t}  = \frac{\rePrice + \sum_{s=1}^{t-1}p_s }{t}, t\in[T]$. 
Then the price curve $\approxMD (\rePrice)$ is optimal when $\positiveRef = \negativeRef$, i.e., $V^*(r) = V^{\approxMD (r)}(r)$. 

Furthermore,
the price curve $\approxMD (\priceUB)$ (i.e., the above price curve computed with $r=\priceUB$) is near optimal when $\positiveRef \neq \negativeRef$,
namely,  for any starting reference price $r$,
$\optVal(\rePrice) - \val^{\approxMD (\priceUB)}(\rePrice) \le O\left(\ln T\right)$.
\end{restatable}
We can see that price curve $\approxMD(\rePrice)$ keeps charging the 
price as the highest possible price $\priceUB$ until the time round $\criticalTime-1$, and then at $\criticalTime$ it strictly markdowns its price to $\criticalPrice$. 
Given the price $p_t$ at time round $t \ge \criticalTime$, 
the markdown amount for the next price $p_{t+1}$
depends on the previous reference price $\rePrice_{t}$, and 
the current time round $t+1$, and also the parameters $\baseDemandParaA, \baseDemandParaB, \positiveRef$. 
In particular, the parameters $\baseDemandParaA, \positiveRef$ control the degree of the markdown amount.
The time round $\criticalTime$ and the price $\criticalPrice$ 
can be determined by the parameters 
$(\baseDemandParaA, \baseDemandParaB, \positiveRef)$, and the starting reference price $\rePrice$.

We below give a graph illustration in \Cref{fig:price curve}
about the shape of the price curve $\approxMD$.
In the figure, we show that the time round $\criticalTime$ could be strictly larger than first time round (see \Cref{fig:case a}), 
or could equal to first time round $t=1$, which in this case, 
the price curve then becomes a strict markdown price curve 
(see \Cref{fig:case b}) throughout the sales horizon. 


\begin{figure}[H]
\centering
\subfloat[]{
\scalebox{1}{\input{figs/w-t-dagger}}
\label{fig:case a}
}
\qquad
\subfloat[]{
\scalebox{1}{\input{figs/wn-t-dagger}}
\label{fig:case b}
}
\caption{\label{fig:price curve}
    An illustration of the price curve defined in \Cref{prop:approx markdown structure}.
}
\end{figure}
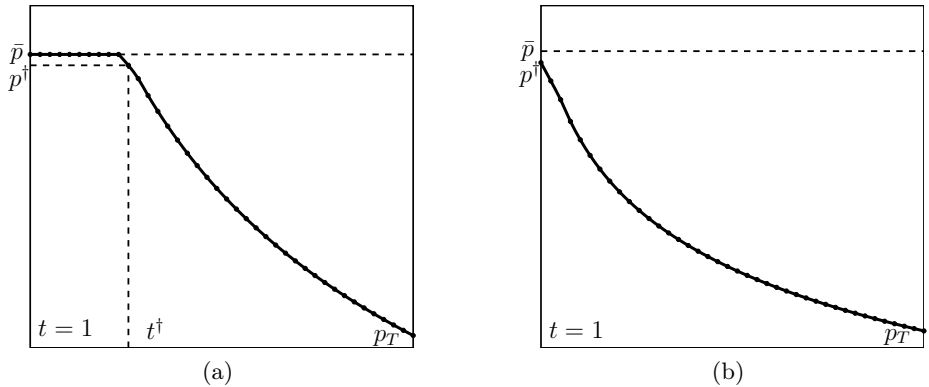
\begin{remark}
It is worth noting that
the price curve defined in \Cref{prop:approx markdown structure} only depends on the gain effect parameter $\positiveRef$. 
This property will be helpful in designing our learning algorithm
in \Cref{sec:algo}. 
\end{remark}

\begin{proof}[Proof Sketch of \Cref{prop:approx markdown structure}]
The full proof is provided in \Cref{apx:proof of approx opt}.
Below we outline the ideas for proving the optimality of $\approxMD(\rePrice)$ when customers are loss-neutral.
We note that we can re-write the seller's 
problem \ref{eq:opt} using the following Bellman Equation:
\begin{align*}
    \label{eq:opt Q function}
    \tag{$\mathcal{P}_{\cc{OPT-BE}}$}
    \optVal(\rePrice, t) 
    = \sup_{p\in[0, \priceUB]} ~~ 
    \staticRev{p, \rePrice}  + \optVal\left(\frac{t\rePrice + p}{t+1}, t+1\right)~.
\end{align*}
When customers are loss-neutral, namely $\positiveRef = \negativeRef$, 
the single-shot revenue function $\staticRev{p, \rePrice}$ is differentiable in $p, \rePrice$.
Thus, we are able to apply the first-order optimality condition and 
envelope theorem for the optimal price to the above program 
\ref{eq:opt Q function} to deduce a condition 
on the partial derivatives on the single-shot revenue function $\staticRev{p, \rePrice}$. In particular, we show that 
the optimal prices $p_t^*, p_{t+1}^*$ at time round $t, t+1$, repectively
must satisfy (partial derivatives are denoted by corresponding subscripts):
\begin{align}
    \label{eq:FOC optimality}
    \staticRev[p]{p_t^*, \rePrice} + 
    \frac{1}{t+1} \staticRev[\rePrice]{p_{t+1}^*, \frac{t\rePrice + p_t^*}{t+1}} - \staticRev[p]{p_{t+1}^*, \frac{t\rePrice + p_t^*}{t+1}} = 0~.
\end{align}
Substituting the revenue function $\staticRev{p, \rePrice}$ 
with the base linear demand can lead to the price markdown rule in \Cref{prop:approx markdown structure}. 
Together with \Cref{prop:approx opt of markdown}, we can deduce that 
the price curve defined in \Cref{prop:approx markdown structure} is indeed 
optimal for loss-neutral customers. 

For customers with $\positiveRef\neq \negativeRef$, we use 
\Cref{lem:opt markdown with larger posiRef} and \Cref{lem:markdown for priceUP as initial} to show that the price curve $\approxMD(\priceUB)$ is always near-optimal.
\end{proof}


\xhdr{Computing a near-optimal markdown policy}
Solving for an optimal pricing policy (i.e., the set of equations \ref{eq:opt Q function}) amounts to solving a dynamic program with non-smooth and non-concave objective function (as customers may have asymmetric reference effects, i.e., $\positiveRef\neq \negativeRef$) 
and time-variant transition function. 
Here, we show that, by leveraging the markdown structure for the price curve defined in \Cref{prop:approx markdown structure}, we are able to design an efficient algorithm that computes a near-optimal markdown pricing curve.
\begin{restatable}[Computing near-optimal markdown]{proposition}{complexityapproxopt}
\label{cor:computational complexity approx opt}
For any $(\positiveRef, \negativeRef)$,
there exists an algorithm (see \Cref{algo:computing opt}) 
that solves only $O(\ln T)$ linear systems to compute the 
price curve $\approxMD(\rePrice)$
defined in \Cref{prop:approx markdown structure} for any $\rePrice\in[0, \priceUB]$.
\end{restatable}
\begin{proof}[Proof Sketch of \Cref{cor:computational complexity approx opt}]
The full proof is provided in \Cref{apx:proof of approx opt}.
Here we provide a proof sketch for the case $\positiveRef = \negativeRef$.
When $\positiveRef = \negativeRef$,
the single-round revenue function $\staticRev{\cdot, \cdot}$ becomes smooth and concave.
By leveraging the optimality condition we have derived in \eqref{eq:FOC optimality},
we can show that the partial price sequence $(p_t)_{t\in[\criticalTime, T]}$ 
in price curve
$\approxMD(\rePrice)$ can be solved by solving a system of linear equations, i.e., a linear system $\Amtrx\pvec =\bvec$
where $\Amtrx$ is a $(T-\criticalTime+1) \times (T-\criticalTime+1) $ matrix and $\bvec$ is a $(T-\criticalTime+1)$-dimensional vector. Both matrix $\Amtrx$ and vector $\bvec$ can be determined by model parameter $\basePara$, 
time round $\criticalTime$, reference price $\rePrice_{\criticalTime}$ at time $\criticalTime$ and sales horizon $T$ 
(see formal definitions of $\Amtrx$ and $\bvec$ in \Cref{defn:linear system}).
Thus, to determine the optimal policy for loss-neutral customers, it suffices to determine the 
time round $\criticalTime$ (as reference price $\rePrice_{\criticalTime} = \frac{\rePrice+(\criticalTime-1)\priceUB}{\criticalTime}$ can be determined by $\criticalTime$). 
Since this price curve $\approxMD$ is a markdown curve, 
we can deduce that the time round
$\criticalTime$ is the smallest time round such that the corresponding linear system  $\Amtrx\pvec =\bvec$ has a feasible solution. Thus, we can use a binary search to determine the location of time $\criticalTime$.
As there at most $O(\ln T)$ steps to binary search the time $\criticalTime$,  
the algorithm only solves at most $O(\ln T)$ many of linear systems
to obtain the price curve $\approxMD(\rePrice)$ for any $\rePrice\in[0, \priceUB]$.
\end{proof}

%% file: figs/w-t-dagger.tex
\begin{tikzpicture}[scale=0.8, transform shape]
\begin{axis}[
    xmin=-2, xmax=40,
    ymin=0.7, ymax=1.05,
    xtick=\empty,
    ytick=\empty,
    grid style=dashed,
    axis line style={draw=none}
]

\draw[draw=black, line width=0.8pt] (axis cs:1,0.7) rectangle (axis cs:40,1.05);

\addplot[
    only marks, 
    mark =*,
    mark options={color=black, scale = 0.5},
] coordinates {
    (1,1) (2,1) (3,1) (4,1) (5,1) (6,1) (7,1) (8,1) (9,1) (10,1) (11,0.98886036) (12,0.97528352) (13,0.95797764) (14,0.94181044) (15,0.92666312) (16,0.91242927) (17,0.89901886) (18,0.88635352) (19,0.8743644) (20,0.86299117) (21,0.85218062) (22,0.84188584) (23,0.83206506) (24,0.82268129) (25,0.81370136) (26,0.80509562) (27,0.79683737) (28,0.78890269) (29,0.78126986) (30,0.77391881) (31,0.76683202) (32,0.75999293) (33,0.75338665) (34,0.74699964) (35,0.74081909) (36,0.73483355) (37,0.72903233) (38,0.72340574) (39,0.71794465) (40,0.71264059)
};

\addplot[domain=1:10, black, line width=0.5mm] (x, {1});

\addplot[
    color=black, 
    mark=none,
    smooth,
    line width=0.5mm,
] coordinates {
    (10,1)
    (11,0.98886036) (12,0.97528352) (13,0.95797764) (14,0.94181044) (15,0.92666312) (16,0.91242927) (17,0.89901886) (18,0.88635352) (19,0.8743644) (20,0.86299117) (21,0.85218062) (22,0.84188584) (23,0.83206506) (24,0.82268129) (25,0.81370136) (26,0.80509562) (27,0.79683737) (28,0.78890269) (29,0.78126986) (30,0.77391881) (31,0.76683202) (32,0.75999293) (33,0.75338665) (34,0.74699964) (35,0.74081909) (36,0.73483355) (37,0.72903233) (38,0.72340574) (39,0.71794465) (40,0.71264059)
};

\draw[dashed, line width=0.3mm] (axis cs:11,0.7) -- (axis cs:11,0.98886036);

\draw[dashed, line width=0.3mm] (axis cs:1,0.98886036) -- (axis cs:11,0.98886036);


\draw[dashed, line width=0.3mm] (axis cs:10,1) -- (axis cs:40,1);
\addplot[] coordinates {(1.1,1)} node[left, pos=1]{\large$\priceUB$};

\addplot[] coordinates {(1.86,0.97586036)} node[left, pos=1]{\large$\criticalPrice$};

\addplot[] coordinates {(15.5,0.72)} node[left, pos=1]{\large$\criticalTime$};

\addplot[] coordinates {(39.5,0.710)} node[left, pos=1]{\large$p_T$};
\addplot[] coordinates {(1,0.72)} node[right, pos=1]{\large$t=1$};

\end{axis}
\end{tikzpicture}

%% file: figs/wn-t-dagger.tex
\begin{tikzpicture}[scale=0.8, transform shape]
\begin{axis}[
    xmin=-2, xmax=40,
    ymin=0.9, ymax=1.20,
    xtick=\empty,
    ytick=\empty,
    grid style=dashed,
    axis line style={draw=none}
]

\draw[draw=black, line width=0.8pt] (axis cs:1,0.9) rectangle (axis cs:40,1.20);

\addplot[
    only marks, 
    mark =*,
    mark options={color=black, scale = 0.5},
] coordinates {
    (1,1.1499084) (2,1.13390866) (3,1.11747745) (4,1.09843101) (5,1.08237626) 
    (6,1.06863444) (7,1.05667771) (8,1.04612271) (9,1.03669229) (10,1.02818011) 
    (11,1.02043125) (12,1.01332526) (13,1.00676739) (14,1.00068249) (15,0.99500982) 
    (16,0.98969873) (17,0.98470783) (18,0.98000168) (19,0.97555117) (20,0.97133106) 
    (21,0.96731907) (22,0.96349624) (23,0.95984729) (24,0.95635626) (25,0.95301177) 
    (26,0.94980147) (27,0.94671595) (28,0.94374628) (29,0.94088431) (30,0.93812265) 
    (31,0.93545495) (32,0.93287488) (33,0.93037764) (34,0.92795805) (35,0.92561122) 
    (36,0.92333358) (37,0.92112098) (38,0.91897021) (39,0.91687798) (40,0.91484119)
};

\addplot[
    color=black, 
    mark=none,
    smooth,
    line width=0.5mm,
] coordinates {
    (1,1.1499084) (2,1.13390866) (3,1.11747745) (4,1.09843101) (5,1.08237626) 
    (6,1.06863444) (7,1.05667771) (8,1.04612271) (9,1.03669229) (10,1.02818011) 
    (11,1.02043125) (12,1.01332526) (13,1.00676739) (14,1.00068249) (15,0.99500982) 
    (16,0.98969873) (17,0.98470783) (18,0.98000168) (19,0.97555117) (20,0.97133106) 
    (21,0.96731907) (22,0.96349624) (23,0.95984729) (24,0.95635626) (25,0.95301177) 
    (26,0.94980147) (27,0.94671595) (28,0.94374628) (29,0.94088431) (30,0.93812265) 
    (31,0.93545495) (32,0.93287488) (33,0.93037764) (34,0.92795805) (35,0.92561122) 
    (36,0.92333358) (37,0.92112098) (38,0.91897021) (39,0.91687798) (40,0.91484119)
};




\draw[dashed, line width=0.3mm] (axis cs:1,1.16) -- (axis cs:40,1.16);
\addplot[] coordinates {(1.1,1.16)} node[left, pos=1]{\large$\priceUB$};

\addplot[] coordinates {(1.86,1.1399084)} node[left, pos=1]{\large$\criticalPrice$};


\addplot[] coordinates {(39.5,0.90984119)} node[left, pos=1]{\large$p_T$};
\addplot[] coordinates {(1,0.91584119)} node[right, pos=1]{\large$t=1$};

\end{axis}
\end{tikzpicture}

%% file: algo-design/4-algo-equal-ref-new.tex
\newcommand{\learnGR}{\cc{LearnGreedy}}
\newcommand{\SteerRef}{\cc{ResetRef}}

\newcommand{\learnGRCounter}{s}
\newcommand{\chosenGR}{\widetilde{p}^{\GR}}
\newcommand{\targetRefPrice}{\rePrice}

\newcommand{\expPhaseOne}{{\color{blue}{\cc{Exploration\ phase\ 1}}}}
\newcommand{\smoothPhase}{{\color{blue}{\cc{Resetting\ phase}}}}
\newcommand{\expPhaseTwo}{{\color{blue}{\cc{Exploration\ phase\ 2}}}}
\newcommand{\exploitPhase}{{\color{blue}{\cc{Exploitation\ Phase}}}}

\newcommand{\firstRePrice}{\rePrice^{a}}
\newcommand{\secondRePrice}{\rePrice^{b}}
\newcommand{\exploitT}{T_2}

\newcommand{\interiorGap}{\delta}

The proceeding section has analyzed the structure of the near-optimal pricing policy 
under the \ARM\ when the seller has the complete information
about the underlying demand function. 
This assumption obviously hinders effective application of 
the resulting pricing policies in practice, where demand functions are typically unknown and have to be learned from sales data.
In this section, we explore the design of the dynamic learning-and-pricing policies in the presence of demand model uncertainty 
and customer reference effects.
In particular, we focus on a base linear demand model, 
namely, $\baseDemand(p) = \baseDemandParaB - \baseDemandParaA p$, 
and initially, the seller does not know the model parameter
$\basePara = (\baseDemandParaA, \baseDemandParaB, \positiveRef, \negativeRef)$.
In \Cref{subsec:challenages}, we first discuss the challenges in the seller's dynamic pricing and learning problem, 
and then in \Cref{subsec:algorithm details}, we present solutions and proposed learning algorithm.




\subsection{The Learning Challenges}
\label{subsec:challenages}
When the seller has demand uncertainty,
a key difficulty to design a low-regret learning algorithm 
is the dynamic nature of \wtr{both the}
reference-price dynamics and the optimal pricing policy. 
Such non-stationarity 
creates unique challenges to our problem. 
The below mentioned two challenges point out two different technical difficulties that make the techniques from previous literature inapplicable to solve our problem.
\xhdr{Challenge one: unclear how to estimate the model parameter}
One tempting dynamic pricing-and-learning  algorithm is to first estimate the model parameter $\basePara = (\baseDemandParaA, \baseDemandParaB, \positiveRef, \negativeRef)$. 
If we had a good estimate for the parameter $\basePara$, then we can compute the price curve 
$\approxMD(\rePrice)$  with the estimated model parameters. By \Cref{prop:approx markdown structure}, this price curve is near-optimal when computed with true model parameters. Thus, if one can establish some kind of Lipschitz property, we may be able to extend the near-optimality to the estimated price curve. 
However, there are two main difficulties in implementing this idea:
(1) A typical approach to estimate the model parameter is using the ``iterated least squares'' \citep{KZ-14}, which charges a test price for certain rounds, and 
another different price for other certain rounds and then uses the observed 
demand to estimate the model parameter.
This approach relies on a crucial assumption that the underlying demand is stationary and does not change over time.
But in our setting, the underlying demand function is non-stationary and depends on the choices of the past prices through the reference price.
Moreover, the underlying demand function may also be non-smooth as
customers may have asymmetric reference effects, \wtr{i.e., $\positiveRef\neq \negativeRef$}.
It is unclear how to use or modify this typical approach to account for 
the non-stationarity and reference effects to learn the model parameters $(\baseDemandParaA, \baseDemandParaB, \positiveRef, \negativeRef)$,
especially learn the reference effect parameters $(\positiveRef, \negativeRef)$. 
(2) In addition, the near-optimal price curve $\approxMD$ 
is highly non-stationary, 
and the price in this curve depends on its current time and the model
parameters in a highly non-trivial way. 
It is unclear how this price curve changes 
w.r.t.\ the model parameter, and thus it is difficult to 
establish the Lipschitz property of seller's cumulative revenue
function w.r.t.\ the model parameters.

\wtr{\xhdr{Challenge two: inapplicable to apply restart mechanism}}
Under \ARM, the optimal revenue $\optVal(\rePrice_t, t)$ over the remaining rounds $[t, T]$ can be quite sensitive to the starting
reference price $\rePrice_t$ at time $t$. 
Indeed, one can see that
when the reference price $\rePrice_t$ 
is not close to the reference price $\rePrice_t^*$ under the optimal pricing policy, then even though the seller can use an optimal pricing policy 
(i.e., $\optPolicy(\rePrice_t, t)$)
w.r.t.\ this reference price $\rePrice_t$ for the remaining rounds, 
the collected revenue $\optVal(\rePrice_t, t)$ could be much smaller than 
the optimal revenue $\optVal(\rePrice_t^*, t)$: 
the difference of these two revenues could be as large as in the order 
of $O(t(\rePrice_t^* - \rePrice_t)\ln \sfrac{T}{t})$ (see \Cref{lem:opt rev diff w.r.t diff ref}). Intuitively, in this problem, the seller needs to learn and follow not only the (near-)optimal price curve but also the reference price curve.
One potential fix is one can periodically try to move 
the current reference price $\rePrice_t$ towards to a target 
reference price $\rePrice_{t'}$ that is close to the 
reference price $\rePrice_{t'}^*$ at time $t'$ 
under optimal pricing policy.
However, this fix does not work as (1) the seller does not 
actually know $\rePrice_{t'}^*$ (as she does not know the optimal pricing policy); 
(2) even if the seller knew $\rePrice_{t'}^*$, one may need 
$\Omega(t\left|\rePrice_{t'}-\rePrice_t\right|)$ rounds to reach to the reference price $\rePrice_{t'}$ from the reference price $\rePrice_t$ at time $t$, 
which could lead to linear regret when $t$ is in the order of  $T$.
This challenge makes certain restart mechanisms,
which is a common approach used to tackle the problem in learning with 
time-variant MDP \citep{BGZ-14,CSZ-20}, inapplicable under \ARM.


    \subsection{Solution Ideas and the Proposed Learning Algorithm}
\label{subsec:algorithm details}
In this section, we present our solution to the above 
challenges and our proposed algorithm details.

\xhdr{Reparameterizing the markdown price curve}
The key observation in our algorithm design is that we 
can generalize and reparameterize the price curve $\approxMD(\rePrice)$ defined in \Cref{prop:approx markdown structure} in the following way:
we can generalize the price curve $\approxMD(\rePrice, t_1)$ with 
an arbitrary starting time $t_1$ and starting reference price $\rePrice$ at this time;
moreover, instead of looking at the model parameter $\basePara$, 
we can reparameterize the price curve such that it
depends on the model parameter $\basePara$
only through a two-dimensional {\em policy parameter} $\para\in \paraSpace$
where $\paraSpace \subseteq \R_+\times \R_+$ is a policy parameter space that will be defined later. 
In particular, we have the following generalized version of \Cref{prop:approx markdown structure}: 
\begin{restatable}[Generalized and reparameterized version of \Cref{prop:approx markdown structure}]{lemma}{approxmarkdownstructurestronger}
\label{prop:approx markdown structure t_1}
Given a policy parameter $\para = (C_1, C_2) \in \paraSpace$,
a starting reference price $r$ at time $t_1\in[T]$,
we define price curve $\approxMD (\rePrice, t_1, \para) \triangleq (p_t)_{t\in[t_1, T]}$ as:
\begin{equation*}
    p_t =
    \begin{cases}
    \begin{alignedat}{3}
    &\priceUB, &&  t\in[t_1, \criticalTime-1] && \\
    &\criticalPrice, && t = \criticalTime && \\
    &p_{t-1} - \frac{C_1 \rePrice_{t-1}}{t + C_1}, ~~&& t\in[\criticalTime+1, T-1], 
    ~~
    && \\
    &C_1 \rePrice_T +C_2, && t = T, 
    &
    \end{alignedat}
    \end{cases}
\end{equation*}
where $\criticalPrice$ and $\criticalTime$ are some deterministic functions 
of $(\para, \rePrice, t_1)$, and $\rePrice_{t}  = \frac{t_1\rePrice + \sum_{s=t_1}^{t-1}p_s }{t}, t\in[t_1, T]$. 
Given an instance $\basePara = (\baseDemandParaA, \baseDemandParaB, \positiveRef, \negativeRef)$,
let $\truePara \triangleq (\trueCOne, \trueCTwo)
$ with 
\begin{align*}
    \trueCOne 
    \triangleq \frac{\positiveRef}{2(\baseDemandParaA + \positiveRef)},
    \quad
    \trueCTwo 
    \triangleq \frac{\baseDemandParaB}{2(\baseDemandParaA + \positiveRef)}~.
\end{align*}
Then the price curve $\approxMD (\rePrice, t_1, \truePara)$ is optimal when $\positiveRef = \negativeRef$, i.e., $\optVal(\rePrice, t_1) = \val^{\approxMD (\rePrice, t_1,  \truePara)}(\rePrice, t_1)$. 
Furthermore,
the price curve $\approxMD (\priceUB, t_1,  \truePara)$ 
is near-optimal when $\positiveRef \neq \negativeRef$,
namely, for any starting reference price $r$,
$\optVal(\rePrice, t_1) - \val^{\approxMD (\priceUB, t_1,  \truePara)}(\rePrice, t_1) \le O\left(t_1\ln \sfrac{T}{t_1}\right)$.
\end{restatable}
One benefit of the above generalization is that it provides a way to obtain a price curve starting from any time $t_1$ that is near-optimal for the remaining rounds as long as $t_1$ is not very large (e.g., if $t_1$ is sublinear in $T$). By reparameterization, we also note that even though 
the price curve $\approxMD(\rePrice, t_1, \truePara)$ heavily depends on the starting time $t_1$, 
the universal constants $\trueCOne$ and $\trueCTwo$ to characterize this curve
are fixed irrespective of the starting time round $t_1$ and the starting reference price $r$ at $t_1$. 
This allows design of an algorithm that first does some price explorations for sublinear rounds in order to estimate these universal constants, and then implements the near-optimal price curve for the remaining rounds based on these estimated parameters.

\wtr{One may recall that from \Cref{prop:approx opt of markdown} and \Cref{lem:markdown for priceUP as initial}, the optimal pricing policy is always a markdown pricing policy when the initial reference price equals to the price upper bound $\priceUB$. Thus, one may expect that the learning problem would become much easier if we assume $\rePrice_1 = \priceUB$ at the very beginning. However, we note that having $\rePrice_1 = \priceUB$ does not simplify the structure of the optimal markdown and the previously mentioned two learning challenges still remain.} 


\begin{remark}
From 
\Cref{assump:maximizer is in feasible set} 
and \Cref{assump:demand non-negativity}, 
we know that the price upper bound $\priceUB$ must satisfy 
$\sfrac{\baseDemandParaB}{2\baseDemandParaA} < \priceUB < 
\sfrac{\baseDemandParaB}{(\baseDemandParaA+\positiveRef)}$, 
which implies that one must have $\baseDemandParaA > \positiveRef$. 
With this observation, 
we can deduce $\trueCOne < \sfrac{1}{4}, \trueCTwo \ge \sfrac{\priceUB}{2}$.
Thus, the policy parameter 
 $\truePara$ induced by all possible $\basePara\in\baseParaSpace$
effectively satisfies that $\theta^* \in [0, \sfrac{1}{4}) \times (\sfrac{\priceUB}{2}, \infty) =: \paraSpace$.
\end{remark}

\begin{remark}
As a sanity check, when $\positiveRef = 0$ and $\negativeRef > 0$, i.e., there is only loss reference effect, 
the price curve defined in \Cref{prop:approx markdown structure} would become a fixed-price policy that keeps charging the price $\sfrac{\baseDemandParaB}{2\baseDemandParaA}$.
\end{remark}

\xhdr{Learning the policy parameter via greedy price}
The above observations motivate us to design an algorithm that could 
efficiently learn a good estimate for the policy parameter $\truePara = (\trueCOne, \trueCTwo)$ defined in \Cref{prop:approx markdown structure t_1}.
Interestingly, we observe that the greedy price for 
a certain range of reference prices $\rePrice$ is also 
fully characterized by the policy parameter $\truePara$. 
In particular, let $\interiorGap \triangleq \priceUB - \max_{\basePara\in\baseParaSpace} \frac{\baseDemandParaB}{2\baseDemandParaA}$, we can show that given any reference price $\rePrice
\in (\priceUB -\interiorGap, \priceUB]$, the greedy price $\greedyPrice(\rePrice)$ that maximizes the single-shot
revenue function $\staticRev{p, \rePrice}$ subject to the constraint $p\le \rePrice$ for this 
reference price $\rePrice$ satisfies 
$\greedyPrice(\rePrice)\in [0, \rePrice]$, and moreover,
it can be characterized as follows (see \Cref{lem:distance new})
\begin{align*}
    \greedyPrice(\rePrice) 
    \triangleq \frac{\baseDemandParaB + \positiveRef \rePrice}{2(\baseDemandParaA + \positiveRef)}
    = \trueCOne \rePrice + \trueCTwo~.
\end{align*}
The above observation implies that if we can 
learn two greedy prices $\greedyPrice(\firstRePrice)$
and $\greedyPrice(\secondRePrice)$ for two different reference prices $\firstRePrice$ and $\secondRePrice$, respectively, then we can solve the following 
system of two linear equations to get the policy 
parameter $\para^* = (\trueCOne, \trueCTwo)$:
\begin{equation}
    \label{eq:compute greedy}
    \left\{\begin{array}{l}
        \greedyPrice(\firstRePrice) = \trueCOne \firstRePrice + \trueCTwo~, \\
       \greedyPrice(\secondRePrice) = \trueCOne \secondRePrice + \trueCTwo~.
        \end{array}
    \right.
\end{equation}
In addition, we observe that
the single-shot revenue function $\staticRev{\cdot, \rePrice}$
over the price range $[0, \rePrice]$ 
is strongly-concave and smooth.
Thus, we can learn the greedy price by utilizing ideas from 
stochastic convex optimization techniques. 
Notably, in our problem, 
the seller can only learn a noisy demand value
$\realizedDemand_t(p_t, \rePrice)$ for the 
expected demand $\realizedDemand(p_t, \rePrice)$ at the chosen price $p_t$.
With this feedback structure, the problem of learning greedy price $\greedyPrice(\rePrice)$ essentially reduces to stochastic convex optimization with bandit feedback.

\begin{algorithm}[H]
\caption{Dynamic pricing and learning under \ARM}
\label{algo:equal ref}
\begin{algorithmic}[1]
\State \textbf{Input:} Horizon $T$, starting reference price $\rePrice_1$, price upper bound $ \priceUB$.
\State \textbf{Initialization:} $T_1$; 
$\firstRePrice  $, $\secondRePrice$.\\
\Comment{Exploration phase 1:\ learning an estimate $\estGreedy(\firstRePrice)$ for reference price $\firstRePrice$}
    \State Run $\learnGR(T_1, \firstRePrice, \priceUB)$ to get $\estGreedy(\firstRePrice)$, and let $t$ be 
    the current time step.\\
\Comment{Exploration phase 2:\ learning an estimate $\estGreedy(\secondRePrice)$ for reference price $\secondRePrice$}
    \State Run $\learnGR(T_1, \secondRePrice, \priceUB)$ to get $\estGreedy(\secondRePrice)$. \\
\Comment{Exploitation phase }
\State Compute policy parameter  estimate
$\paraEst = (\estCOne, \estCTwo)$ from \eqref{eq:policy para from greedy} using $\estGreedy(\firstRePrice)$, $\estGreedy(\secondRePrice)$. \\
\Comment{Suppose current time is $\exploitT$}
\State Compute the price curve 
$\approxMD(\priceUB, \exploitT, \paraEst)$ 
defined as in \Cref{prop:approx markdown structure t_1}
with $\paraEst$, and implement this price curve for 
remaining rounds.
\end{algorithmic}
\end{algorithm}

\xhdr{The algorithm details}
With the above challenges and solutions in mind, we design an 
``explore-then-exploit''-style policy-parameter learning algorithm.
At a high-level, our algorithm first learns  estimates $\estGreedy(\firstRePrice)$,  $\estGreedy(\secondRePrice)$ of two greedy prices 
$\greedyPrice(\firstRePrice), \greedyPrice(\secondRePrice)$ for two carefully-set reference prices $\firstRePrice, \secondRePrice$, respectively. 
Then the algorithm
uses $\estGreedy(\firstRePrice)$,  $\estGreedy(\secondRePrice)$ to construct 
an estimate $\paraEst$ for the policy parameter $\truePara \in \Theta$ by substituting these in \eqref{eq:compute greedy}. Finally, it
 implements the price curve defined in \Cref{prop:approx markdown structure t_1} with the estimated $\paraEst$ 
for the remaining rounds.
We summarize our proposed algorithm in \Cref{algo:equal ref}.
\squishlist
    \item In both \expPhaseOne\ and \expPhaseTwo:
    we design a stochastic convex optimization with bandit feedback algorithm 
    to learn two greedy price estimates $\estGreedy(\firstRePrice)$ (in \expPhaseOne) and $\estGreedy(\secondRePrice)$ (in \expPhaseTwo). 
    We summarize the steps of learning greedy price in  \Cref{algo:zeroth-order},
    \learnGR($T_1, \rePrice, \priceUB$),
    which takes a time budget $T_1$, a reference price $\rePrice$, 
    and the price upper bound $\priceUB$
    as the input. 
    In particular, we use $T_1$
    to balance the estimation error $|\estGreedy(\rePrice) - \greedyPrice(\rePrice)|$ and the regret incurred in this algorithm;
    use reference price $\rePrice$  to learn the greedy price $\greedyPrice(\rePrice)$;  
    and use the price upper bound $\priceUB$ to control the learning rate.
    One difficulty here is while running the \Cref{algo:zeroth-order}, 
    the reference price will not remain fixed, but evolve according to \ARM.
    Thus, to learn the greedy price for a particular reference price $\rePrice$, 
    before each learning round $t$
    in \Cref{algo:zeroth-order}, we use a subroutine $\SteerRef(t, r_t, \secondRePrice)$ (described as \Cref{algo:steering reference price} in \Cref{apx:missing algo})
    to reset the current
    reference price to the reference price $\rePrice$.\footnote{\wtr{The subroutine $\SteerRef(\cdot, \cdot, \cdot)$ is designed to set the prices within the price range $[0, \priceUB]$.}}\footnote{\wtr{Notice that the seller knows the initial reference price $\rePrice_1$, thus the seller knows when to stop when she reaches to $\firstRePrice$ or $\secondRePrice$.}}

    \item 
    In \exploitPhase: we first construct the estimate 
    $\paraEst = (\estCOne, \estCTwo)$ using \eqref{eq:compute greedy} with the estimates $\estGreedy(\firstRePrice)$ and $\estGreedy(\secondRePrice)$.
    Suppose $T_2$ is the time that the algorithm enters in this phase,
    we then compute the price curve 
    $\approxMD(\priceUB, T_2, \paraEst)$ (defined in \Cref{prop:approx markdown structure t_1})
    with the policy parameter estimate $\paraEst$, and starting time $T_2$ (note that here we compute the price curve 
    with starting reference price $\priceUB$, this price curve is 
    near-optimal, by \Cref{prop:approx markdown structure t_1}, 
    for any $(\positiveRef, \negativeRef)$
    if we had $\paraEst = \truePara$ ).
    Then we implement this price curve for remaining rounds.
\squishend

\vspace{-10pt}
\begin{algorithm}[H]
\begin{algorithmic}[1]
\State \textbf{Input:} Budget $T_1$, 
reference price $\targetRefPrice$,
price upper bound $\priceUB$, 
current time round $t$.
\State \textbf{Initialization:} 
Initialize 
$p_1 = \greedyPrice_1 \in [0, \priceUB]$ arbitrarily;
let $\distance = 
\frac{1}{2}(\rePrice - \max_{\basePara\in\baseParaSpace}  \frac{\baseDemandParaB}{2\baseDemandParaA})$.
\State \textbf{Initialization:} 
Initialize $s \gets 0$.\\
\Comment{Counter $t$ records all time rounds}\\
\Comment{Counter $s$ records learning rounds}
\While{$t < T$ and $\learnGRCounter < T_1$}
    
    \State $t \gets t + \SteerRef(t, \rePrice_t, \targetRefPrice)+1$. \Comment{Use $\SteerRef(t, \rePrice_t, \targetRefPrice)$ to reset reference price to $r$.}
    \State $s\gets s+1$.
    \State 
    Pick $\randomDirec\in\{-1, 1\}$ uniformly at random. 
    \State 
    Set the price $p_t \gets \greedyPrice_s + \randomDirec \distance $, and observe the realized demand $\demand_{t}(p_t, \targetRefPrice)$.
    \State
    Let $\greedyPrice_{\learnGRCounter+1} \gets 
    \Proj_{[\distance, \rePrice - \distance]} 
    \left(\greedyPrice_{\learnGRCounter} + \frac{p_t \demand_{t}(p_t, \targetRefPrice)}{2\priceUB \distance \learnGRCounter}\randomDirec\right)$.
\EndWhile
\State \textbf{Return} 
$\estGreedy(\rePrice)  = 
\frac{\sum_{s = 1}^{T_1} \greedyPrice_s }{T_1}$, 
and the current time $t$.
\end{algorithmic}
\caption{$\learnGR$: Bandit stochastic convex 
optimization for learning greedy price}
\label{algo:zeroth-order}
\end{algorithm}



%% file: algo-design/5-analysis-equal-ref.tex
\newcommand{\constant}{c}

In this section, we derive an upper bound the regret of 
\Cref{algo:equal ref} for any problem instance that has parameter $I=(\baseDemandParaA, \baseDemandParaB, \positiveRef, \negativeRef)$ in $\baseParaSpace$. Recall that $\baseParaSpace \subseteq [0, 1]^4$ was defined as the set of all feasible parameter vectors  satisfying 
\Cref{assump:maximizer is in feasible set}. 
\begin{theorem}
\label{thm:regret upper bound}
With $T_1 = 
\widetilde{\Theta}\left(\frac{\priceUB^2\sqrt{T}}{\sqrt{1+\priceUB}}\right)$,
for any instance $\basePara\in\baseParaSpace$, 
any starting reference price $\rePrice\in[0, \priceUB]$ at $t=1$,
\Cref{algo:equal ref} has expected regret 
\begin{align*}
    \Reg[\pricingAlgo]{T, \basePara}\le 
    \widetilde{O}\left(\priceUB^3\sqrt{\priceUB T}\right)~.
\end{align*}
\end{theorem}
As a comparison, for the setting where reference-price dynamics follow \ESM, 
\citet{dK-21}, which also focuses on linear base demand model, proposed a learning algorithm
with a regret upper bound of
$\widetilde{O}({\priceUB^6\sqrt{T}}/{(1-\ESMPara)^2})$ for loss-averse customers\footnote{The dependence of regret bound on $\priceUB$ has not been explicitly mentioned in the results of this paper. We have derived it here to the best of our understanding of their proof.}. Here, $\ESMPara$ is the constant averaging factor as described in \Cref{rmk:non-stationary}. 
Note that under \ARM, the averaging factor $\ESMPara_t$ is 
$\ESMPara_t = \frac{t}{t+1}$ so that $1/(1-\ESMPara_t) = t+1$.
Thus naive application of their algorithm in our setting with \ARM\ would incur a linear regret. 
This again highlights the previously mentioned observation that our 
\ARM\ is fundamentally different from \ESM.

Also note that our regret upper bound in \Cref{thm:regret upper bound} depends polynomially on the price upper bound $\priceUB$. 
\wtr{One may wonder if we can first consider a setting with price range $[0, 1]$ (instead of $[0,\priceUB]$), and then recover the general regret upper bound (for arbitrary price upper bound $\priceUB$) simply by scaling. 
In doing so, one may hope to obtain a final regret bound that has quadratic dependency of $\priceUB$ (since the instant revenue quadratically depends on the price).
However, we would like to point out that the price upper bound $\priceUB$ in our setting appears beyond just scaling, it also regulates the range of model parameters. 
Moreover, the instant revenue quadratically depending on price does not imply that the optimal total revenue also quadratically depends on $\priceUB$ due to the nature of the nonstationarity of our demand function.}
When there are no reference effects, \citet{KZ-14} have shown that
no learning algorithm 
can have regret growing slower than $\Omega(\sqrt{T})$, even when restricted to instances satisfying \Cref{assump:maximizer is in feasible set}. We build upon their proof to derive the following more detailed lower bound that shows that for any learning algorithm, the dependence of regret on price upper bound $\priceUB$ is unavoidable even when there are no reference effects.
\begin{restatable}{proposition}{lowerbound}
\label{thm:LB}
Given a price upper bound $\priceUB \ge 1$, 
consider the following problem instances $\baseParaSpace'$: 
let $\positiveRef=\negativeRef\equiv 0$,  
$\baseDemandParaB \equiv 1$, and
$\baseDemandParaA \in [\frac{3}{4\priceUB}, \frac{1}{\priceUB}]$. 
Then clearly all instances in $\baseParaSpace'$
satisfy 
\Cref{assump:maximizer is in feasible set}.
And the expected regret of any algorithm $\pricingAlgo$
satisfies that 
$\sup_{\basePara\in\baseParaSpace'}\Reg[\pricingAlgo]{T, \basePara} 
\ge \frac{\priceUB\sqrt{T-1}}{3\sqrt{1 + 36\pi^2 }}$.\footnote{\wtr{The lower bound we derive here scales linearly w.r.t.\ price upper bound $\priceUB$, it is an interesting open question to study whether one can tighten this gap between the linear dependency in lower bound and the cubic-ish dependency of $\priceUB$ in our upper bound.}}
\end{restatable}

The proof of above \Cref{thm:LB} is in \Cref{apx:proof LB}.
We below provide a proof outline of \Cref{thm:regret upper bound}.
All missing proofs in the following subsection are in \Cref{apx:proof for para estimation} to \Cref{apx:proof lipschitz}.

\subsection{Proof of \Crefrobust{thm:regret upper bound}}
In this subsection, we provide an outline of different steps involved in proving \Cref{thm:regret upper bound}.

\xhdr{Step 1: Bounding the estimation error of the 
policy parameter}
\input{algo-design/policy-para-est-error}

\xhdr{Step 2:  Bounding the total resetting regret}

In this step, 
we bound the total regret incurred from running the subroutine
$\SteerRef(\cdot, \cdot,)$ (i.e., \Cref{algo:steering reference price}) before each of the ($2T_1$) learning round in \Cref{algo:zeroth-order}. 
In particular, we upper bound these regrets 
by upper bounding the total number of rounds used in 
invoking $\SteerRef(\cdot, \cdot,)$.
The main results in this subsection are summarized 
in the following lemma:
\begin{restatable}{lemma}{boundsmoothinexploration}
\label{lem:bound smooth in exploration}
In all \expPhaseOne\ 
and \expPhaseTwo\  of \Cref{algo:zeroth-order}, the number of rounds used in running $\SteerRef(\cdot, \cdot,)$ is $O(T_1)$.
\end{restatable}


\xhdr{Step 3:  Bounding the Lipschitz error in 
policy parameter space}
\input{algo-design/lipschitz-error}

\xhdr{Step 4:  Putting it all together}
We now put all pieces together to show the regret bound 
of \Cref{algo:equal ref}.
\begin{proof}[Proof of \Cref{thm:regret upper bound}]
Let $(\rePrice_t^*)_{t\in[T]}$ be the resulting reference price sequence
under the optimal policy $\optPolicy$.
Let $\pricingAlgo$ be the pricing policy implemented by \Cref{algo:equal ref}, 
and let $\policy$ be the realized price sequence from algorithm $\pricingAlgo$.
Suppose $\exploitT \in [T]$ is the first time round 
(notice that this is a random variable) that
the \Cref{algo:equal ref} enters in the Exploitation phase.
According to our algorithm design, we know that the pricing decisions 
over the time window $[\exploitT, T]$ is the
price curve 
$\approxMD(\priceUB, \exploitT, \paraEst)$
with the starting reference price $\priceUB$
and parameter $\paraEst$. 
Let $\revUB \triangleq \max_{p, \rePrice} \staticRev{p, \rePrice}$.
For any instance $\basePara\in\baseParaSpace$,
the total regrets can be decomposed and bounded as follows: 
\begin{align*}
    \Reg[\pricingAlgo]{T, \basePara} 
    = ~ & \val^*(\rePrice_1) - \expect[\policy\sim\pricingAlgo]{\val^{\policy}(\rePrice_1)} \\
    \overset{(a)}{\le} ~ & 
    \expect[\pricingAlgo]{O(\exploitT\revUB) 
    + \val^*\left(\rePrice_{\exploitT}^*, \exploitT\right) 
    - \val^{\approxMD(\priceUB, \exploitT, \paraEst)}\left(\rePrice_{\exploitT}, \exploitT\right)}\\
    \overset{(b)}{\le} ~ & 
    \expect[\pricingAlgo]{O(\exploitT\revUB) 
    + \val^*\left(\priceUB, \exploitT\right) 
    - \val^{\approxMD(\priceUB, \exploitT, \paraEst)}\left(\rePrice_{\exploitT}, \exploitT\right)}\\
    \overset{(c)}{\le} ~ & 
    \expect[\pricingAlgo]{O(\exploitT\revUB) 
    + \val^*\left(\priceUB, \exploitT\right) 
    - \val^{\approxMD(\priceUB, \exploitT, \paraEst)}\left(\priceUB, \exploitT\right)
    + O\left(\positiveRef \priceUB\exploitT \left|\rePrice_{\exploitT} - \priceUB\right| \log \frac{T}{\exploitT}\right)
    } 
\end{align*}
where 
inequality (a) holds true by observing that the regret incurred 
in first $\exploitT$ rounds is at most $\exploitT\revUB$;
inequality (b) holds true by \cref{lem:opt rev diff w.r.t diff ref} where $\rePrice_{\exploitT}^* \le \priceUB$; 
inequality (c) holds true by
\Cref{lem:rev gap fixed policy} which 
bounds the revenue gap when a same price curve is applied 
with two different starting reference prices.

We now observe that by \Cref{lem:markdown for priceUP as initial}, 
the optimal pricing policy to achieve
$\optVal(\priceUB, \exploitT)$ is a markdown pricing policy, thus, 
$\optVal(\priceUB, \exploitT)$ exactly equals to the optimal revenue 
$\optVal(\priceUB, \exploitT\mid (\positiveRef, \positiveRef))$
when the customers have symmetric reference effects $\positiveRef = \negativeRef$ with starting reference price $\priceUB$.
Here, we slightly abuse the notation and let $\optVal(\rePrice, t\mid (\positiveRef, \positiveRef))$ feature that this cumulative revenue is 
computed with
reference effect parameter $(\positiveRef, \positiveRef)$.
Moreover, the cumulative revenue $\val^{\approxMD(\priceUB, \exploitT, \paraEst)}\left(\priceUB, \exploitT\right)$ is obtained by applying the markdown price curve $\approxMD(\priceUB, \exploitT, \paraEst)$ 
with the starting reference price $\priceUB$ at time $\exploitT$.
Thus, it only depends parameter $\positiveRef$, and $\val^{\approxMD(\priceUB, \exploitT, \paraEst)}\left(\priceUB, \exploitT\right)$ also equals to the cumulative revenue 
 $\val^{\approxMD(\priceUB, \exploitT, \paraEst)}\left(\priceUB, \exploitT\mid (\positiveRef, \positiveRef)\right)$
when customers have $\positiveRef = \negativeRef$ with the starting reference price $\priceUB$.
With this observation, we can further bound the regret as follows:
\begin{align*}
    \Reg[\pricingAlgo]{T, \basePara}  
    \le ~ & 
    \mathbb{E}_{\pricingAlgo}\left[O(\exploitT\revUB) 
    + \val^*\left(\priceUB, \exploitT\mid (\positiveRef, \positiveRef)\right) 
    - \val^{\approxMD(\priceUB, \exploitT, \paraEst)}\left(\priceUB, \exploitT\mid (\positiveRef, \positiveRef)\right) \right.\\
    & \left.+ O\left(\positiveRef \priceUB\exploitT \left|\rePrice_{\exploitT} - \priceUB\right| \log \frac{T}{\exploitT}\right)
    \right]
    \\
    \overset{(a)}{\le} ~ & 
    \mathbb{E}_{\pricingAlgo} \left[O(\exploitT\revUB) + 
    O\left(\positiveRef \priceUB\exploitT \left|\rePrice_{\exploitT} - \priceUB \right| \log \frac{T}{\exploitT}\right) 
    + \delta \revUB\cdot \left(T - \exploitT\right) \right.\\
    &  \left.+ 
    (1-\delta) O\left(\frac{\left(\priceUB \log\left(\log \sfrac{T_1}{\delta}\right) + 1\right)\priceUB^4}{\priceUB^2 
    T_1} \cdot (\baseDemandParaA + \positiveRef)\priceUB^2(T-\exploitT)\left(\log \frac{T}{\exploitT}\right)^2 \right)
    \right]\\
    \overset{(b)}{\le} ~ & 
    O\left(\priceUB^3 \sqrt{(1+\priceUB)T(\log\log T + 1)} (\log T)^{\sfrac{3}{2}}\right)
    = 
    \widetilde{O}\left(\priceUB^3\sqrt{\priceUB T} \right)~.
\end{align*}
where inequality (a) holds by the fact that 
we are able to use \Cref{lem:lipschitz error in policy space} to bound the revenue gap via the policy parameter estimation error that we obtain in \Cref{prop:policy para est error} for symmetric reference effects $\positiveRef = \negativeRef$;
inequality (b) holds true by \Cref{lem:bound smooth in exploration} 
where we have $\exploitT = O(T_1)$, 
and by noticing that $\revUB \le \priceUB(1+\priceUB)$, 
and we optimize $T_1 
= \Theta\left(\priceUB^2\sqrt{\frac{T(\log\log T + 1)\log T}{1+\priceUB}}\right)
= 
\widetilde{\Theta}\left(\frac{\priceUB^2\sqrt{T}}{\sqrt{1+\priceUB}}\right)$
and $\delta = \sfrac{1}{T}$ to get the desired regret.
\end{proof}

%% file: algo-design/policy-para-est-error.tex
In this step, we bound the estimation 
error $\|\paraEst-\truePara\|$ of the policy parameter estimate $\paraEst$, which is 
computed using the outputs of \Cref{algo:zeroth-order} for two reference prices $\firstRePrice, \secondRePrice$. 
Given any reference price $\rePrice\in[0, \priceUB]$,  we consider the following greedy price $\greedyPrice(\rePrice)$,  which maximizes the single round revenue function among all prices in the range $[0, \rePrice]$:
\begin{equation}
    \label{eq:constrained greedy price}
    \greedyPrice(\rePrice) 
    \triangleq 
    \max_{p\in[0, \rePrice]}~ 
    \staticRev{p, \rePrice}
    = \max_{p\in[0, \rePrice]}~ 
    p \left(\baseDemand(p) 
    + \positiveRef(\rePrice - p)^+\right)~.
\end{equation}
We first show that when the reference price $\rePrice$ 
satisfies $\rePrice\in(\priceUB-\interiorGap, \priceUB]$, 
we always have $\greedyPrice(\rePrice) \in [0, \rePrice]$,
and it can be fully characterized by the policy parameter $\truePara$.
\begin{lemma}
\label{lem:distance new}
Fix any reference price $\rePrice\in(\priceUB - \interiorGap, \priceUB]$, then greedy price $\greedyPrice(\rePrice) = \trueCOne\rePrice+\trueCTwo \in (0, \rePrice)$.
\end{lemma}


With the above \Cref{lem:distance new},
we can see that on fixing two different reference prices $\firstRePrice, \secondRePrice \in(\priceUB - \interiorGap, \priceUB]$, the greedy prices $\greedyPrice(\firstRePrice), \greedyPrice(\secondRePrice)$ satisfy  \Cref{eq:compute greedy}. 
And therefore, given the estimates $\estGreedy(\firstRePrice)$, $\estGreedy(\secondRePrice)$ to 
the greedy price $\greedyPrice(\firstRePrice)$ and $\greedyPrice(\secondRePrice)$, 
we can use \Cref{eq:compute greedy} to compute 
a policy parameter estimate $\paraEst= (\estCOne, \estCTwo)$ as follows
\begin{align}
    \label{eq:policy para from greedy}
    \estCOne = \frac{\estGreedy(\secondRePrice) - \estGreedy(\firstRePrice)}{\secondRePrice-\firstRePrice};
    \quad 
    \estCTwo =  
    \frac{\estGreedy(\firstRePrice)\secondRePrice - \estGreedy(\secondRePrice)\firstRePrice}{\secondRePrice-\firstRePrice}~.
\end{align}
We then have the following estimation error 
on the policy parameter estimate $\paraEst$:
\begin{restatable}{proposition}{policyparaesterror}
\label{prop:policy para est error}
Given two reference prices $\firstRePrice, \secondRePrice\in (\priceUB - \interiorGap, \priceUB]$ where $\firstRePrice <\secondRePrice$,
let $\paraEst = (\estCOne, \estCTwo)$ be the estimated parameter 
obtained from Line $12$ 
in \Cref{algo:equal ref},
then the following holds with probability at least $1 - 2\delta$ for $\delta \in (0, \sfrac{1}{e})$:
\begin{align*}
    \left\|\paraEst - \truePara\right\|
    \le 
    O\left(\frac{\priceUB^2\sqrt{\priceUB \left(\log\left(\log \sfrac{t}{\delta}\right) + 1\right)}}{(\secondRePrice - \firstRePrice) \sqrt{T_1}} \right)
    ~.
\end{align*}
\end{restatable}
The key step to prove \Cref{prop:policy para est error}
is the following characterization on how
the estimated greedy price $\estGreedy(\rePrice)$ is close 
to the true greedy price $\greedyPrice(\rePrice)$.
\begin{restatable}{lemma}{optgreedyesterror}
\label{prop:opt greedy est error}
Let $\delta \in(0, \sfrac{1}{e})$ and 
let $\rePrice$ be the reference price to the input of \Cref{algo:zeroth-order}. 
For any $T_1 \ge 4$, we have the following holds with probability at least $1 - \delta$:
\begin{align*}
    \left|\estGreedy(\rePrice) - \greedyPrice(\rePrice)\right|
    \le 
    O\left( \frac{\priceUB^2}{\distance}\sqrt{\frac{\left(\log\left(\log \sfrac{T_1}{\delta}\right) + 1\right)}{T_1}} \right)~,
\end{align*}
where $\distance$ is the parameter
defined in Line 2 in \Cref{algo:zeroth-order}.
\end{restatable}

Our \Cref{algo:zeroth-order} is built on 
\citet{S-13}, which proposed an algorithm 
for stochastic strongly-convex optimization with bandit feedback with assuming that the distance between the maximizer and the domain boundary is at least $\distance$. We satisfy this assumption by carefully 
choosing the distance parameter $\distance$
such that the greedy price
$\greedyPrice(\rePrice)$ is always in the range $[\distance, \rePrice-\distance]$ (see \Cref{lem:distance new}).

Intuitively, \Cref{algo:zeroth-order} 
utilizes a well-known $1$-point gradient estimate technique, to get an unbiased estimate of the gradient at the chosen price by randomly querying for a (noisy) value of the function around it.   
For general strongly-concave-and-smooth function (i.e., not necessarily being quadratic), the attainable estimation error of the maximizer is $\Theta(T_1^{-\sfrac{1}{4}})$ \citep{S-13,ADX-10,JNR-12}.
Instead, by leveraging the quadratic structure 
of the revenue function $\staticRev{\cdot, \targetRefPrice}$, the gradient estimates we construct have much smaller variance compared to the one for general function, which allows us to obtain an improved estimation error $O(T_1^{-\sfrac{1}{2}})$. 

With the estimation error bound in \Cref{prop:opt greedy est error}, \Cref{prop:policy para est error}  follows immediately
by the triangle inequality. 
The missing steps to prove \Cref{prop:opt greedy est error}
are provided in \Cref{apx:proof for para estimation}.



%% file: algo-design/lipschitz-error.tex
\newcommand{\criticalTimeOpt}{\tau^*}
\newcommand{\criticalTimeEst}{\tau}

In this step, 
we establish a Lipschitz property of the revenue function $ \val^{\approxMD(\rePrice, t_1, \para)}(\rePrice, t_1)$ with respect to the policy parameter $\para$ for the case 
of symmetric reference effects, i.e., when
$\positiveRef = \negativeRef \equiv \eta$. Since in this case $\val^{\approxMD(\rePrice, t_1, \para^*)}(\rePrice, t_1) = V^*(\rePrice,t)$, this allows us to bound the 
gap $\val^*(\rePrice, t_1)
-
\val^{\approxMD(\rePrice, t_1, \para)}(\rePrice, t_1)$ 
in terms of policy parameter estimation error $\|\truePara - \para\|$.  Later, we describe how we use this 
result in the next step to establish regret for 
arbitrary  $\positiveRef, \negativeRef$.

\begin{restatable}[Bounding the Lipschitz error]{proposition}{lipschitzerrorinpolicyspace}
\label{lem:lipschitz error in policy space}
 Assume $\positiveRef =\negativeRef \equiv\eta$. Fix a starting time $t_1\in[T]$, a starting reference price $ \rePrice$ at time $t_1$. 
Then, the following holds for all $\para \in \paraSpace$,
\begin{align*}
    \val^{\approxMD(\rePrice, t_1, \para^*)}(\rePrice, t_1)
    -
    \val^{\approxMD(\rePrice, t_1, \para)}(\rePrice, t_1)
    \le
    O\left((\baseDemandParaA+\eta)\priceUB^2 \left\|\truePara - \para\right\|^2 (T-t_1) \left(\ln \sfrac{T}{t_1}\right)^2\right)~.
\end{align*}    
\end{restatable}
\begin{proof}[Proof Sketch of \Cref{lem:lipschitz error in policy space}]
A typical way to establish the Lipschitz property of a function 
is to bound the gradient of this function, namely $\nabla_\para \val^{\approxMD(\rePrice, t_1, \para)}(\rePrice, t_1)$. 
In our problem, the dependence of the total revenue function $\val^{\approxMD(\rePrice, t_1, \para)}(\rePrice, t_1)$
over the policy parameter $\para$ is through  the corresponding  
price curve $\approxMD(\rePrice, t_1, \para)$ (defined in \Cref{prop:approx markdown structure t_1}), 
and this price curve depends on the policy parameter in a highly non-trivial way. Thus, it is not clear how to directly compute and bound 
the gradient $\nabla_\para \val^{\approxMD(\rePrice, t_1, \para)}(\rePrice, t_1)$.
Instead, we use a two-step approach to establish the Lipschitz property of function $\val^{\approxMD(\rePrice, t_1, \para)}(\rePrice, t_1)$.
\squishlist
    \item \textbf{Bounding the revenue gap $\left|\val^{\policy}(\rePrice, t_1) - \val^{\policy'}(\rePrice, t_1)\right|$ via the distance on the price sequences $\left\|\policy - \policy'\right\|$}:
    We first show that when the optimal price curve 
    $\approxMD(\rePrice, t_1, \truePara)$ (i.e., the price curve defined in \Cref{prop:approx markdown structure t_1} 
    computed with true policy parameter $\truePara$) is a strict markdown price curve, 
    then the seller's total revenue function 
    $\val^{\pvec}(\rePrice, t_1)$ is strongly concave 
    over the price sequence $\pvec$ (see below \Cref{lem: strongly concavity}).
    We prove this result by bounding the eigenvalues of 
    the Hessian matrix $\nabla^2_{\pvec} \val^{\pvec}(\rePrice, t_1)$.
    We note that this result does not require that the price sequence $\policy$ is the price curve computed in \Cref{prop:approx markdown structure t_1}.
    \begin{restatable}[Strong concavity of function $\val^{\policy}$]{lemma}{stronglyconcavity}
    \label{lem: strongly concavity}
    Fix a starting time $t_1 \in [T]$ and starting reference price $\rePrice$.
    If the price curve $\approxMD(\rePrice, t_1, \truePara)$ is a strict
    markdown price curve (i.e., $\criticalTime = t_1$),
    then there exists two positive, finite constants $c_1$ and $c_2$ where
    $c_1 > c_2$,
    such that the eigenvalue $\eigenVal_\val$ of the Hessian matrix $\nabla^2_{\pvec} \val^{\pvec}(\rePrice, t_1)$ satisfies 
    $\eigenVal_\val\in 
    \left[-2(\baseDemandParaA + \eta)c_1, 
    -2(\baseDemandParaA + \eta)c_2\right]$.
    \end{restatable}

    \item 
    \textbf{Bounding the price curve distance $\left\|\approxMD(\rePrice, t_1, \para_1) - \approxMD(\rePrice, t_1, \para_2)\right\|$ 
    via the parameter distance $\|\para_1 - \para_2\|$}:
    We next establish a ``Lipschitz property'' for the  
    price curve $\approxMD(\rePrice, t_1, \para)$ in terms of the 
    differences between the policy parameters (see \Cref{lem:lipschitz error on prices}). 
    \begin{restatable}[Lipschitz error on the price curve]{lemma}{lipschitzerroronprices}
    \label{lem:lipschitz error on prices}
    Fix a starting time $t_1 \in [T]$ and starting reference price $\rePrice$.
    If price curves $\approxMD(\rePrice, t_1, \para_1)$ and 
    $\approxMD(\rePrice, t_1, \para_2)$ are both strict markdown price curves,
    then we have
    $\left\|\approxMD(\rePrice, t_1, \para_1) - \approxMD(\rePrice, t_1, \para_2)\right\|
    \le 
    O\left(\priceUB \left\|\para_1 - \para_2\right\|\sqrt{T-t_1}\ln \sfrac{T}{t_1}\right)$.
    \end{restatable}
\squishend
Combining above two results can prove \Cref{lem:lipschitz error in policy space} when the both price curves $\approxMD(\rePrice, t_1, \truePara)$ and 
$\approxMD(\rePrice, t_1, \para)$ are strictly markdown. 
On the other hand, 
as we can see from \Cref{prop:approx markdown structure t_1}, the curve 
$\approxMD(\rePrice, t_1, \truePara)$ (or $\approxMD(\rePrice, t_1, \para)$) is not necessarily a strict markdown curve, as it may keep charging the same price $\priceUB$ for initial certain rounds. 
For this case, notice there must exist a time round 
$\criticalTime\in[t_1, T]$ such that the partial price sequence 
$\approxMD(\rePrice_{\criticalTime}, \criticalTime, \truePara) = 
(p_t)_{t\ge \criticalTime}$ in $\approxMD(\rePrice, t_1, \truePara)$
must still be a strict markdown price curve 
(and similarly, it holds true for $\approxMD(\rePrice, t_1, \para)$). 
Thus, in the analysis, we also bound the gap between 
the time rounds at which the two price curves start to strictly markdown 
their prices.
Notably, we show that the revenue gap incurred due to such time round gap 
is negligible compared to the revenue gap incurred due to the 
policy parameter distance $\|\para-\truePara\|$. 
\end{proof}

%% file: algo-design/7-conclusions.tex
In this work, we study dynamic pricing problem 
where customer response to the current price is impacted by 
a reference price, which is formed by following an averaging-reference mechanism (\ARM). 
We demonstrate that a fixed-price policy is highly suboptimal in this setting, which sets it distinctively apart from the well-studied \ESM\ dynamics for reference price effects. 
We also establish the (near-)optimality of markdown pricing in \ARM\ models.
We  show that under \ARM\ with gain-seeking customers, markdown pricing is optimal, and for loss-averse customers, markdown pricing is near-optimal in the sense that the revenue achieved is within $O(\log(T))$ of the optimal revenue.

Investigating this problem further for a linear base demand model, we provide a detailed structural
characterizations of a near-optimal markdown pricing policy 
for both gain-seeking and loss-averse customers, 
along with an efficient algorithm for computing such policies.
We then study the dynamic pricing and learning problem, where the demand model
parameters are apriori unknown.
We provide an efficient learning algorithm with an asymptotically optimal revenue performance.

Below we mention a few possible avenues for future research, \wtr{from the perspective of algorithm design and customer behavior modeling, respectively.}

\xhdr{From algorithm design perspective}
What is the general characterization of the optimal pricing policy 
when the underlying base demand model is beyond linear? 
We notice that for general base demand model, the condition \eqref{eq:FOC optimality} 
for optimal prices that we derive here still holds. 
It would be interesting to explore further what additional structural 
characterizations we can infer from this condition. 
Meanwhile, on the algorithmic side, the learning part of our work also considers a linear base demand model.
Though it is already interesting and challenging enough to develop 
efficient learning algorithm for this case, it would be interesting to 
generalize our idea to more general or non-parametric demand models.
In addition, our learning algorithm is an explore-first-then-exploit type algorithm. Though this simple algorithm can already guarantee us a regret bound that has optimal dependency on the sales horizon, it would be interesting to explore whether a learning algorithm with ``adaptive exploration'' (e.g., UCB-type algorithm, Thompson Sampling) can further improve the bound, e.g., tighten the regret gap on the $\priceUB$ dependency.

\wtr{\xhdr{From customer behavior modeling perspective}
Almost all the reference price models (including ours) in current literature assume that the reference price updates depend only on the offered price (and its offered time), and not the customer demand response to those prices.  
These mechanisms could lead to an (unsatisfied) pricing strategy that the seller can set a single very large price (especially when the price upper bound is very high) to increase the reference price and lead the customer to purchase more.
One potential approach to address this is to consider models where reference price update depends on the sales that happen at the offered price. 

Secondly, in our current \ARM\ model \eqref{eq:ref dynamics via iterative}, the averaging factor is $\ESMPara_t = \sfrac{1}{t}$, while in \ESM, we have $\ESMPara_t \equiv \ESMPara$ for some constant $\ESMPara$. 
An interesting direction is to consider an intermediate setting where the averaging factor $\ESMPara_t = \sfrac{1}{t^\alpha}$ is parameterized by some rate parameter $\alpha \ge 0$ that interpolates between the \ARM\ ($\alpha = 1$) and \ESM\ ($\alpha = 0$).
One can explore the same set of questions asked in this work for this more general setting. 
For example, how does the fixed-price policy perform? One may conjecture that the total revenue from the fixed-price policy may approach the optimal total revenue gracefully as $\alpha$ goes to $0$.

Lastly, we have implicitly considered a setting where customers are myopic, i.e., they are not forward-looking and not strategically timing their purchasing decisions. Yet the markdown nature of the (near-)optimal pricing policy that we characterize may incentivize the customers to strategically decide when to enter the market and make the purchase decision.
It thus would be interesting to explore the design of optimal pricing policies in the presence of long-term reference effects and strategic customer behavior.}

%% file: algo-design/apx-proof-opt.tex
\newcommand{\criticalrePrice}{\rePrice^\dagger}

\subsection{Missing Proofs of \Crefrobust{subsec:linear regret simple}}

\linearregretfixed*
\begin{proof}[Proof of \Cref{prop:linear regret fixed}]
We consider following problem instance: 
$\negativeRef = \positiveRef \equiv \eta$, $\rePrice_1 = 0$, 
and the base demand is a linear demand $\baseDemand(p) = \baseDemandParaB - \baseDemandParaA p$.
Let $\pvec(p) = (p, \ldots, p)$  denote a fixed-price policy that keeps charging the price $p$ 
throughout the sales horizon. 
The total revenue under the fixed-price policy is $\pvec$: 
\begin{align*}
    \val^{\pvec(p)}(\rePrice_1) 
    = T p(\baseDemandParaB - \baseDemandParaA p) + 
     \sum_{t=1}^T\eta p(\rePrice_t - p)
    = 
     T p(\baseDemandParaB - \baseDemandParaA p) + 
     \eta p(\rePrice_1 - p)\sum_{t=1}^T \frac{1}{t}~.
\end{align*}
Let $p^{*, \fixed} = \argmax_{p\in[0, \priceUB]} \val^{\pvec}(\rePrice_1)$ be the optimal fixed-price,
and $\val^{*, \fixed}(\rePrice_1)$ be its corresponding total revenue, then we have
\begin{align*}
    p^{*, \fixed} = \frac{T \baseDemandParaB + \eta \rePrice_1 \sum_{t=1}^T \frac{1}{t}}{2\left(T \baseDemandParaA + \eta  \sum_{t=1}^T \frac{1}{t}\right)}, 
    \quad 
    \val^{*, \fixed}(\rePrice_1)
    = \frac{\left(T \baseDemandParaB + \eta \rePrice_1 \sum_{t=1}^T \frac{1}{t}\right)^2}{4\left(T \baseDemandParaA + \eta  \sum_{t=1}^T \frac{1}{t}\right)}
    \overset{(a)}{\le} \frac{T\baseDemandParaB^2}{4\baseDemandParaA}~,
\end{align*}
where inequality (a) is due to $\rePrice_1 = 0$.
Given a time round $T_1\in[T]$, 
and let $\alpha = \frac{T}{T_1}$.
We now consider the following non-fixed-price policy
$\policy=(p_t)_{p_t\in[T]}$ where 
$p_t = p_u\indicator{t\le \alpha T_1} + p_d\indicator{t\ge \alpha T_1+1}$, where $p_u, p_d$ are determined later. 
Under the policy $\policy$, the total revenue is
\begin{align*}
    \val^{\policy}(\rePrice_1)
    & = 
    T_1p_u(\baseDemandParaB-\baseDemandParaA p_u) + \sum_{t=1}^{T_1} \eta p_u(\rePrice_t - p_u)
    + 
    (T-T_1) p_d(\baseDemandParaB-\baseDemandParaA p_d) + \sum_{t=T_1+1}^{T} \eta p_d(\rePrice_t - p_d) \\
    & = 
    T_1 p_u(\baseDemandParaB-\baseDemandParaA p_u) + \sum_{t=1}^{T_1} \eta p_u \frac{\rePrice_1 - p_u}{t}
    + 
    (T-T_1) p_d(\baseDemandParaB-\baseDemandParaA p_d) + \sum_{t=T_1+1}^{T} \eta p_d  \frac{\rePrice_1 + T_1p_u - T_1 p_d}{t}\\
    & 
    =
    \alpha T p_u(\baseDemandParaB-\baseDemandParaA p_u) - \eta p_u^2 \ln (\alpha T)
    + 
    T(1-\alpha ) p_d(\baseDemandParaB-\baseDemandParaA p_d) + \eta p_d  \alpha T(p_u-p_d) \ln \frac{1}{\alpha}\\
    & = 
    T\cdot 
    \left(\alpha p_u(\baseDemandParaB-\baseDemandParaA p_u) 
    + 
    (1-\alpha ) p_d(\baseDemandParaB-\baseDemandParaA p_d) + \eta p_d  \alpha (p_u-p_d) \ln \frac{1}{\alpha}\right)
    - \eta p_u^2 \ln (\alpha T)~.
\end{align*}
Given the value of $\alpha$, we choose $p_u$ and $p_d$ as follows
\begin{align}
    \label{eq:two-price policy}
    p_d = \frac{(1-\alpha) \baseDemandParaB + \eta \alpha \frac{\baseDemandParaB}{2\baseDemandParaA} \ln\frac{1}{\alpha}}{2(\baseDemandParaA(1-\alpha) + \eta \alpha \ln \frac{1}{\alpha}) - (\eta\ln \frac{1}{\alpha})^2 \frac{\alpha}{2\baseDemandParaA}}, \quad 
    p_u= \frac{\baseDemandParaB + \eta p_d \ln \frac{1}{\alpha}}{2\baseDemandParaA}~.
\end{align}
Essentially, we choose the above $p_u, p_d$ such that it maximizes 
the total revenue $\val^{\policy}(\rePrice_1)$ 
under the value $\alpha$.
With the above value of $p_u, p_d$, 
let $A(\alpha) \triangleq 
\alpha p_u(\baseDemandParaB-\baseDemandParaA p_u) 
+ 
(1-\alpha ) p_d(\baseDemandParaB-\baseDemandParaA p_d) + \eta p_d  \alpha (p_u-p_d) \ln \frac{1}{\alpha}$. 
With the above definitions, we have
\begin{align*}
    \optVal(\rePrice_1) - \val^{*, \fixed}(\rePrice_1)
    & \ge 
    \val^{\policy}(\rePrice_1) - \frac{\baseDemandParaB^2 T}{4\baseDemandParaA} \\
    & = TA(\alpha) - \eta p_u^2 \ln (\alpha T) -
    \frac{\baseDemandParaB^2 T}{4\baseDemandParaA}\\
    & = T\cdot\left(A(\alpha) - \frac{\baseDemandParaB^2}{4\baseDemandParaA}\right) - \eta p_u^2 \ln (\alpha T)~.
\end{align*}
Notice that it is easy to find values for $\baseDemandParaA, \baseDemandParaB, \alpha, \eta$ such that we have
$A(\alpha) - \frac{\baseDemandParaB^2}{4\baseDemandParaA} \ge C$
for some positive constant $C > 0$. 
For example, let $\baseDemandParaB = 2, \baseDemandParaA = 1, \eta = 0.5, \alpha = 0.3$, we have $A(\alpha) - \frac{\baseDemandParaB^2}{4\baseDemandParaA} = 0.0318$.
Moreover, under the above choices of 
$\baseDemandParaA, \baseDemandParaB, \eta, \alpha$, we also have
$p_u = 1.2787$ and $p_d = 0.926$, and 
consider $\priceUB = \frac{\baseDemandParaB}{\baseDemandParaA + \eta} = 1.3333$. 
This implies that the above defined pricing policy $\policy$ 
is indeed a feasible pricing policy. 
Thus, we can conclude that 
$\optVal(\rePrice_1) - \val^{*, \fixed}(\rePrice_1)  = \Omega(T)$.
\end{proof}

\subsection{Missing Proofs of \Crefrobust{subsec:approx opt markdown}}
\label{apx-proof-approx-markdown}

\optmarkdownwithlargerposiRef*
\begin{proof}[Proof of \Cref{lem:opt markdown with larger posiRef}]
Consider a pricing policy $\policy = (p_t)_{t\in[t_1, T]}$
where the reference price at time $t$ is $\rePrice_t$ and $p_t < p_{t+1}$.
Now consider a new pricing policy $\policy' = (p_s')_{s\in[t_1, T]}$ where:
(1) $p_t' \gets p_{t+1}, p_{t+1}' \gets  p_t$;
(2) $p_s' \gets p_{s}$ for all $s\in[t_1, T]\setminus \{t, t+1\}$.
Let $\rePrice_t', \rePrice_{t+1}'$ be the induced reference price at rounds $t, t+1$ under policy $\policy'$, respectively. 
By definition, we have $\rePrice_t' = \rePrice_t$,
moreover, we also have the following observation
\begin{align}
    \label{larger ref}
    \rePrice_{t+1}' = \frac{t\rePrice_t'+p_t'}{t+1}
    = \frac{t\rePrice_t+p_{t+1}}{t+1}
    > \frac{t\rePrice_t+p_t}{t+1} = \rePrice_{t+1}~. 
\end{align}
Notice that the revenue difference between the 
policy $\policy'$ and policy $\policy$ is
\begin{align*}
    \Delta 
    & \triangleq 
    p_t'\left(\negativeRef ((\rePrice_t' - p_t') \wedge 0)+ \positiveRef ((\rePrice_t' - p_t') \vee 0)\right)
    + 
    p_{t+1}'\left(
    \negativeRef ((\rePrice_{t+1}' - p_{t+1}') \wedge 0)+ \positiveRef ((\rePrice_{t+1}' - p_{t+1}') \vee 0)
    \right) \\
    & - p_t\left(\negativeRef ((\rePrice_t - p_t) \wedge 0)+ \positiveRef ((\rePrice_t - p_t) \vee 0)\right)
    - 
    p_{t+1}\left(
    \negativeRef ((\rePrice_{t+1} - p_{t+1}) \wedge 0)+ \positiveRef ((\rePrice_{t+1} - p_{t+1}) \vee 0)
    \right) \\
    & = 
    p_{t+1}\left(\negativeRef ((\rePrice_t - p_{t+1}) \wedge 0)+ \positiveRef ((\rePrice_t - p_{t+1}) \vee 0)\right)
    + 
    p_t\left(
    \negativeRef ((\rePrice_{t+1}' - p_t) \wedge 0)+ \positiveRef ((\rePrice_{t+1}' - p_t) \vee 0)
    \right) \\
    & - p_t\left(\negativeRef ((\rePrice_t - p_t) \wedge 0)+ \positiveRef ((\rePrice_t - p_t) \vee 0)\right)
    - 
    p_{t+1}\left(
    \negativeRef ((\rePrice_{t+1} - p_{t+1}) \wedge 0)+ \positiveRef ((\rePrice_{t+1} - p_{t+1}) \vee 0)
    \right) 
\end{align*}
To analyze whether $\Delta \ge 0$,
we below consider two cases 
\begin{enumerate}
    \item When $\rePrice_t \ge p_t$. 
    Under this case, we know that 
    \begin{align*}
        \rePrice_{t+1} - \rePrice_t 
        & = \frac{t\rePrice_t+p_t}{t+1} - 
        \rePrice_t 
        = \frac{p_t - \rePrice_t}{t+1} \le 0 \\
        \rePrice_{t+1}' - p_t 
        & = \frac{t\rePrice_t +p_{t+1}}{t+1} - p_t
        = 
        \frac{t(\rePrice_t - p_t) +p_{t+1} - p_t}{t+1} > 0 
    \end{align*}
    Thus, we have
    \begin{align*}
        \Delta 
        & =  
        p_{t+1}\left(\negativeRef ((\rePrice_t - p_{t+1}) \wedge 0)+ \positiveRef ((\rePrice_t - p_{t+1}) \vee 0)\right)
        + 
        p_t\positiveRef (\rePrice_{t+1}' - p_t) \\
        & - p_t\positiveRef (\rePrice_t - p_t) 
        - 
        p_{t+1}\left(
        \negativeRef ((\rePrice_{t+1} - p_{t+1}) \wedge 0)+ \positiveRef ((\rePrice_{t+1} - p_{t+1}) \vee 0)
        \right)
    \end{align*}
    We further consider following two sub-cases
    \begin{enumerate}
        \item When $p_{t+1} \ge \rePrice_t$.
        Under this sub-case, we have
        \begin{align*}
            \rePrice_{t+1}-p_{t+1}
            & = \frac{t\rePrice_t+p_t}{t+1} - p_{t+1}
            = \frac{t(\rePrice_t-p_{t+1})+p_t-p_{t+1}}{t+1} < 0\\
            \rePrice_{t+1}' - \rePrice_t 
            & = \frac{t\rePrice_t+p_{t+1}}{t+1}- \rePrice_t 
            = \frac{p_{t+1} - \rePrice_t}{t+1} \ge 0
        \end{align*}
        Thus, we have
        \begin{align*}
            \Delta
            & =  
            p_{t+1}\negativeRef(\rePrice_t - p_{t+1})
            + 
            p_t\positiveRef (\rePrice_{t+1}' - p_t)  
            - p_t\positiveRef (\rePrice_t - p_t) 
            - 
            p_{t+1}\negativeRef(\rePrice_{t+1} - p_{t+1}) \\
            & = 
            p_{t+1}\negativeRef (\rePrice_t-\rePrice_{t+1}) 
            + p_t\positiveRef (\rePrice_{t+1}'-\rePrice_t) \ge 0
        \end{align*}
        \item When $p_{t+1} < \rePrice_t$.
        Under this sub-case, we have
        \begin{align*}
            \rePrice_{t+1}' - \rePrice_t 
            & = \frac{t\rePrice_t+p_{t+1}}{t+1}- \rePrice_t 
            = \frac{p_{t+1} - \rePrice_t}{t+1} < 0
        \end{align*}
        Thus, we have
        \begin{align*}
            \Delta
            & = 
            p_{t+1}\positiveRef (\rePrice_t - p_{t+1})
            + 
            p_t\positiveRef (\rePrice_{t+1}' - p_t) \\
            & - p_t\positiveRef (\rePrice_t - p_t) 
            - 
            p_{t+1}\left(
            \negativeRef ((\rePrice_{t+1} - p_{t+1}) \wedge 0)+ \positiveRef ((\rePrice_{t+1} - p_{t+1}) \vee 0)
            \right) 
        \end{align*}
        Now suppose we have $\rePrice_{t+1} \le p_{t+1}$, then 
        \begin{align*}
            \Delta 
            & = 
            p_{t+1}\positiveRef (\rePrice_t - p_{t+1})
            + 
            p_t\positiveRef (\rePrice_{t+1}' - p_t) 
            - p_t\positiveRef (\rePrice_t - p_t) 
            - 
            p_{t+1}\negativeRef(\rePrice_{t+1} - p_{t+1}) \\
            & = 
            p_{t+1}\positiveRef (\rePrice_t - p_{t+1})
            - 
            p_{t+1}\negativeRef(\rePrice_{t+1} - p_{t+1})
            + 
            p_t\positiveRef (\rePrice_{t+1}' - \rePrice_t)\\
            & = 
            p_{t+1}\positiveRef (\rePrice_t - p_{t+1})
            - 
            p_{t+1}\negativeRef(\rePrice_{t+1} - p_{t+1})
            - 
            p_t\positiveRef \frac{\rePrice_t - p_{t+1}}{t+1}\\
            & \overset{(a)}{>}
            p_t\positiveRef (\rePrice_t - p_{t+1})
            - 
            p_{t+1}\negativeRef(\rePrice_{t+1} - p_{t+1})
            - 
            p_t\positiveRef \frac{\rePrice_t - p_{t+1}}{t+1}\\
            & =
            p_t\positiveRef (\rePrice_t - p_{t+1})\cdot\left(1 - \frac{1}{t+1}\right)
            +
            p_{t+1}\negativeRef(p_{t+1} - \rePrice_{t+1}) \ge 0
        \end{align*}
        where in inequality (a) we use the 
        fact that $p_t < p_{t+1}$.
        
        Now suppose we have $\rePrice_{t+1} > p_{t+1}$, then 
        we have 
        \begin{align*}
            \Delta 
            & = 
            p_{t+1}\positiveRef (\rePrice_t - p_{t+1})
            + 
            p_t\positiveRef (\rePrice_{t+1}' - p_t) 
            - p_t\positiveRef (\rePrice_t - p_t) 
            - 
            p_{t+1}\positiveRef(\rePrice_{t+1} - p_{t+1}) \\
            & = 
            p_{t+1}\positiveRef (\rePrice_t - \rePrice_{t+1})
            + p_t\positiveRef (\rePrice_{t+1}' - \rePrice_t)\\
            & =  
            \positiveRef \cdot 
            \left(p_{t+1} \frac{\rePrice_t-p_t}{t+1}+
            p_t \frac{p_{t+1} - \rePrice_t}{t+1}
            \right) 
            = \positiveRef \cdot 
            \frac{\rePrice_t(p_{t+1} - p_t)}{t+1} \ge 0
        \end{align*}
        Thus, under this sub-case, we also have $\Delta \ge 0$.
    \end{enumerate}

    \item When $\rePrice_t < p_t$.
    Under this case, we have $\rePrice_t < p_t \le p_{t+1}$, and moreover
    \begin{align*}
        \rePrice_{t+1} - \rePrice_t 
        & = \frac{t\rePrice_t+p_t}{t+1} - 
        \rePrice_t 
        = \frac{p_t - \rePrice_t}{t+1} > 0 \\
        \rePrice_{t+1} - p_{t+1}
        & = 
        \frac{t\rePrice_t+p_t}{t+1} 
        - p_{t+1}
        = \frac{t(\rePrice_t-p_{t+1}) + p_t - p_{t+1}}{t+1} < 0 
    \end{align*}
    Thus, we have
    \begin{align*}
        \Delta 
        & =  
        p_{t+1}\negativeRef (\rePrice_t - p_{t+1})
        + 
        p_t\left(
        \negativeRef ((\rePrice_{t+1}' - p_t) \wedge 0)+ \positiveRef ((\rePrice_{t+1}' - p_t) \vee 0)
        \right) \\
        & - p_t\negativeRef (\rePrice_t - p_t)
        - 
        p_{t+1}\negativeRef (\rePrice_{t+1} - p_{t+1}) \\
        & = 
        p_{t+1}\negativeRef (\rePrice_t-\rePrice_{t+1}) 
        - p_t\negativeRef (\rePrice_t - p_t)
        + 
        p_t\left(
        \negativeRef ((\rePrice_{t+1}' - p_t) \wedge 0)+ \positiveRef ((\rePrice_{t+1}' - p_t) \vee 0)
        \right) 
    \end{align*}
    Now suppose we have $\rePrice_{t+1}' \le p_t$, then 
    \begin{align*}
        \Delta 
        & = 
        p_{t+1}\negativeRef (\rePrice_t-\rePrice_{t+1}) 
        - p_t\negativeRef (\rePrice_t - p_t)
        + 
        p_t\negativeRef (\rePrice_{t+1}' - p_t)\\
        & = 
        p_{t+1}\negativeRef (\rePrice_t-\rePrice_{t+1}) 
        + p_t\negativeRef (\rePrice_{t+1}' - \rePrice_t)\\
        & = \negativeRef\cdot\left(
        p_{t+1} (\rePrice_t-\rePrice_{t+1}) 
        + p_t (\rePrice_{t+1}' - \rePrice_t) 
        \right)\\
        & = 
        \negativeRef\cdot\left(
        p_{t+1} \frac{\rePrice_t-p_t}{t+1}
        + p_t \frac{p_{t+1}-\rePrice_t}{t+1}
        \right)
        = 
        \negativeRef\cdot\frac{\rePrice_t(p_{t+1}-p_t)}{t+1} \ge 0
    \end{align*}
    Now suppose we have $\rePrice_{t+1}' > p_t$, then 
    \begin{align*}
        \Delta 
        & = 
        p_{t+1}\negativeRef (\rePrice_t-\rePrice_{t+1}) 
        - p_t\negativeRef (\rePrice_t - p_t)
        + 
        p_t\positiveRef (\rePrice_{t+1}' - p_t)\\
        & \overset{(a)}{\ge}
        p_{t+1}\negativeRef (\rePrice_t-\rePrice_{t+1}) 
        - p_t\negativeRef (\rePrice_t - p_t)
        + 
        p_t\negativeRef (\rePrice_{t+1}' - p_t)\\
        & = p_{t+1}\negativeRef (\rePrice_t-\rePrice_{t+1}) 
        + 
        p_t\negativeRef (\rePrice_{t+1}' - \rePrice_t ) \ge 0
    \end{align*}
    where in inequality (a), we use the fact that $\positiveRef \ge \negativeRef$.
    Thus, under this case, we always have $\Delta \ge 0$
\end{enumerate}
Putting all pieces together, we can prove the statement.
\end{proof}

\optrevdiffdiffref*
To prove \Cref{lem:opt rev diff w.r.t diff ref}, 
we first show the following lemma which bound the revenue gap
when implementing a same pricing policy under different
starting reference prices. 
\begin{lemma}
\label{lem:rev gap fixed policy}
Fix any policy $\policy$, we have
$0\le \val^{\policy}(\rePrice', t_1) - \val^{\policy}(\rePrice, t_1) \le 
O\left(
\priceUB t_1(\rePrice'-\rePrice) (\negativeRef+\positiveRef)\ln \frac{T}{t_1}
\right)$ for any $\rePrice' \ge \rePrice$.
\end{lemma}
\begin{proof}[Proof of \Cref{lem:rev gap fixed policy}]
Fix any pricing policy $\policy = (p_t)_{t\in[t_1,T]}$, let $(\rePrice_t)_{t\in[t_1,T]}$ 
(resp.\ $(\rePrice_t')_{t\in[t_1,T]}$) be the resulting reference price path
under the starting reference price $\rePrice$ (resp.\ $\rePrice'$).
Then by definition, for any $t\in[t_1, T]$, we have
\begin{align*}
    \rePrice_t' - \rePrice_t = 
    \frac{\rePrice' t_1 + \sum_{s=t_1}^{t-1}p_s}{t} - 
    \frac{\rePrice t_1 + \sum_{s=t_1}^{t-1}p_s}{t} 
    = \frac{t_1(\rePrice' - \rePrice)}{t} \ge 0~.
\end{align*}
Thus,
\begin{align*} 
    & \val^{\policy}(\rePrice', t_1) - 
    \val^{\policy}(\rePrice, t_1)\\
    =~&  \sum_{t=t_1}^{T} p_t \cdot \left(
    \negativeRef ((\rePrice_t'-p_t)\wedge 0 - (\rePrice_t-p_t)\wedge 0) 
    + \positiveRef ((\rePrice_t'-p_t)\vee 0 - (\rePrice_t-p_t)\vee 0)\right) 
    \ge 0
\end{align*}
where the last inequality is due to the fact
that $\rePrice_t' \ge \rePrice_t $ for all $t\in[t_1,T]$. 
Moreover,
\begin{align*}
    & \val^{\policy}(\rePrice', t_1) - \val^{\policy}(\rePrice, t_1) \\
    = ~ &   
    \sum_{t=t_1}^{T} p_t \cdot \left(
    \negativeRef ((\rePrice_t'-p_t)\wedge 0 - (\rePrice_t-p_t)\wedge 0) 
    + \positiveRef ((\rePrice_t'-p_t)\vee 0 - (\rePrice_t-p_t)\vee 0)\right) \\
    \le ~ &  
    \sum_{t=t_1}^{T} p_t \cdot \left(
    \negativeRef \frac{t_1(\rePrice' - \rePrice)}{t}
    + \positiveRef \frac{t_1(\rePrice' - \rePrice)}{t} \right)
    = \priceUB t_1(\rePrice'-\rePrice) (\negativeRef+\positiveRef)\sum_{t=t_1}^{T} \frac{1}{t}\\
    = ~ &  O\left(
    \priceUB t_1(\rePrice'-\rePrice) (\negativeRef+\positiveRef)\ln \frac{T}{t_1}
    \right)~.
\end{align*}
\end{proof}
We now prove \Cref{lem:opt rev diff w.r.t diff ref}.
\begin{proof}[Proof of \Cref{lem:opt rev diff w.r.t diff ref}]

Fix any pricing policy $\policy = (p_t)_{t\in[t_1,T]}$, let $(\rePrice_t)_{t\in[t_1,T]}$ 
(resp.\ $(\rePrice_t')_{t\in[t_1,T]}$) be the resulting reference price path
under the starting reference price $\rePrice$ (resp.\ $\rePrice'$).

Let $\optPolicy(\rePrice, t_1)$ (resp.\ $\optPolicy(\rePrice', t_1)$)
be the optimal pricing policy 
under the starting reference price $\rePrice$ (resp. $\rePrice'$).
Then,
\begin{align*}
    \optVal(\rePrice', t_1) - \optVal(\rePrice, t_1)
    \ge 
    \val^{\optPolicy(\rePrice, t_1)}(\rePrice', t_1) - \val^{\optPolicy(\rePrice, t_1)}(\rePrice, t_1)
    \ge 0~,
\end{align*}
where last inequality is by \Cref{lem:rev gap fixed policy}.
%
Moreover,  
\begin{align*}
    \optVal(\rePrice', t_1) - \optVal(\rePrice, t_1) 
    & \le 
    \optVal(\rePrice', t_1) - \val^{\optPolicy(\rePrice', t_1)}(\rePrice, t_1) 
    = O\left(
    \priceUB t_1(\rePrice'-\rePrice) (\negativeRef+\positiveRef)\ln \frac{T}{t_1}
    \right)~,
\end{align*}
where last equality is by \Cref{lem:rev gap fixed policy},
thus completing the proof.
\end{proof}

\markdownforpriceUPasinitial*
\begin{proof}[Proof of \Cref{lem:markdown for priceUP as initial}]
We prove by contradiction. 
Let us fix $\rePrice_{t_1} = \priceUB$.
Suppose under a pricing policy $\policy = (p_t)_{t\in[t_1, T]}$, 
there exists a time step $k \in [t_1, T]$ such that the resulting reference price 
$\rePrice_{k}$ satisfies $\rePrice_{k} < p_{t_1}$. 
Then it implies that there exists a 
time step $s< k$ such that 
(i) $p_s < p_{k}$;
(ii) $p_t \ge p_{k}$ for all $t < s$. 
We now define 
a new pricing policy $\policy=(p_t')_{t\in[t_1, T]}$
such that it satisfies
(1) $p_s' = p_{k}, p_{k}' = p_s;$
(2) $ p_t' = p_t, \forall t\in[t_1, T]\setminus \{s, k\}$. 
Then we consider
\begin{align*}
    & \val^{\policy'}(\priceUB, t_1)
    - 
    \val^{\policy}(\priceUB, t_1) \\
    = ~ &  
    \sum_{t\in[s, k]}
    p_t'\cdot \baseDemand(p_t') + p_t'\cdot \left(\negativeRef ((\rePrice_t' - p_t') \wedge 0)+ \positiveRef ((\rePrice_t' - p_t') \vee 0)\right) \\
    & \quad - 
    \sum_{t\in[s, k]} p_t\cdot \baseDemand(p_t) + p_t\cdot
    \left(\negativeRef ((\rePrice_t - p_t) \wedge 0)+ \positiveRef ((\rePrice_t - p_t) \vee 0)\right) \\
    \overset{(a)}{\ge} ~ &  
    p_s'\cdot \baseDemand(p_s') +  
    p_s'\cdot \left(\negativeRef ((\rePrice_s' - p_s') \wedge 0)+ \positiveRef ((\rePrice_s' - p_s') \vee 0)\right) \\
    &  + 
    p_{k}'\cdot\baseDemand(p_{k}')
    + p_{k}'\cdot\left(\negativeRef ((\rePrice_{k}' - p_{k}') \wedge 0)+ \positiveRef ((\rePrice_{k}' - p_{k}') \vee 0)\right)\\
    &  - 
    p_s\cdot \baseDemand(p_s) +  p_s\cdot \left(\negativeRef ((\rePrice_s - p_s) \wedge 0)+ \positiveRef ((\rePrice_s - p_s) \vee 0)\right) \\
    &  - 
    p_{k}\cdot \baseDemand(p_{k}) + p_{k}\cdot \left(\negativeRef ((\rePrice_{k} - p_{k}) \wedge 0)+ \positiveRef ((\rePrice_{k} - p_{k}) \vee 0)\right) \\
    \overset{(b)}{=} ~ &  
    p_{k}\cdot \positiveRef (\rePrice_s - p_{k}) + 
    p_s\cdot \left(\negativeRef ((\rePrice_{k}' - p_s) \wedge 0)+ \positiveRef ((\rePrice_{k}' - p_s) \vee 0)\right)\\
    &  - 
    p_s\cdot \positiveRef (\rePrice_s - p_s)  - 
    p_{k}\negativeRef (\rePrice_{k} - p_{k}) 
\end{align*}
where inequality (a) is due to the fact that 
for any  $t\in[s+1, k-1]$, 
$\rePrice_t' - 
\rePrice_t= \frac{p_{k} - p_s}{t} > 0$,
and $p_t' = p_t$,
thus we have $\staticRev{p_t', \rePrice_t'} 
\ge \staticRev{p_t, \rePrice_t}$,
and equality (b) is due to 
$\rePrice_s \ge p_{k} > p_s$ and
$\rePrice_{k} < p_{k}$. 
Let $A \triangleq  \sum_{s=s+1}^{k-1}p_t$.
Then we notice that 
\begin{align*}
    \rePrice_{k} = \frac{s\rePrice_s + p_s + A}{k} \le \rePrice_s 
    \quad \Rightarrow 
    \quad (k-s)\rePrice_s - A \ge p_s \ge 0
\end{align*}
Let $\Delta$ denote the right-hand-side of 
the above equation (b).
We below consider two possible cases:
\begin{enumerate}
    \item 
    When $\rePrice_{k}' \ge p_s$. Under this case, we know that 
    \begin{align*}
        \Delta 
        & = 
        p_{k}\cdot \positiveRef (\rePrice_s - p_{k}) + 
        p_s\cdot \positiveRef (\rePrice_{k}' - p_s)
        - 
        p_s\cdot \positiveRef (\rePrice_s - p_s)  - 
        p_{k}\negativeRef (\rePrice_{k} - p_{k}) \\
        & = 
        p_{k}\cdot \positiveRef (\rePrice_s - p_{k}) + 
        p_s\cdot \positiveRef (\rePrice_{k}' - \rePrice_s)
        - 
        p_{k}\negativeRef (\rePrice_{k} - p_{k}) \\
        & \ge 
        p_{k}\cdot \positiveRef (\rePrice_s - p_{k}) + 
        p_s\cdot \positiveRef (\rePrice_{k}' - \rePrice_s)
        - 
        p_{k}\positiveRef (\rePrice_{k} - p_{k})\\
        & = 
        p_{k}\cdot \positiveRef(\rePrice_s - \rePrice_{k})+ 
        p_s\cdot \positiveRef (\rePrice_{k}' - \rePrice_s)\\
        & = \positiveRef\cdot\left(p_{k} \frac{k \rePrice_s - s\rePrice_s - p_s - A}{k}
        + p_s \frac{s\rePrice_s + p_{k} + A - k\rePrice_s}{k}\right) \\
        & = \positiveRef\cdot \frac{p_{k}(k -s) \rePrice_s - p_{k}p_s - p_{k}A + p_s(s-k)\rePrice_s + p_{k}p_s + p_s A
        }{k} \\
        & = 
        \positiveRef\cdot \frac{(p_{k}-p_s)((k -s) \rePrice_s - A)
        }{k} \ge 0
    \end{align*}

    \item 
    When $\rePrice_{k}' < p_s$. Under this case, we know that 
    \begin{align*}
        \Delta 
        & = 
        p_{k}\cdot \positiveRef (\rePrice_s - p_{k}) + 
        p_s\cdot \negativeRef (\rePrice_{k}' - p_s)
        - 
        p_s\cdot \positiveRef (\rePrice_s - p_s)  - 
        p_{k}\negativeRef (\rePrice_{k} - p_{k})\\
        & = 
        \positiveRef\cdot\left(p_{k}(\rePrice_s - p_{k}) - p_s(\rePrice_s - p_s)\right)
        + \negativeRef\left(p_{k}(p_{k} - \rePrice_{k}) - p_s(p_s - \rePrice_{k}')\right)
    \end{align*}
    When 
    $p_{k}(\rePrice_s - p_{k}) - p_s(\rePrice_s - p_s) \ge 0$, 
    we have
    \begin{align*}
        \Delta \ge 
        \negativeRef\left(p_{k}(p_{k} - \rePrice_{k}) - p_s(p_s - \rePrice_{k}')\right) \ge 0
    \end{align*}
    where the last inequality is due to 
    $p_{k} > p_s, \rePrice_{k} < \rePrice_{k}'$.
    
    When 
    $p_{k}(\rePrice_s - p_{k}) - p_s(\rePrice_s - p_s) < 0$, 
    with the assumption that $\negativeRef \ge \positiveRef$
    we have
    \begin{align*}
        \Delta 
        & \ge 
        \negativeRef\cdot\left(p_{k}(\rePrice_s - p_{k}) - p_s(\rePrice_s - p_s)\right)
        + \negativeRef\left(p_{k}(p_{k} - \rePrice_{k}) - p_s(p_s - \rePrice_{k}')\right) \\
        & = 
        \negativeRef\cdot \left(p_{k}(\rePrice_s - \rePrice_{k})-p_s(\rePrice_s -\rePrice_{k}')\right) \ge 0
    \end{align*}
\end{enumerate}
Putting all pieces together, we can prove the statement.
\end{proof}

\approxoptofmarkdown*
\begin{proof}[Proof of \Cref{prop:approx opt of markdown}]
When $\positiveRef \ge \negativeRef $,
\Cref{prop:approx opt of markdown} holds true due to \Cref{lem:opt markdown with larger posiRef}. 
When $\positiveRef < \negativeRef$, for this case, 
let the policy $\optPolicy(\priceUB)$ be the 
optimal pricing policy under the reference effect
$(\negativeRef, \positiveRef)$ and under the 
starting reference price $\priceUB$.
From \Cref{lem:markdown for priceUP as initial}, 
we know that policy $\optPolicy(\priceUB)$ is a markdown pricing policy.
Thus, 
\begin{align*}
    \optVal(\rePrice_1) - 
    \val^{\optPolicy(\priceUB)}(\rePrice_1)
    \overset{(a)}{\le }
    \optVal(\priceUB) - 
    \val^{\optPolicy(\priceUB)}(\rePrice_1) 
    \overset{(b)}{\le}
    O\left(
    \priceUB(\priceUB-\rePrice_1) (\negativeRef+\positiveRef)\ln T
    \right)
\end{align*}
where inequality (a) is by  \Cref{lem:opt rev diff w.r.t diff ref} with $\priceUB \ge \rePrice_1$, 
and inequality (b) is by \Cref{lem:rev gap fixed policy}.
\end{proof}


\subsection{Missing Proofs of \Crefrobust{subsec:characterize approx opt markdown}}
\label{apx:proof of approx opt}

To prove \Cref{prop:approx markdown structure}, we prove the following
generalized and reparameterized version presented in \Cref{sec:algo}:
\approxmarkdownstructurestronger*
\begin{proof}[Proof of \Cref{prop:approx markdown structure t_1}]
We first prove the optimality of price curve 
$\approxMD(\rePrice, t_1, \truePara)$ when $\positiveRef=\negativeRef$.
Then we show the near optimality of price curve 
$\approxMD(\priceUB, t_1, \truePara)$ when $\positiveRef\neq \negativeRef$.
\xhdr{The optimality of $\approxMD(\rePrice, t_1, \truePara)$ when $\positiveRef = \negativeRef$}
In the proof, we show that the optimal pricing policy $\optPolicy(\rePrice, t_1) = \approxMD(\rePrice, t_1, \truePara)$.
Let $\positiveRef = \negativeRef \equiv \eta$.
Fix a time window $[t_1,T]$ and an 
starting reference price $\rePrice_{t_1} = \rePrice$ at time $t_1$.
Recall that seller's program \ref{eq:opt Q function}
\begin{align*}
    \optVal(\rePrice, t) 
    = 
    \max_{p\in[0, \priceUB]}
    \staticRev{p, \rePrice} + \optVal\left(\frac{t\rePrice + p}{t+1}, t+1\right)~.
\end{align*}
We denote partial derivatives by using subscripts. 
By first-order optimality condition, we know that the optimal price $p_t^*$ 
must satisfy
\begin{align}
    \label{eq:foc optimality}
    \staticRev[p]{p_t^*, \rePrice} + \optVal_\rePrice\left(\frac{t\rePrice + p_t^*}{t+1}, t+1\right) \frac{1}{t+1} = 0~.
\end{align}
By envelope theorem, we have
\begin{align}
    \label{eq:envelope theorem}
    \optVal_r(\rePrice, t) = 
    \staticRev[\rePrice]{p_t^*, \rePrice} + \optVal_r\left(\frac{t\rePrice + p_t^*}{t+1}, t+1\right) \frac{t}{t+1}~.
\end{align}
From \eqref{eq:foc optimality}, we have 
$\optVal_\rePrice\left(\frac{t\rePrice + p_t^*}{t+1}, t+1\right) = -(t+1) \staticRev[p]{p_t^*, \rePrice}$,
substituting it in \eqref{eq:envelope theorem} and we get
\begin{align*}
    \optVal_\rePrice(\rePrice, t) 
    & = 
    \staticRev[\rePrice]{p_t^*, \rePrice} -t \staticRev[p]{p_t^*, \rePrice} 
    = \eta p_t^*  - t\left(p_t^*\baseDemand_p(p_t^*) + \baseDemand(p_t^*) + \eta \rePrice -2\eta p_t^*\right)~.
\end{align*}
Thus, we can also deduce that 
\begin{align*}
    \optVal_\rePrice\left(\frac{t\rePrice + p_t^*}{t+1}, t+1\right) = 
    \staticRev[\rePrice]{p_{t+1}^*, \frac{t\rePrice + p_t^*}{t+1}} - (t+1) \staticRev[p]{p_{t+1}^*, \frac{t\rePrice + p_t^*}{t+1}}~.
\end{align*}
Finally, substitute the above formula into \eqref{eq:foc optimality} and obtain a condition which does not depend on the value function anymore:
\begin{align*}
    & \staticRev[p]{p_t^*, \rePrice} + 
    \frac{1}{t+1} \staticRev[\rePrice]{p_{t+1}^*, \frac{t\rePrice + p_t^*}{t+1}} - \staticRev[p]{p_{t+1}^*, \frac{t\rePrice + p_t^*}{t+1}} = 0  
\end{align*}
where we have used the fact that $\staticRev[\rePrice]{p, \rePrice} = \eta p$. 
For base linear demand $\baseDemand(p) = \baseDemandParaB -\baseDemandParaA p$,
the above equality gives us 
\begin{align}
    \label{linear opt rule}
    \optPrice_t
    = \optPrice_{t+1} + 
    \frac{\eta\rePrice}{2(\baseDemandParaA+\eta)(t+1) +\eta}
    = \optPrice_{t+1} + 
    \frac{\trueCOne\rePrice}{t+1 +\trueCOne}~.
\end{align}
From \Cref{prop:approx opt of markdown}, we know that when
$\positiveRef = \negativeRef$, 
the optimal pricing policy 
is a markdown pricing policy.
Together with the above observation, 
we can deduce that the optimal pricing policy $\optPolicy = (p_t^*)_{t\in[t_1, T]}$
must keep charging the price as the highest possible price $\priceUB$
for until time round $\criticalTime$, and then markdowns its prices according to \eqref{linear opt rule} for the remaining rounds.

It now remains to characterize the time $\criticalTime$
and the price $\criticalPrice$. 
From \Cref{prop:approx opt of markdown} and 
\eqref{linear opt rule}, we know optimal pricing policy $\optPolicy = (p_t^*)$ must satisfy that there exists a time round $\criticalTime\in[t_1, T]$ and 
a price $\criticalPrice$ at time $\criticalTime$ such that:
\begin{equation*}
    \begin{aligned}
    p_t^*
    & = 
    \left\{\begin{array}{ll}
    \priceUB & \text{if } t  \in [t_1, \criticalTime-1]\\
    \displaystyle 
    \criticalPrice  & \text{if } t = \criticalTime\\
    \displaystyle  
    p_{t-1}^* - \frac{\trueCOne\rePrice_{t-1}^*}{t + \trueCOne}
    & \text{if } t  \in [\criticalTime+1, T-1]\\
    \displaystyle  
    \trueCOne\rePrice_t^* + \trueCTwo
    & \text{if } t  = T
    \end{array}
    \right. 
    \end{aligned}
\end{equation*}
where $(\rePrice_t^*)$ is the reference price 
sequence from optimal pricing policy.

Let $\criticalrePrice \triangleq \rePrice_{\criticalTime}^*$.
We can now roll out the above relation for the price $p_{t+1}^*$
until we can write the price $p_{t+1}$ as a function of the initial price $\criticalPrice$:
\begin{align*}
    p_{t+1}^*= A_{t+1}(\trueCOne) \criticalPrice + B_{t+1}(\trueCOne, \criticalrePrice)~.
\end{align*}
where $(A_t(\trueCOne), B_t(\trueCOne, \criticalrePrice))_{t\in[\criticalTime, T-1]}$
are defined as follows:
\begin{equation*}
    \begin{aligned}
    A_t(\trueCOne) 
    & = 
    \left\{\begin{array}{ll}
    1 & \text{if } t  = \criticalTime\\
    \displaystyle 
    A_{t-1}(\trueCOne) - \frac{\trueCOne}{t+\trueCOne} \frac{1}{t-1} \cdot \sum_{s=\criticalTime}^{t-2}A_s(\trueCOne)  & \text{if } t \in[\criticalTime+1:T-1]\\
    \displaystyle  \frac{\trueCOne\sum_{s=\criticalTime}^{T-1} A_s(\trueCOne)}{T} 
    & \text{if } t = T
    \end{array}
    \right.  \\
    B_t(\trueCOne, \criticalrePrice) 
    & = 
    \left\{\begin{array}{ll}
    0 & \text{if } t  = \criticalTime\\
    \displaystyle 
    B_{t-1}(\trueCOne, \criticalrePrice) - \frac{\trueCOne}{t+\trueCOne} \frac{1}{t-1} \cdot \left(\criticalTime \criticalrePrice +  \sum_{s=\criticalTime}^{t-2} B_s(\trueCOne, \criticalrePrice)\right)   & \text{if } t \in[\criticalTime+1:T-1]\\
    \displaystyle  
    \trueCOne\cdot\frac{\criticalTime\criticalrePrice + \sum_{s=\criticalTime}^{T-1} B_s(\trueCOne, \criticalrePrice)}{T}
    & \text{if } t = T
    \end{array}
    \right. 
    \end{aligned}
\end{equation*}
From \eqref{linear opt rule}, we can also deduce 
\begin{align*}
    \criticalPrice = \trueCTwo + \trueCOne \sum_{s = \criticalTime+1}^{T} \frac{p_s^*}{s}~.
\end{align*}
Plugging in the relation $p_t^* = A_t(\trueCOne)\criticalPrice + B_t(\trueCOne, \criticalrePrice)$
and $p_T^* = \trueCOne\rePrice_{T} + \trueCTwo$,
we can pin down the price $\criticalPrice$ as follows:
\begin{align*}
    \criticalPrice
    = \frac{ \bar{B}_{[\criticalTime,T]}(\trueCOne, \criticalrePrice)  + \trueCTwo + \frac{\trueCOne\trueCTwo}{T}}{1 -\bar{A}_{[\criticalTime,T]}(\trueCOne)}
    = p\left([\criticalTime, T], \criticalrePrice, (\trueCOne, \trueCTwo)\right) ~.
\end{align*}
where we intentionally feature the $[\criticalTime, T], \criticalrePrice, (\trueCOne, \trueCTwo)$
dependence of $p\left([\criticalTime, T], \criticalrePrice, (\trueCOne, \trueCTwo)\right)$ prominently and
\begin{equation*}
    \begin{aligned}
    \bar{A}_{[\criticalTime,T]}(\trueCOne) 
    & = 
    \trueCOne \cdot \left(
    \sum_{s=\criticalTime+1}^{T-1}\frac{A_s(\trueCOne)}{s} + 
    \frac{\trueCOne}{T^2} \sum_{s=\criticalTime}^{T-1} A_s(\trueCOne)\right)  \\
    \bar{B}_{[\criticalTime,T]}(\trueCOne, \criticalrePrice) 
    & = \trueCOne \criticalrePrice  + \trueCOne \cdot \left(
    \sum_{s=\criticalTime+1}^{T-1}\frac{B_s(\trueCOne, \criticalrePrice)}{s} + 
    \frac{\trueCOne}{T^2} \left(\criticalTime\criticalrePrice + \sum_{s=\criticalTime}^{T-1} B_s(\trueCOne, \criticalrePrice)\right)\right)~.
    \end{aligned}
\end{equation*}
Clearly, the time step $\criticalTime\in[t_1, T]$ 
in optimal pricing policy should satisfy that 
\begin{align*}
    \criticalTime =
    \min\left\{k\in[t_1,T]: p\left([\criticalTime, T], \frac{t_1\rePrice + (\criticalTime-t_1)\priceUB}{\criticalTime}, (\trueCOne, \trueCTwo)\right) \in [0, \priceUB]\right\}~.
\end{align*}
We finish the proof of this part by showing the existence of $\criticalTime$.
Let us look at the final time round $T$ 
with the starting reference price
$\rePrice_{T} = \frac{t_1\rePrice + (T-t_1)\priceUB}{T}\in[0, \priceUB]$, then according to 
the above definition, we have
\begin{align*}
    p\left([T, T], \rePrice_{T}, (\trueCOne, \trueCTwo) \right)
    = \trueCTwo + \trueCOne\rePrice_{T} 
    \overset{(a)}{\le} \priceUB
\end{align*}
where inequality (a) is due to the fact that 
$\priceUB - (\trueCTwo + \trueCOne\rePrice_{T}) \ge \priceUB - (\trueCTwo + \trueCOne\priceUB) = \frac{2\baseDemandParaA \priceUB - \baseDemandParaB + \eta \priceUB}{2(\baseDemandParaA+\eta)} \ge 0$
due to $\priceUB > \sfrac{\baseDemandParaB}{2\baseDemandParaA}$ 
from \Cref{assump:maximizer is in feasible set}.

\xhdr{The near optimality of $\approxMD(\priceUB, t_1, \trueCOne)$
when $\positiveRef\neq \negativeRef$}
In this part, we also write $\optVal(\rePrice, t_1\mid (\negativeRef, \positiveRef))$
to emphasize that it is the optimal value under the starting reference price $\rePrice$ and customers reference effect parameters $(\negativeRef, \positiveRef)$.

When $\positiveRef \ge \negativeRef$,
let $\optPolicy(\priceUB, t_1) = (p_{t, \priceUB}^*)_{t\in[t_1, T]}$
be the optimal pricing policy under starting reference price $\priceUB$
and customers' reference effects $(\positiveRef, \negativeRef)$.
Let $(\rePrice_{t, \priceUB}^*)_{t\in[T]}$ be the resulting reference price sequence
under policy $\optPolicy(\priceUB, t_1)$ and starting reference price $\priceUB$.
We first note that 
\begin{align*}
    \optVal(\priceUB, t_1 \mid (\positiveRef, \negativeRef))
    & = \sum_{t=t_1}^T p_{t, \priceUB}^* \cdot\left(\baseDemand(p_{t, \priceUB}^*)
    + \negativeRef \left(\rePrice_{t, \priceUB}^* - p_{t, \priceUB}^*\right)\wedge 0
    + \positiveRef \left(\rePrice_{t, \priceUB}^* - p_{t, \priceUB}^*\right)\vee 0\right) \\
    & \overset{(a)}{=}
    \sum_{t=t_1}^T p_{t, \priceUB}^* \cdot\left(\baseDemand(p_{t, \priceUB}^*)
    + \positiveRef \left(\rePrice_{t, \priceUB}^* - p_{t, \priceUB}^*\right)\right) \\
    & = 
    \optVal(\priceUB, t_1 \mid (\positiveRef, \positiveRef))~.
\end{align*}
where equality (a) is by \Cref{prop:approx opt of markdown} where 
we know that policy $\optPolicy(\priceUB, t_1)$ must be a markdown 
pricing policy, and given the starting reference price $\priceUB$, 
we then must have $p_{t, \priceUB}^* \le \rePrice_{t, \priceUB}^*$ for all $t\in[t_1, T]$. 
Since the optimal pricing policy when customers have symmetric 
reference effects $\positiveRef = \negativeRef$ is also a markdown policy, 
we can conclude that policy $\optPolicy(\priceUB, t_1)$ also maximizes
$\val^{\policy}(\rePrice, t_1\mid (\positiveRef, \positiveRef))$. 
In other words, we have
policy $\optPolicy(\priceUB, t_1) = \approxMD(\priceUB, t_1, \truePara)$.

Then,
\begin{align*}
    \optVal(\rePrice, t_1 \mid (\positiveRef, \negativeRef)) - 
    \val^{\approxMD(\priceUB, t_1, \truePara)}(\rePrice, t_1 \mid (\positiveRef, \negativeRef))
    & \overset{(a)}{\le} 
    \optVal(\priceUB, t_1 \mid (\positiveRef, \negativeRef)) - 
    \val^{\approxMD(\priceUB, t_1, \truePara)}(\rePrice, t_1 \mid (\positiveRef, \negativeRef)) \\
    & \overset{(b)}{\le} 
    O\left(t_1\positiveRef\priceUB(\priceUB - \rePrice) \ln \sfrac{T}{t_1}\right)
\end{align*}
where inequality (a) is by \Cref{lem:opt rev diff w.r.t diff ref} with $\rePrice \le \priceUB$, and inequality (b) is by \Cref{lem:rev gap fixed policy}.


When $\negativeRef > \positiveRef$,
follow \Cref{lem:markdown for priceUP as initial} and the similar argument above, we know that 
policy $\optPolicy(\priceUB, t_1)$ 
also maximizes $\val^{\policy}(\priceUB, t_1)
= \approxMD(\priceUB, t_1, \truePara)$.
Then we have
\begin{align*}
    \optVal(\rePrice, t_1) - 
    \val^{\approxMD(\priceUB, t_1, \truePara)}(\rePrice, t_1)
    \le \optVal(\priceUB, t_1) - \val^{\approxMD(\priceUB, t_1, \truePara)}(\rePrice, t_1)
    & = 
    \val^{\optPolicy(\priceUB, t_1)}(\priceUB, t_1)-
    \val^{\approxMD(\priceUB, t_1, \truePara)}(\rePrice, t_1)\\
    & \le 
    O\left(t_1
    \priceUB(\priceUB-\rePrice) (\negativeRef+\positiveRef)\ln \sfrac{T}{t_1}
    \right)
\end{align*}
where the last inequality is by \Cref{lem:rev gap fixed policy}.
\end{proof}

Then the proof of \Cref{prop:approx markdown structure} follows immediately. 
\begin{proof}[Proof of \Cref{prop:approx markdown structure}]
\Cref{prop:approx markdown structure}
follows by noting that $t_1 = 1$ in \Cref{prop:approx markdown structure t_1}.
\end{proof}

\begin{algorithm}[H]
\begin{algorithmic}[1]
\State \textbf{Input:} starting reference price $\rePrice$, time horizon $T$, $\truePara \gets (\trueCOne, \trueCTwo)$. 
\State \textbf{Initialization:} $t_1 \gets 1$.
\While{$t_1 < T$}
    \State $\criticalTime \gets \frac{t_1 + T}{2}$.
    \State Let $\rePrice_{\criticalTime} \gets \frac{\rePrice + (\criticalTime-1) \priceUB}{\criticalTime}$. 
    \State Solve the linear system $\Amtrx_{[\criticalTime, T]}(\truePara) \pvec = \bvec_{[\criticalTime, T], \rePrice_{\criticalTime}}(\truePara)$
    where $\Amtrx_{[\criticalTime, T]}(\truePara), \bvec_{[\criticalTime, T], \rePrice_{\criticalTime}}(\truePara)$ are defined as in \Cref{defn:linear system}.
    \State Let $t_1 \gets \frac{t_1 + \criticalTime}{2}$ if $\pvec \in [0, \priceUB]^{T-\criticalTime+1}$,
    otherwise let $t_1 \gets \frac{\criticalTime + T}{2}$.
\EndWhile
\State 
\textbf{Return} $\approxMD(\rePrice) = (p_t)_{t\in[T]}$
    where 
$p_t \gets \priceUB\indicator{t< \criticalTime} + \pvec[t-\criticalTime+1] \indicator{t\ge \criticalTime}$ for all $t\in[T]$.
\end{algorithmic}
\caption{Computing the pricing curve $\approxMD(\rePrice)$}
\label{algo:computing opt}
\end{algorithm}
\complexityapproxopt*
\begin{proof}[Proof of \Cref{cor:computational complexity approx opt}]
Notice that the price curve $\approxMD(\rePrice)$ is optimal 
to the loss-neutral customers. 
Thus, To prove \Cref{cor:computational complexity approx opt}, 
it suffices to show that there exists an algorithm that can solve for an 
optimal pricing policy for loss-neutral customers by solving 
at most $O(\ln T)$ linear systems. 

When $\positiveRef = \negativeRef = \eta$,
the derived optimality condition \eqref{linear opt rule} can 
be reformulated as follows: for any $t\in[\criticalTime, T]$
\begin{align*}
    p_t^* = 
    \trueCOne \rePrice_t^* + \trueCOne \sum_{s=t+1}^T p_s^* + \trueCTwo 
    & = 
    \trueCOne \cdot \frac{\criticalTime \criticalrePrice + \sum_{s=\criticalTime}^{t-1} p_s^*}{t} + \trueCOne \sum_{s=t+1}^T p_s^* + \trueCTwo~.
\end{align*}
As we can see, given a time round $\criticalTime$ and the reference price 
$\criticalrePrice$ at this round, 
the above condition essentially says that the optimal price $p_t^*$ 
can be represented by a linear combination over all other prices $(p_s^*)_{s\in [\criticalTime, T]\setminus\{t\}}$.
In other words, the partial price sequence $(p_t^*)_{t\in[\criticalTime, T]}$
forms a linear system where the matrix and vector in this system 
depend on model parameters $\trueCOne, \trueCTwo, \criticalTime, \criticalrePrice, T$.
In particular, the linear system can be defined as follows:
\begin{definition}[Linear system]
\label{defn:linear system}
Given a time window $[\criticalTime,T]$ and an 
starting reference price $\rePrice_{\criticalTime} = \criticalrePrice$ at time $\criticalTime$.
We define the following matrix $\Amtrx_{[\criticalTime, T]}(\truePara)$ and 
the vector $\bvec_{[\criticalTime, T], \criticalrePrice}(\truePara)$
which takes parameter $\truePara = (\trueCOne, \trueCTwo)$ (defined as in \eqref{prop:approx markdown structure t_1}) as input:
\begin{equation}
    \label{matrix defn}
    \begin{aligned}
    \displaystyle 
    \Amtrx_{[\criticalTime, T]}(\truePara)
    & =
    {\everymath={\displaystyle} 
    \left[
    \renewcommand{\arraystretch}{2}
    \centering
    \begin{tabular}
    {ccccccc}
    $1$ & $-\frac{\trueCOne}{\criticalTime+1}$ & $-\frac{\trueCOne}{\criticalTime+2}$ & $-\frac{\trueCOne}{\criticalTime+3}$  & $\ldots$ & $-\frac{\trueCOne}{T-1}$ & $-\frac{\trueCOne}{T}$ \\
    $-\frac{\trueCOne}{\criticalTime+1}$ & $1$ & $-\frac{\trueCOne}{\criticalTime+2}$ & $-\frac{\trueCOne}{\criticalTime+3}$ & $\ddots$ & $-\frac{\trueCOne}{T-1}$ & $-\frac{\trueCOne}{T}$ \\
    $-\frac{\trueCOne}{\criticalTime+2}$ & $-\frac{\trueCOne}{\criticalTime+2}$ & $1$ & $-\frac{\trueCOne}{\criticalTime+3}$ & $\ddots$ & $-\frac{\trueCOne}{T-1}$ & $-\frac{\trueCOne}{T}$ \\
    $\vdots$ & $\ddots$ & $\ddots$ & $\ddots$ & $\ddots$ & $\ddots$ & $\vdots$ \\
    $-\frac{\trueCOne}{T-2}$ & $-\frac{\trueCOne}{T-2}$ & $-\frac{\trueCOne}{T-2}$ & $\ddots$ & $1$ & $-\frac{\trueCOne}{T-1}$  & $-\frac{\trueCOne}{T}$ \\
    $-\frac{\trueCOne}{T-1}$ & $-\frac{\trueCOne}{T-1}$ & $-\frac{\trueCOne}{T-1}$ & $\ddots$ & $-\frac{\trueCOne}{T-1}$ & $1$ & $-\frac{\trueCOne}{T}$ \\
    $-\frac{\trueCOne}{T}$ & $-\frac{\trueCOne}{T}$ & $-\frac{\trueCOne}{T}$ & $\cdots$ & $-\frac{\trueCOne}{T}$ & $-\frac{\trueCOne}{T}$ & $1$
    \end{tabular} \right].} \\
    \bvec_{[\criticalTime, T], \criticalrePrice}(\truePara)
    & =
    {\everymath={\displaystyle} 
    \left[
    \centering
    \begin{tabular}
    {ccccccc}
    $\trueCOne\criticalrePrice + \trueCTwo$ & 
    $\frac{\trueCOne \criticalTime \criticalrePrice}{\criticalTime+1} + \trueCTwo$ 
    & 
    $\frac{\trueCOne \criticalTime \criticalrePrice}{\criticalTime+2} + \trueCTwo$  
    & 
    $\ldots$ 
    & 
    $\frac{\trueCOne \criticalTime \criticalrePrice}{T-1} + \trueCTwo$
    & 
    $\frac{\trueCOne \criticalTime \criticalrePrice}{T} + \trueCTwo$
    \end{tabular} \right].}
    \end{aligned}
\end{equation}

\end{definition}

So if we pin down the time round $\criticalTime$ and the reference price 
$\criticalrePrice$ at this time round, then the optimal price sequence 
for the remaining times $(p_t^*)_{t\in[\criticalTime, T]}$, 
which satisfies the condition \eqref{linear opt rule}, must also be the solution 
to the linear system $\Amtrx_{[\criticalTime, T]}(\truePara) \pvec = \bvec_{[\criticalTime, T], \criticalrePrice}(\truePara)$. 
Recall that by \Cref{prop:approx opt of markdown}, we know optimal pricing 
policy must be a markdown pricing policy, by \Cref{prop:approx markdown structure}, the reference price $\criticalrePrice$ satisfies 
$\criticalrePrice = \frac{\rePrice + \priceUB(\criticalTime-1)}{\criticalTime}$. 
Thus, to solve the optimal pricing policy $(p_t^*)_{t\in[T]}$, it
suffices to determine the time round $\criticalTime$.
Recall that by \Cref{prop:approx opt of markdown}, we know optimal pricing 
policy must be a markdown pricing policy.
Thus, the time round $\criticalTime$ is smallest time index 
such that the solution $\Amtrx_{[t, T]}(\truePara) \pvec = \bvec_{[t, T], \frac{\rePrice+(t-1)\priceUB}{t}}(\truePara)$ is a feasible solution in $[0, \priceUB]^{T-t+1}$. That is, 
\begin{align*}
    \criticalTime 
    \leftarrow
    \min\left\{t\in[T]: 
    \pvec \in [0, \priceUB]^{T-t+1} \text{ s.t. }
    \pvec \text{ solves }
    \Amtrx_{[t, T]}(\truePara) \pvec = \bvec_{[t, T], \rePrice_t}(\truePara), \rePrice_t = \frac{\rePrice + (t-1)\priceUB}{t}\right\}
\end{align*}
With the above observation, we can have a binary search algorithm to 
pin down the time round $\criticalTime$. We summarize our algorithm as in \Cref{algo:computing opt}.
\end{proof}



%% file: algo-design/apx-missing-algo.tex
\begin{algorithm}[H]
\begin{algorithmic}[1]
\State \textbf{Input:} Current time step $t$, the
reference price $\rePrice_t$ 
for this time $t$, and a target reference price $\targetRefPrice$
\If{$\rePrice_t< \targetRefPrice$}
    \State Let $N(t, \rePrice_t, \targetRefPrice) 
    \triangleq \min \{N\in \N: 
    (t+N+1)\targetRefPrice - t\rePrice_t - N\priceUB \in (0, \priceUB)\}$.
    \State 
    Keep setting the price $\priceUB$
    with $N(t, \rePrice_t, \targetRefPrice) $ rounds
    and then set the price 
    $(t+N(t, \rePrice_t, \targetRefPrice)+1)\targetRefPrice - t\rePrice_t - N(t, \rePrice_t, \targetRefPrice)\priceUB$ for $1$ round.
    \State \textbf{Return} 
    $N(t, \rePrice_t, \targetRefPrice) + 1$.
\ElsIf{ $\rePrice_t >  \targetRefPrice$}
    \State
    Let $N(t, \rePrice_t, \targetRefPrice) 
    \triangleq \min \{N\in \N: 
    (t+N+1)\targetRefPrice - t\rePrice_t \in (0, \priceUB)\}$.
    \State 
    Keep setting the price $0$
    with $N(t, \rePrice_t, \targetRefPrice) $ rounds
    and then set the price 
    $(t+N(t, \rePrice_t, \targetRefPrice)+1)\targetRefPrice - t\rePrice_t$ for $1$ round.
    \State \textbf{Return} 
    $N(t, \rePrice_t, \targetRefPrice) + 1$.
\Else 
    \State \textbf{Return}  $0$.
\EndIf
\end{algorithmic}
\caption{$\SteerRef(t, \rePrice_t, \targetRefPrice)$: 
Reset to the target reference price}
\label{algo:steering reference price}
\end{algorithm}

%% file: algo-design/apx-LB-proof.tex
\newcommand{\density}{\Lambda}
\newcommand{\fisherInfor}{\mathcal{F}}
\newcommand{\densityfisherInfor}{\widetilde{\fisherInfor}}
\newcommand{\empiricalFisherInfor}{\mathcal{J}}
\newcommand{\variance}{\sigma}

\newcommand{\baseParaSpaceLB}{\baseParaSpace'}

\subsection{Proof for \texorpdfstring{\Cref{apx:proof LB}}{}}
\label{apx:proof LB}
\lowerbound*
\begin{proof}[Proof of \Cref{thm:LB}]
Our lower bound proof is based on the proof of 
Theorem $1$ in \citet{KZ-14} with certain necessary modifications. 
We first show that all instances in $\baseParaSpaceLB$ satisfy 
\Cref{assump:maximizer is in feasible set}.
Notice that this assumption implies
\begin{align*}
    \frac{\baseDemandParaB}{2\baseDemandParaA} < \priceUB \le \frac{\baseDemandParaB}{\baseDemandParaA}~.
\end{align*}
Since $\baseDemandParaB\equiv 1$, above condition implies that 
$\frac{1}{2\priceUB} < \baseDemandParaA\le \frac{1}{\priceUB}$. 
By construction, we know $\baseDemandParaA\in[\frac{3}{4\priceUB}, \frac{1}{\priceUB}]$ for all $\basePara\in\baseParaSpaceLB$.
Thus given $\priceUB\ge 1$, all instances $\basePara\in\baseParaSpaceLB$ satisfy 
\Cref{assump:maximizer is in feasible set}.

Given an instance $\basePara \in\baseParaSpaceLB$, let $p(\basePara) = \frac{\baseDemandParaB}{2\baseDemandParaA} = \frac{1}{2\baseDemandParaA}$.
Let $\density$ be an absolutely continuous density on $\baseParaSpaceLB$, 
taking positive values on the interior of $\baseParaSpaceLB$ and zero on its boundary. 
Then, the multivariate van Trees inequality (cf.\ \citealp{GL-95}) implies that
\begin{align}
    \label{ineq:van trees}
    \expect[\density]{\newexpect{\pricingAlgo}{\basePara}{
    \left(p_t - p(\basePara)\right)^2}}
    \ge 
    \frac{\left(\expect[\density]{C(\basePara)\left(\frac{\partial p(\basePara)}{\partial \basePara}\right)^\top}\right)^2}{
    \expect[\density]{C(\basePara)\fisherInfor_{t-1}^{\pricingAlgo}(\basePara) C(\basePara)^\top }
    + \densityfisherInfor(\density)
    }~.
\end{align}
where $\densityfisherInfor(\density)$ 
is the Fisher information for the density $\density$, 
and  $\expect[\density]{\cdot}$ is the expectation operator
with respect to density $\density$.
Notice that for all $\basePara = (\baseDemandParaA, \baseDemandParaB)\in\baseParaSpaceLB$,  we have 
$C(\basePara)\left(\frac{\partial p(\basePara)}{\partial \basePara}\right)^\top
= -\frac{p(\basePara) }{2\baseDemandParaA}$, 
and 
$\fisherInfor_{t-1}^{\pricingAlgo}(\basePara) = \sfrac{\newexpect{\pricingAlgo}{\basePara}{\empiricalFisherInfor_{t-1}}}{\variance^2}$
where 
\begin{align*}
    \empiricalFisherInfor_{t} \triangleq 
    \left[\begin{array}{cc}
        t & \sum_{s=1}^t p_s^{\top} \\
        \sum_{s=1}^t p_s & \sum_{s=1}^t p_s p_s^{\top}
    \end{array}\right]
\end{align*}
denotes the empirical Fisher information, 
$(p_s)_{s\ge 1}$ are the pricing choices realized from algorithm $\pricingAlgo$.
Using these identities and adding up inequality \eqref{ineq:van trees} for $t = 2, \ldots, T$, we obtain
\begin{align*}
    \sum_{t=2}^T
    \expect[\density]{\newexpect{\pricingAlgo}{\basePara}{
    \left(p_t - p(\basePara)\right)^2}}
    \ge 
    \sum_{t=2}^T 
    \frac{\left(\expect[\density]{\frac{p(\basePara)}{2\baseDemandParaA}}\right)^2}{
    \variance^{-2}
    \expect[\density]{C(\basePara)\newexpect{\pricingAlgo}{\basePara}{\empiricalFisherInfor_{t-1}} C(\basePara)^\top }
    + \densityfisherInfor(\density)
    }~.
\end{align*}
Since $\expect[\density]{\cdot}$ is a monotone operator,
we further have
\begin{align}
    \label{ineq:LB helper}
    \sup_{\basePara\in\baseParaSpaceLB}
    \sum_{t=2}^T
    \newexpect{\pricingAlgo}{\basePara}{
    \left(p_t - p(\basePara)\right)^2}
    \ge 
    \sum_{t=2}^T 
    \frac{\inf_{\basePara\in\baseParaSpaceLB } \left(\frac{p(\basePara)}{2\baseDemandParaA}\right)^2}{
    \sup_{\basePara\in\baseParaSpaceLB}\variance^{-2}
    C(\basePara)\newexpect{\pricingAlgo}{\basePara}{\empiricalFisherInfor_{t-1}} C(\basePara)^\top 
    + \densityfisherInfor(\density)
    }~.
\end{align}
Note due to the fact that 
$C(\basePara) = [-p(\basePara) 1]$, 
we have $ C(\basePara)\newexpect{\pricingAlgo}{\basePara}{\empiricalFisherInfor_{t-1}} C(\basePara)^\top 
= \sum_{s=1}^{t-1}\newexpect{\pricingAlgo}{\basePara}{(p_s - p(\basePara))^2}$.
Thus, \eqref{ineq:LB helper} implies that 
\begin{align*}
    \sup_{\basePara\in\baseParaSpaceLB}
    \sum_{t=2}^T
    \newexpect{\pricingAlgo}{\basePara}{
    \left(p_t - p(\basePara)\right)^2}
    \ge 
    \sum_{t=2}^T 
    \frac{\variance^2 \inf_{\basePara\in\baseParaSpaceLB } \left(\frac{p(\basePara)}{2\baseDemandParaA}\right)^2}{
    \sup_{\basePara\in\baseParaSpaceLB}
    \sum_{s=1}^{t-1}\newexpect{\pricingAlgo}{\basePara}{(p_s - p(\basePara))^2}
    + \variance^2 \densityfisherInfor(\density)
    }~.
\end{align*}
Recall the definition of our regret, we then have
\begin{align*}
    \Reg[\pricingAlgo]{T, \basePara}
    = 
    \sum_{t=1}^T\newexpect{\pricingAlgo}{\basePara}{p(\basePara)(\baseDemandParaB - \baseDemandParaA p(\basePara)) - 
    p_t(\baseDemandParaB - \baseDemandParaA p_t)}
    = \baseDemandParaA 
    \sum_{t=1}^T\newexpect{\pricingAlgo}{\basePara}{
    \left(p(\basePara) - p_t\right)^2}~.
\end{align*}
Thus, 
\begin{align*}
    \sup_{\basePara\in\baseParaSpaceLB}
    \Reg[\pricingAlgo]{T, \basePara}
    & =\sup_{\basePara\in\baseParaSpaceLB}
    \baseDemandParaA 
    \sum_{t=1}^T\newexpect{\pricingAlgo}{\basePara}{
    \left(p(\basePara) - p_t\right)^2} \\
    & \ge 
    \baseDemandParaALB
    \sup_{\basePara\in\baseParaSpaceLB}
    \sum_{t=1}^T\newexpect{\pricingAlgo}{\basePara}{
    \left(p(\basePara) - p_t\right)^2} \\
    & \ge 
    \baseDemandParaALB
    \sum_{t=2}^T 
    \frac{\variance^2 \inf_{\basePara\in\baseParaSpaceLB } \left(\frac{p(\basePara)}{2\baseDemandParaA}\right)^2}{
    \sup_{\basePara\in\baseParaSpaceLB}
    \sum_{s=1}^{t-1}\newexpect{\pricingAlgo}{\basePara}{(p_s - p(\basePara))^2}
    + \variance^2 \densityfisherInfor(\density)
    }\\
    &\ge 
    \baseDemandParaALB^2
    \sum_{t=2}^T 
    \frac{\variance^2 \inf_{\basePara\in\baseParaSpaceLB } \left(\frac{p(\basePara)}{2\baseDemandParaA}\right)^2}{
    \sup_{\basePara\in\baseParaSpaceLB}
    \Reg[\pricingAlgo]{t-1, \basePara}
    + \baseDemandParaALB \variance^2 \densityfisherInfor(\density)}\\
    & \ge 
    \baseDemandParaALB^2
    \frac{ (T-1) \variance^2 \inf_{\basePara\in\baseParaSpaceLB } \left(\frac{p(\basePara)}{2\baseDemandParaA}\right)^2}{
    \sup_{\basePara\in\baseParaSpaceLB}
    \Reg[\pricingAlgo]{T, \basePara}
    + \baseDemandParaALB \variance^2 \densityfisherInfor(\density)}\\
    & =
    \frac{\baseDemandParaALB^2 (T-1) \variance^2 \inf_{\basePara\in\baseParaSpaceLB } \left(\frac{p(\basePara)}{2\baseDemandParaA}\right)^2}{
    \sup_{\basePara\in\baseParaSpaceLB}
    \Reg[\pricingAlgo]{T, \basePara}
    \left(1 + \frac{\baseDemandParaALB \variance^2 \densityfisherInfor(\density)}{\sup_{\basePara\in\baseParaSpaceLB}
    \Reg[\pricingAlgo]{T, \basePara}}\right)}~.
\end{align*}
Rearranging the above inequality we further have
\begin{align*}
    \sup_{\basePara\in\baseParaSpaceLB}
    \Reg[\pricingAlgo]{T, \basePara}
    \ge \sqrt{\frac{\baseDemandParaALB^2 (T-1) \variance^2 \inf_{\basePara\in\baseParaSpaceLB } \left(\frac{p(\basePara)}{2\baseDemandParaA}\right)^2}{1 + \frac{\baseDemandParaALB \variance^2 \densityfisherInfor(\density)}{\sup_{\basePara\in\baseParaSpaceLB}
    \Reg[\pricingAlgo]{T, \basePara}}}} 
    = 
    \frac{\baseDemandParaALB\variance \inf_{\basePara\in\baseParaSpaceLB } \frac{p(\basePara)}{2\baseDemandParaA}}{\sqrt{1 + \frac{\baseDemandParaALB \variance^2 \densityfisherInfor(\density)}{\sup_{\basePara\in\baseParaSpaceLB}
    \Reg[\pricingAlgo]{T, \basePara}}}} \cdot \sqrt{T-1}~.
\end{align*}
Now observe that 
\begin{align*}
    \sup_{\basePara\in\baseParaSpaceLB}
    \Reg[\pricingAlgo]{T, \basePara}
    \ge
    \sup_{\basePara\in\baseParaSpaceLB}
    \Reg[\pricingAlgo]{1, \basePara}
    = \frac{1}{4\baseDemandParaALB}~,
\end{align*}
where $\baseDemandParaALB \triangleq \frac{3}{4\priceUB}$.
Since we know that $\baseDemandParaB \equiv 1, \forall \basePara\in\baseParaSpaceLB$, following a standard choice of 
$\density$  (cf.\ \citealp{GZ-09}, p.\ 1632 for a choice of $\density$) over the model parameter space $\baseParaSpaceLB$
can give us
$\fisherInfor(\density) = \pi^2 (\baseDemandParaAUB - \baseDemandParaALB)^{-2}$ where $\baseDemandParaAUB \triangleq \frac{1}{\priceUB}$.
Thus, we further have 
\begin{align*}
    \sup_{\basePara\in\baseParaSpaceLB}
    \Reg[\pricingAlgo]{T, \basePara}
    & \ge 
    \frac{\variance \frac{1}{4\baseDemandParaALB}}{\sqrt{1 + 4 \baseDemandParaALB^2 \variance^2 \pi^2 (\baseDemandParaAUB-\baseDemandParaALB)^{-2}}} \sqrt{T-1}~.
\end{align*}
Recall that by our construction, we have
\begin{align*}
    \sup_{\basePara\in\baseParaSpaceLB}
    \Reg[\pricingAlgo]{T, \basePara}
    \ge 
    \frac{\variance \priceUB\sqrt{T-1}}{3\sqrt{1 + 36 \variance^2 \pi^2 }}~,
\end{align*}
which finishes the proof by considering bounded variance noises 
\end{proof}

%% file: algo-design/apx-proof-estimation.tex
\newcommand{\smallgUB}{\bar{g}}

In the analysis of this step, 
for notation simplicity we omit the superscript $\cc{GR}$ in price $\greedyPrice_s$, 
and we ignore all time steps used to reset the reference price (i.e., Line 7 in \Cref{algo:zeroth-order}). 


To prove \Cref{lem:distance new}, 
we show the following stronger result:
\begin{restatable}{lemma}{distancenew}
\label{lem:distance new stronger} 
Fix any reference price $\rePrice\in(\priceUB - \interiorGap, \priceUB]$, the greedy price $\greedyPrice(\rePrice)$
for this reference price satisfies that 
$ \greedyPrice(\rePrice) = \trueCOne\rePrice+\trueCTwo  
\in [\distance, \rePrice - \distance]$ where $\distance =\frac{1}{2}(\rePrice - \max_{\basePara\in\baseParaSpace} \frac{\baseDemandParaB}{2\baseDemandParaA})$ and $\distance < \rePrice - \distance$.
\end{restatable}
\begin{proof}[Proof of \Cref{lem:distance new stronger}]
By the first-order condition for the 
maximization problem in \eqref{eq:constrained greedy price}, we know that 
\begin{equation*}
    \begin{aligned}
    \greedyPrice(\rePrice) 
    & = 
    \left\{\begin{array}{ll}
    \displaystyle \trueCOne\rePrice+\trueCTwo & \text{if } 
    \displaystyle \trueCOne\rePrice+\trueCTwo  \le \rePrice \\
    \rePrice  & \text{if } 
    \displaystyle \trueCOne\rePrice+\trueCTwo > \rePrice
    \end{array}
    \right. 
    \end{aligned}
\end{equation*}
Below we show that when $\rePrice\in(\priceUB - \interiorGap, \priceUB]$, we always have 
$\trueCOne\rePrice+\trueCTwo < \rePrice$. To see this
\begin{align*}
    \rePrice - (\trueCOne\rePrice+\trueCTwo)
    = 
    \frac{\rePrice(2\baseDemandParaA + \positiveRef) - \baseDemandParaB}{2(\baseDemandParaA+\positiveRef)}
    = 
    (2\baseDemandParaA + \positiveRef)\frac{\rePrice - \frac{\baseDemandParaB}{2\baseDemandParaA + \positiveRef}}{2(\baseDemandParaA+\positiveRef)}
    >
    \frac{1}{2}\left(\rePrice - \max_{\basePara\in\baseParaSpace}\sfrac{\baseDemandParaB}{2\baseDemandParaA}\right) >0 ~.
\end{align*}
For an input reference price $\targetRefPrice \in (\priceUB -\interiorGap, \priceUB]$,
we thus always have 
$\greedyPrice\in [\distance, \rePrice - \distance]$.
The proof then completes.
\end{proof}

\policyparaesterror*
\begin{proof}[Proof of \Cref{prop:policy para est error}]
By triangle inequality, we know that
\begin{align*}
    \left|\estCOne - \trueCOne\right|
    \le \frac{1}{|\secondRePrice - \firstRePrice|}\cdot\left(
    \left|\estGreedy(\secondRePrice) - \greedyPrice(\secondRePrice)\right|
    + 
    \left|\estGreedy(\firstRePrice) - \greedyPrice(\firstRePrice)\right|
    \right)
    \overset{(a)}{=} 
    O\left( \frac{\priceUB\sqrt{\log\left(\log \sfrac{t}{\delta}\right) + 1}\newadjustedGradUB}{|\secondRePrice - \firstRePrice| \sqrt{T_1}} \right)
\end{align*}
where equality (a) holds true by \Cref{prop:opt greedy est error} and holds with probability at least $1-2\delta$.
Similarly, we can also derive the high-probability bound
for the estimation error $\left|\estCTwo - \trueCTwo\right|$. 
\end{proof}

\begin{lemma}
\label{lem:unbiased grad}
Let $\adjustedGrad_t
\triangleq \frac{p_t \demand_t(p_t, \rePrice)}{\distance}\randomDirec$,
where $p_t$ is the chosen price in Line 10 in \Cref{algo:zeroth-order},
then we have 
$\expect{\adjustedGrad_t} = \frac{\partial\staticRev{p_t,\rePrice}}{\partial p_t}$.
\end{lemma}
\begin{proof}[Proof of \Cref{lem:unbiased grad}]
By the way $\randomDirec$ is chosen, we have that 
$\expect{\randomDirec} = 0, \expect{\randomDirec^2} = 1$ and $\expect{\randomDirec^3} = 0$.
Recall that by the design of the algorithm, we must have
$p_t\le \rePrice$. 
Thus, we have $\demand(p_t, \rePrice) = \baseDemandParaB-\baseDemandParaA p+\positiveRef(\rePrice-p_t)$.
\begin{align*}
    \expect{\adjustedGrad_t}
    & = \expect{\frac{p_t \demand_t(p_t, \rePrice)}{\distance}\randomDirec} \\
    & = 
    \expect{\frac{p_t\randomDirec}{\distance} \cdot \left(\demand(p_t, \rePrice) + \shock_t\right)}\\
    & = 
    \expect{\frac{p_t\randomDirec }{\distance}\cdot \left(\baseDemand(p_t) + \positiveRef p_t(\rePrice - p_t) + \shock_t\right)} \\
    & = 
    \expect{\frac{\randomDirec}{\distance} \cdot \left(- (\baseDemandParaA + \positiveRef)p_t^2 + (\baseDemandParaB-\positiveRef\rePrice)\chosenPrice + \shock_tp_t\right)}\\
    & = 
    \expect{\frac{\randomDirec}{\distance} \cdot \left(- (\baseDemandParaA + \positiveRef)\left(p_t + \distance\randomDirec\right)^2 + (\baseDemandParaB-\positiveRef\rePrice)\left(p_t + \distance\randomDirec\right) + \shock_t\left(p_t + \distance\randomDirec\right)\right)} \\
    & = 
    \expect{\frac{1}{\distance} \left(-(\baseDemandParaA + \positiveRef)\distance^2 \randomDirec^3 
    + 
    \left(-(\baseDemandParaA + \positiveRef)p_t^2 + p_t(\distance -\positiveRef\rePrice) + \shock_tp_t\right)\randomDirec\right)}\\
    & \quad + 
    \expect{
    \left(-2(\baseDemandParaA+\positiveRef)p_t + \baseDemandParaB - \positiveRef\rePrice \right)\randomDirec^2}
    + \expect{\shock_t\randomDirec^2}\\
    & = 0 + \left(-2(\baseDemandParaA+\positiveRef)p_t + \baseDemandParaB - \positiveRef\rePrice\right) = \frac{\partial \staticRev{p_t,\rePrice}}{\partial p_t}
\end{align*}

\end{proof}

\begin{lemma}
\label{lem:bound the gap via lipschitz}
Fix a $\rePrice$ as the input of \Cref{algo:zeroth-order}, the 
one-shot revenue function $\staticRev{\cdot,\rePrice}$ is 
$\lipConstEqualRef$-Lipschitz where $\lipConstEqualRef \triangleq \baseDemandParaB + \maxRef\priceUB$.
And moreover for any time $t$, we have that $\left|p_t - \greedyPrice(\rePrice) \right| \le \sfrac{2\lipConstEqualRef}{\strongConcavPara}$ holds with probability $1$
where $\strongConcavPara \triangleq \sfrac{1}{2\priceUB}$.
\end{lemma}
\begin{proof}[Proof of \Cref{lem:bound the gap via lipschitz}]
We first show that fix any reference price $\rePrice\in[0, \priceUB]$, 
the one-shot revenue function $\staticRev{\cdot,\rePrice}$ is 
$(\baseDemandParaBUB + \maxRef\priceUB)$-Lipschitz. 
Notice that 
\begin{align*}
    \left|\frac{\partial \staticRev{p,\rePrice} }{\partial p}\right|
    & = 
    \left|\baseDemandParaB - 2(\baseDemandParaA + \positiveRef)p + \positiveRef \rePrice\right|\\
    & \overset{(a)}{\le} 
    \left|\baseDemandParaB - 2(\baseDemandParaA + \positiveRef)p\right| + \positiveRef \rePrice\\
    & \overset{(b)}{\le}  \max\{\baseDemandParaB, 
    2(\baseDemandParaA + \positiveRef)\priceUB - \baseDemandParaB\} + \positiveRef \rePrice \\
    & \overset{(c)}{=} \baseDemandParaB + \positiveRef\rePrice
    \le \baseDemandParaBUB + \maxRef\priceUB 
\end{align*}
where inequality (a) holds true by triangle inequality;
inequality (b) holds true by \Cref{assump:maximizer is in feasible set} which 
implies that $\priceUB > \sfrac{\baseDemandParaB}{2\baseDemandParaA}$, 
and thus $2(\baseDemandParaA + \positiveRef)\priceUB - \baseDemandParaB > 0$;
and equality (c) holds true by 
$\baseDemandParaB -\left(2(\baseDemandParaA + \positiveRef)\priceUB - \baseDemandParaB\right) = 2(\baseDemandParaB - (\baseDemandParaA + \positiveRef)\priceUB)\ge 0$ by 
$\demand(\priceUB,\rePrice) \ge 0$. 

For any reference price $\rePrice\in[0, \priceUB]$, 
we know that the function $\staticRev{p,\rePrice}$ is a  $-\strongConcavPara$-strongly-concave function. 
Using strong concavity, we have
\begin{align*}
    \frac{\strongConcavPara}{2}
    (p_t - \greedyPrice(\rePrice))^2 
    & \le 
    \frac{\partial \staticRev{p_t,\rePrice}}{\partial p_t} \cdot (\greedyPrice(\rePrice) - p_t)\\
    & \le 
    \left|\frac{\partial \staticRev{p_t,\rePrice}}{\partial p_t}\right|\cdot |\greedyPrice(\rePrice) - p_t|\\
    & \le \lipConstEqualRef  |\greedyPrice(\rePrice) - p_t|~.
\end{align*}
If $|\greedyPrice(\rePrice) - p_t|  = 0$, 
then the statement trivially holds true, otherwise we have $\left|p_t - \greedyPrice(\rePrice) \right| \le 
\sfrac{2\lipConstEqualRef}{\strongConcavPara}$ by dividing $\left|p_t - \greedyPrice(\rePrice) \right|$ in both sides of the above inequality.  
\end{proof}

\begin{lemma}
\label{lem:error bound via grad diff}
For any $t\ge 2$, we have the following holds
\begin{align*}
    \left|p_{t+1} - \greedyPrice(\rePrice)\right|^2 
    \le 
    \frac{2}{\strongConcavPara} \frac{1}{t(t-1)} \sum_{i=2}^t (i-1)\gradientDiff_i(p_i - \greedyPrice(\rePrice))  + \frac{1}{2\strongConcavPara^2}\frac{\sum_{i=2}^t \adjustedGrad_i^2}{t(t-1)}
\end{align*}
where $\gradientDiff_i \triangleq \adjustedGrad_i - \gradient_i$ for all $i\in[t]$.
\end{lemma}
\begin{proof}[Proof of \Cref{lem:error bound via grad diff}]
By the strong concavity of the revenue function $\staticRev{p,\rePrice}$ and $\greedyPrice(\rePrice)$ is the maximizer of the function $\staticRev{p,\rePrice}$, 
we have
\begin{align}
    \label{ineq:strong concave LB}
    \frac{\strongConcavPara}{2} \left(p_t - \greedyPrice(\rePrice)\right)^2 + 
    \staticRev{\greedyPrice(\rePrice),\rePrice} - \staticRev{p_t,\rePrice}
    \le \gradient_t(\greedyPrice(\rePrice) - p_t)
\end{align}
and moreover, 
\begin{align}
    \label{ineq:strong concave UB}
    \staticRev{\greedyPrice(\rePrice),\rePrice} - \staticRev{p_t,\rePrice}
    \ge \frac{\strongConcavPara}{2} (p_t - \greedyPrice(\rePrice))^2~.
\end{align}
Notice that for any price $p'$ and any $p\in[0, \priceUB]$ we have $\left|\Proj_{[\distance, \priceUB - \distance]}(p') - p\right| \le |p'-p|$.
We also define $\gradientDiff_t 
\triangleq \adjustedGrad_t - \gradient_t$.
Recall that from \Cref{lem:unbiased grad}, we know
$\expect{\gradientDiff_t} = 0$.
With these inequalities and observations, we have the following:
\begin{align*}
    \left(p_{t+1} - \greedyPrice(\rePrice)\right)^2
    & \le 
    \left(\Proj_{[\distance, \rePrice - \distance]}\left(p_t + \frac{1}{\strongConcavPara t} \adjustedGrad_t\right) - \greedyPrice(\rePrice)\right)^2 \\
    & \le 
    \left(p_t + \frac{1}{\strongConcavPara t} \adjustedGrad_t - \greedyPrice(\rePrice)\right)^2  \\
    & = 
    (p_t - \greedyPrice(\rePrice))^2
    + 2(p_t - \greedyPrice(\rePrice))\frac{1}{\strongConcavPara t} \adjustedGrad_t 
    + \left(\frac{1}{\strongConcavPara t} \adjustedGrad_t\right)^2\\
    & = 
    (p_t - \greedyPrice(\rePrice))^2
    + 2(p_t - \greedyPrice(\rePrice))\frac{1}{\strongConcavPara t} \gradient_t 
    +
    2(p_t - \greedyPrice(\rePrice))\frac{1}{\strongConcavPara t} \gradientDiff_t 
    + \left(\frac{1}{\strongConcavPara t} \adjustedGrad_t\right)^2\\
    & \overset{(a)}{\le} 
    (p_t - \greedyPrice(\rePrice))^2
    + 2\frac{1}{\strongConcavPara t} 
    \left(-\frac{\strongConcavPara}{2}(p_t - \greedyPrice(\rePrice))^2 + \staticRev{p_t,\rePrice} - \staticRev{\greedyPrice(\rePrice),\rePrice}\right) \\
    & \quad + 
    2(p_t - \greedyPrice(\rePrice))\frac{1}{\strongConcavPara t} \gradientDiff_t 
    + \left(\frac{1}{\strongConcavPara t} \right)^2 \adjustedGrad_t^2\\
    & \overset{(b)}{\le} 
    (p_t - \greedyPrice(\rePrice))^2
    -\frac{2}{t}(p_t - \greedyPrice(\rePrice))^2
    + 
    2(p_t - \greedyPrice(\rePrice))\frac{1}{\strongConcavPara t} \gradientDiff_t 
    + \left(\frac{1}{\strongConcavPara t} \right)^2 \adjustedGrad_t^2\\
    & =
    \left(1 - \frac{2}{t}\right) (p_t - \greedyPrice(\rePrice))^2 
    + 
    2(p_t - \greedyPrice(\rePrice))\frac{1}{\strongConcavPara t} \gradientDiff_t 
    + \left(\frac{1}{\strongConcavPara t} \right)^2 \adjustedGrad_t^2~,
\end{align*}
where inequality (a) holds true by \eqref{ineq:strong concave LB};
and inequality (b) holds true by \eqref{ineq:strong concave UB}.
We unfold the above recursive inequality 
till $t = 2$,
we then have for any $t\ge 2$,
\begin{align*}
    \left(p_{t+1} - \greedyPrice(\rePrice)\right)^2
    & \le 
    \frac{2}{\strongConcavPara} \sum_{i=2}^t \frac{1}{i} \left(\prod_{j=i+1}^t \left(1 - \frac{2}{j}\right)\right) \gradientDiff_i(p_i - \greedyPrice(\rePrice)) + \frac{1}{\strongConcavPara^2}\sum_{i=2}^t\frac{\adjustedGrad_i^2}{i^2}
    \prod_{j=i+1}^t \left(1 - \frac{2}{j}\right) \\
    & \overset{(a)}{=} 
    \frac{2}{\strongConcavPara} \sum_{i=2}^t \frac{1}{i} \frac{i(i-1)}{t(t-1)} \gradientDiff_i(p_i - \greedyPrice(\rePrice)) + \frac{1}{\strongConcavPara^2}\sum_{i=2}^t\frac{\adjustedGrad_i^2}{i^2}
    \frac{i(i-1)}{t(t-1)} \\
    & = 
    \frac{2}{\strongConcavPara} \frac{1}{t(t-1)} \sum_{i=2}^t (i-1)\gradientDiff_i(p_i - \greedyPrice(\rePrice))  
    + 
    \frac{1}{\strongConcavPara^2}\frac{\sum_{i=2}^t \adjustedGrad_i^2}{t(t-1)}
\end{align*}
where equality (a) uses the observation that 
$\prod_{j=i+1}^t \left(1 - \frac{2}{j}\right)
= \frac{i(i-1)}{t(t-1)}$,
thus completing the proof.
\end{proof}
To finish the proof of \Cref{prop:opt greedy est error}, 
we also need the following technique lemma:
\begin{lemma}[See \citealp{RSS-12}]
\label{lem:margingale diff}
Let $\delta_1, \ldots, \delta_T$ be a martingale difference sequence with a uniform bound $|\delta_i| \le b$ for all $i$. 
Let $V_s = \sum_{t=1}^s\Variance[t-1]{\delta_t}$ be the sum of conditional variances of $\delta_t$'s. 
Further, let $\sigma_s = \sqrt{V_s}$. 
Then we have, for any $\delta\le \sfrac{1}{e}$ and $T \ge 4$,
\begin{align*}
    \prob{\sum_{t=1}^s \delta_t>2 \max \left\{2 \sigma_s, b \sqrt{\log (1 / \delta)}\right\} \sqrt{\log (1 / \delta)} \quad \text { for some } s \leq T} \leq \log (T) \delta
\end{align*}
\end{lemma}
We are now ready to prove \Cref{prop:opt greedy est error}.
\optgreedyesterror*
\begin{proof}[Proof of \Cref{prop:opt greedy est error}]
In below proof, we fix a reference price $\rePrice\in[0, \priceUB]$ as the input 
to the \Cref{algo:zeroth-order}. 
We start our analysis from \Cref{lem:error bound via grad diff}. 
Let $z_t \triangleq \gradientDiff_t(p_t - \greedyPrice(\rePrice))$. Notice that 
by definition, we have
$\expect{z_t \mid (p_s, \chosenPrice_s, \adjustedGrad_s, \demand_s)_{s\in[t-1]}} = 0$. 
Recall that 
$|\gradientDiff_t| = |\adjustedGrad_t-\gradient_t| \le |\adjustedGrad_t| + |\gradient_t| \le|\adjustedGrad_t| + (\baseDemandParaBUB + \maxRef\priceUB)$.
For notation simplicity, 
let $\lipConstEqualRef \triangleq \baseDemandParaBUB + \maxRef\priceUB$. 
By Cauchy-Schwartz inequality, we also know 
that the conditional variance 
$\Variance[t-1]{z_t} \triangleq \expect{(z_t - \expect{z_t})^2 \mid (p_s, \chosenPrice_s, \adjustedGrad_s, \demand_s)_{s\in[t-1]}}
\le (|\adjustedGrad_t| + \lipConstEqualRef)^2(p_t-\greedyPrice(\rePrice))^2$.
Thus, we have the following holds about the
sum of conditional variances:
\begin{align*}
    \sum_{i=2}^t \Variance[i-1]{(i-1)z_i}
    \le 
    \sum_{i=2}^t (i-1)^2(|\adjustedGrad_i| + \lipConstEqualRef)^2(p_i-\greedyPrice(\rePrice))^2
\end{align*}
For $i\in[t]$, we also have the following uniform bound
\begin{align*}
    \left|(i-1)z_t\right|
    \le 
    (t-1)\cdot|\gradientDiff_t|\cdot|p_t-\greedyPrice(\rePrice)| 
    \overset{(a)}{\le }
    \frac{2(t-1)(|\adjustedGrad_t| + \lipConstEqualRef)\lipConstEqualRef}{\strongConcavPara}~.
\end{align*}
where inequality (a) holds true by \Cref{lem:bound the gap via lipschitz}.
Thus with \Cref{lem:margingale diff}, 
when $T \ge 4, \delta < \sfrac{1}{e}$, we have
with probability at least $1- \delta$, 
the following holds for all $t\le T$:
\begin{align*}
    & \sum_{i=2}^t(i-1) z_i \\
    \leq ~ & 
    4\max \left\{\sqrt{\sum_{i=2}^t(i-1)^2(|\adjustedGrad_i| + \lipConstEqualRef)^2\left|p_i-\greedyPrice(\rePrice)\right|^2}, 
    \frac{\lipConstEqualRef(t-1)(|\adjustedGrad_t| + \lipConstEqualRef) }{\strongConcavPara} \sqrt{\log \left(\frac{\log (T)}{\delta}\right)}\right\} \sqrt{\log \left(\frac{\log (T)}{\delta}\right)}
\end{align*}
Thus, back to \Cref{lem:error bound via grad diff}, we have
the following holds with probability at least 
$1 - \delta$
\begin{align}
    & (p_{t+1} - \greedyPrice(\rePrice))^2 \nonumber \\
    \le ~ &
    \frac{2}{\strongConcavPara} \frac{1}{t(t-1)} \sum_{i=2}^t (i-1)z_i  
    + 
    \frac{1}{\strongConcavPara^2}\frac{\sum_{i=2}^t \adjustedGrad_i^2}{t(t-1)}  \nonumber
    \\
    \overset{(a)}{\le} ~ & 
    \frac{8\sqrt{\log(\frac{\log T}{\delta})}}{\strongConcavPara t(t-1)}
    \sqrt{\sum_{i=2}^t(i-1)^2(|\adjustedGrad_i|+\lipConstEqualRef)^2\left|p_i-\greedyPrice(\rePrice)\right|^2} 
    + 
    \frac{8(|\adjustedGrad_t|+\lipConstEqualRef)\lipConstEqualRef\log(\frac{\log T}{\delta})}{\strongConcavPara^2 t} 
    + 
    \frac{1}{\strongConcavPara^2}\frac{\sum_{i=2}^t \adjustedGrad_i^2}{t(t-1)} 
    \label{ineq:bound via gradient UB}
\end{align}
where inequality (a) uses \Cref{lem:margingale diff}.
When the noise $\{\shock_i\}$ 
follows a bounded distribution with uniform bound $\shockUB$, then we have
\begin{align*}
    \left|\adjustedGrad_t\right|^2
    = 
    \frac{\randomDirec^2}{\distance^2}\left(p_t(\demand(p_t,\rePrice) + \shock_t)\right)^2
    & = 
    \frac{1}{\distance^2}\left(p_t(\demand(p_t,\rePrice) + \shock_t)\right)^2 \\
    & \le 
    \frac{2}{\distance^2}\left(p_t^2\demand(p_t, \rePrice)^2 + p_t^2\shock_t^2\right)\\
    & \le
    \frac{2}{\distance^2}\left(\revUB^2 + \priceUB^2\shockUB^2\right)
    \triangleq \smallgUB^2~.
\end{align*}
where $\revUB \triangleq \max_{p, \rePrice} \staticRev{p, \rePrice}$.
Back to \eqref{ineq:bound via gradient UB}, we have
\begin{align*}
    \text{RHS of } \eqref{ineq:bound via gradient UB} 
    & \le 
    \frac{8(\smallgUB + \lipConstEqualRef)\sqrt{\log(\frac{\log T}{\delta})}}{\strongConcavPara t(t-1)}
    \sqrt{\sum_{i=2}^t(i-1)^2\left|p_i-\greedyPrice(\rePrice)\right|^2} 
    + 
    \frac{8(\smallgUB+\lipConstEqualRef)\lipConstEqualRef\sqrt{\log(\frac{\log T}{\delta})} + \smallgUB^2}{\strongConcavPara^2 t} \\
    & \le 
    \frac{8\newadjustedGradUB\log(\frac{\log T}{\delta})}{\strongConcavPara t(t-1)}
    \sqrt{\sum_{i=2}^t(i-1)^2\left|p_i-\greedyPrice(\rePrice)\right|^2} 
    + 
    \frac{16 \newadjustedGradUB^2\log(\frac{\log T}{\delta}) + \newadjustedGradUB^2}{\strongConcavPara^2 t} \\
    & \le 
    O\left( \frac{\left(\log\left(\log \sfrac{t}{\delta}\right) + 1\right)\newadjustedGradUB^2}{\strongConcavPara^2 t} \right)
\end{align*}
where we have 
$\newadjustedGradUB\triangleq \smallgUB \vee \lipConstEqualRef
= \sqrt{\frac{2}{\distance^2}(\revUB^2 + \priceUB^2 \shockUB^2)} \vee (\baseDemandParaBUB + \maxRef\priceUB) 
\le \frac{\sqrt{2\priceUB(1+\priceUB)}}{\distance} \vee (1+\priceUB)$, 
and in the last step, we use
an induction argument to prove the desired inequality. 
Notice that by \Cref{assump:maximizer is in feasible set}, we know
$\priceUB > \frac{\baseDemandParaA}{2\baseDemandParaA}$.
Thus we have $\baseDemandParaA \ge \frac{\baseDemandParaB}{2\priceUB}$
for all $\baseDemandParaA, \baseDemandParaB$.
Since we know $\max_{\basePara\in\baseParaSpace} \baseDemandParaB = 1$,
and $\max_{\basePara\in\baseParaSpace} \baseDemandParaA = 1$, we know that 
$\priceUB \ge \sfrac{1}{2}$.
Thus we have $\newadjustedGradUB\le \sfrac{3\priceUB}{\distance}$.
\end{proof}

%% file: algo-design/apx-proof-smoothing.tex
\begin{proof}[Proof of \Cref{lem:bound smooth in exploration}]
We bound the total number of rounds used in \Cref{algo:steering reference price} in Exploration phase 1, similar analysis can be carried over to the Exploration phase 2. 
When $\rePrice_t < \firstRePrice$, we have the following inequality on the 
number of rounds used $N(t, \rePrice_t, \firstRePrice)$:
\begin{align*}
    (t+N(t, \rePrice_t, \firstRePrice)+1)\firstRePrice - t\rePrice_t - N(t, \firstRePrice, \secondRePrice)\priceUB \ge 0
\end{align*}
Thus, we have
\begin{align*}
    N(t, \rePrice_t, \firstRePrice)
    \le \frac{(t+1)\firstRePrice - t\rePrice_t}{\priceUB - \firstRePrice}
    \overset{(a)}{=}
    \frac{(t+1)\firstRePrice - ((t-1)\firstRePrice + p_{t-1})}{\priceUB - \firstRePrice} 
    = 
    \frac{2\firstRePrice - p_{t-1}}{\priceUB - \firstRePrice} 
    \le 
    \frac{2\firstRePrice}{\priceUB - \firstRePrice} \le \frac{2}{3}
\end{align*}
where in equality (a) we have used the fact that whenever we have $\rePrice_t \neq \firstRePrice$, one must have that $\rePrice_{t-1} = \firstRePrice$,
and in last inequality, we have used the fact that $\firstRePrice = \sfrac{\priceUB}{4}$.

When $\rePrice_t > \firstRePrice$, we have the following inequality on the 
number of rounds used $N(t, \rePrice_t, \firstRePrice)$:
\begin{align*}
    (t+N(t, \rePrice_t, \firstRePrice)+1)\firstRePrice - t\rePrice_t < \priceUB
\end{align*}
Thus we have
\begin{align*}
    N(t, \rePrice_t, \firstRePrice) 
    < \frac{t\rePrice_t}{ \firstRePrice} - (t+1)
    \overset{(a)}{=}
    \frac{(t-1)\firstRePrice +p_{t-1} }{\firstRePrice} - (t+1)
    = \frac{p_{t-1} - 2\priceUB}{\firstRePrice} \le 0
\end{align*}
where in equality (a) we have used the fact that whenever we have $\rePrice_t \neq \firstRePrice$, one must have that $\rePrice_{t-1} = \firstRePrice$,
and in last inequality, we have used the fact that $\firstRePrice = \sfrac{\priceUB}{4}$.
Thus, in both cases, we have that $N(t, \rePrice_t, \firstRePrice) = 0$ for any $t$ 
in Exploration phase 1. 
Similarly, we can also show that 
$N(t, \rePrice_t, \firstRePrice) = \Theta(1)$ for any $t$ 
in Exploration phase 1. 
\end{proof}

We next bound the number rounds used in \SteerRef\ for resetting 
reference price $\firstRePrice$ to be $\secondRePrice$.
Let $t$ be the time round that enters in the smoothing phase. 
From \Cref{lem:bound smooth in exploration}, we know that $t = \Theta(T_1)$. 
Since $\secondRePrice > \firstRePrice$, we have the following inequality on the 
number of rounds used $N(t, \firstRePrice, \secondRePrice)$:
\begin{align*}
    (t+N(t, \firstRePrice, \secondRePrice)+1)\secondRePrice - t\firstRePrice - N(t, \firstRePrice, \secondRePrice)\priceUB \ge 0
\end{align*}
Thus, we have
\begin{align*}
    N(t, \firstRePrice, \secondRePrice)
    \le \frac{(t+1)\secondRePrice - t\firstRePrice}{\priceUB - \secondRePrice}
    \overset{(a)}{=}
    \frac{(t+1)\frac{\priceUB}{2} - \frac{t\priceUB}{4}}{\frac{1}{2}\priceUB }
    = O(t) = O(T_1)
\end{align*}
which completes the proof.

%% file: algo-design/apx-proof-Lipschitz.tex
\stronglyconcavity*
\begin{proof}[Proof of \Cref{lem: strongly concavity}]
Given a pricing sequence $\pvec = (p_t)_{t\in[t_1, T]}$,
we can write the value function as follows
\begin{align*}
    \val^{\pvec}(\rePrice, t_1)
    = 
    \sum_{t=t_1}^{T}
    p_t\left(\baseDemandParaB - (\baseDemandParaA + \eta) p_t + \eta \rePrice_t\right)
    =
    2(\baseDemandParaA + \eta)
    \sum_{t=t_1}^{T}
    p_t\left(\trueCTwo - \frac{1}{2} p_t + \trueCOne \rePrice_t\right)~.
\end{align*}
Taking the first-order derivative of function 
$\val^{\pvec}(\rePrice, t_1)$ over each price $p_t$, 
\begin{align*}
    \frac{\partial \val^{\pvec}(\rePrice, t_1)}{\partial p_t}
    & = 
    2(\baseDemandParaA + \eta) \left(\trueCTwo - p_t + \trueCOne\left(\rePrice_t + \sum_{s=t+1}^{T}\frac{p_s}{s}\right)\right) \\
    & = 
    2(\baseDemandParaA + \eta) \left(\trueCTwo - p_t + \trueCOne S_t\right)~,
\end{align*}
where $S_t 
\triangleq \rePrice_t  + \sum_{s=t+1}^{T}\frac{p_s}{s}
= \frac{t_1 \rePrice + \sum_{s=t_1}^{t-1}p_s }{t} + \sum_{s=t+1}^{T}\frac{p_s }{s}$. 
The Hessian matrix of $\val^{\pvec}(\rePrice, t_1)$ equals to 
\begin{equation}
    \displaystyle 
    \Hessian_{\val}\left(\pvec\right)=
    2(\baseDemandParaA+\eta) \cdot 
    \Amtrx_{[t_1, T]}(\truePara)~.
\end{equation}
In below, we will show that the matrix $-\Hessian_{\val}\left(\pvec\right)$ is a strictly diagonally dominant matrix. 
Notice that in matrix $-\Hessian_{\val}\left(\pvec\right)$, the sum of all non-diagonal entries in a row is decreasing when row index increases, and all 
diagonal entries have the value $1$. 
Thus, to show the strictly diagonal dominance of matrix 
$-\Hessian_{\val}\left(\pvec\right)$, 
it suffices to show that 
\begin{align*}
    \trueCOne \sum_{s=t_1+1}^{T} \frac{1}{s} < 1~.
\end{align*}
To prove the above inequality, notice that 
the Hessian matrix $\Hessian_{\val}\left(\pvec\right)$ 
is exactly the matrix $\Amtrx_{[t_1, T]}(\truePara)$. Let $\pvec^* = (p_t^*)_{t\in[t_1, T]}$ be the solution to the following linear systems: $\Amtrx_{[t_1, T]}(\truePara) \pvec = \bvec_{[t_1, T], \rePrice}(\truePara)$.
By definition, we then have
\begin{align}
    \label{ineq helper3}
    p_{t_1}^* - \sum_{s=t_1+1}^{T} \frac{\trueCOne p_s^*}{s} 
    = \trueCOne\rePrice + \trueCTwo, \quad
    - \sum_{s=t_1}^{T-1} \frac{y p_s^*}{T} + p_{T}^*
    = \frac{\trueCOne t_1\rePrice}{T} + \trueCTwo~.
\end{align}
Since $p_t\le [0, \priceUB]$ for every $t\in[t_1, T]$, we know that 
\begin{align*}
    \left\|\pvec^*\right\|_{\infty} = \left\|\left(\Amtrx_{[t_1, T]}(\truePara)\right)^{-1}\bvec_{[t_1, T], \rePrice}(\truePara)\right\|_{\infty} \le \priceUB~.
\end{align*}
Moreover,  from \Cref{prop:approx opt of markdown},
we know that the pricing policy 
$\pvec^*$ is a markdown pricing policy. 
Thus, together with the above observations, 
from \eqref{ineq helper3}, we have 
\begin{align*}
    \trueCOne\rePrice + \frac{\priceUB}{2}
    & \overset{(a)}{\le} 
    p_{t_1}^* - \sum_{s=t_1+1}^{T} \frac{\trueCOne p_s^*}{s}
    \overset{(b)}{\le} 
    \priceUB - \sum_{s=t_1+1}^{T} \frac{\trueCOne p_s^*}{s} 
    \overset{(c)}{\le} 
    \priceUB - p_{T}^*\sum_{s=t_1+1}^{T} \frac{\trueCOne }{s}
    \\
    \frac{\trueCOne t_1\rePrice}{T} + \frac{\priceUB}{2}
    & \overset{(d)}{\le} 
    - \sum_{s=t_1}^{T-1} \frac{\trueCOne p_s^*}{T} + p_{T}^*
    \overset{(e)}{\le} 
    - \sum_{s=t_1}^{T-1} \frac{\trueCOne p_{T}^*}{T} + p_{T}^*
    = p_{T}^* \left(1 - \frac{\trueCOne (T-t_1)}{T}\right)~,
\end{align*}
where inequalities (a), (d) hold true by $\trueCTwo = \frac{\baseDemandParaB}{2(\baseDemandParaA + \eta)} \ge \frac{1}{2}\priceUB$, 
inequality (b) holds true $p_{t_1}^* \le \priceUB$, and  
inequalities (c), (e) hold true by $p_t^* \ge p_{T}^*, \forall t\in[t_1, T]$. 
From the above inequalities, we can deduce the following inequality
\begin{align}
    \label{eigen inequ}
    \sum_{s=t_1+1}^{T} \frac{\trueCOne }{s}
    \le \frac{1}{p_{T}^*}\left(
    \frac{\priceUB}{2} - \trueCOne \rePrice
    \right)
    \le 
    \left(1 - \frac{\trueCOne (T-t_1)}{T}\right)\cdot
    \frac{\frac{\priceUB}{2} - \trueCOne \rePrice}{\frac{\priceUB}{2} +  \frac{ \trueCOne t_1\rePrice}{T}}  < 1~.
\end{align}
Thus, matrix 
$-\Hessian_{\val}\left(\pvec\right)$ is a strictly diagonally dominant matrix with positive diagonal entries.
This implies that the Hessian matrix 
$\Hessian_{\val}\left(\pvec\right)$ is strictly negative definite. 
By Gershgorin Circle Theorem, we know that any eigenvalue $\eigenVal_{\val}$ of the Hessian matrix $\Hessian_{\val}\left(\pvec\right)$ must satisfy that 
\begin{align*}
    \eigenVal_{\val} \in 
    \left[-2(\baseDemandParaA + \eta)\left(1 + \trueCOne \sum_{s=t_1+1}^{T} \frac{1}{s}\right), 
    -2(\baseDemandParaA + \eta)\left(1 - \trueCOne \sum_{s=t_1+1}^{T} \frac{1}{s}\right)\right]~.
\end{align*}
which implies that the value function 
$\val^{\pvec}(\rePrice, t_1)$ 
is strongly concave.
\end{proof}

\lipschitzerroronprices*
\begin{proof}[Proof of \Cref{lem:lipschitz error on prices}]
In below proof, we fix a starting time $t_1\in [T]$ and 
the initial reference price $\rePrice$.
Given the policy parameter $\para = (y, z)\in[0, \sfrac{1}{2}) \times (\sfrac{\priceUB}{2}, \infty]$, 
let the matrix $\Amtrx_{[t_1, T]}(\para)$
and 
$\bvec_{[t_1, T], \rePrice}(\para)$ be 
defined as in \eqref{matrix defn}. 
Given two policy parameters $\para_1 = (y_1, z_1)$
and $\para_2 = (y_2, z_2)$,
we can also express the matrix
$\Amtrx_{[t_1, T]}(\para_1) = \Amtrx_{[t_1, T]}(\para_2) + \Delta \Amtrx$
and the vector
$\bvec_{[t_1, T], \rePrice}(\para_1)
= \bvec_{[t_1, T], \rePrice}(\para_2) + \Delta \bvec$ 
as the perturbed matrix of $\Amtrx_{[t_1, T]}(\para_2)$
and the perturbed vector of
$\bvec_{[t_1, T], \rePrice}(\para_2)$ where perturbation matrix 
$\Delta \Amtrx$ and the perturbation vector 
$\Delta \bvec$ depend on the error $\para_1 - \para_2$. 
For notation simplicity, let
$\approxMD(\rePrice, t_1, \para) = \pvec(\para)$.
With the above definitions, we then have 
\begin{align*}
    \left(\Amtrx_{[t_1, T]}(\para_2) + \Delta \Amtrx\right)\left(\pvec(\para_2) + \Delta \pvec \right)
    = \bvec_{[t_1, T], \rePrice}(\para_2) + \Delta \bvec
\end{align*}
where $\pvec(\para_1) = \pvec(\para_2) + \Delta \pvec$.
Expanding the above equality and using the 
exact solution $\Amtrx_{[t_1, T]}(\para_2) \pvec(\para_2) = \bvec_{[t_1, T], \rePrice}(\para_2)$, we then have 
\begin{align*}
    \Delta \pvec = 
    \left(\Amtrx_{[t_1, T]}(\para_2)\right)^{-1} \cdot\left(
    \Delta \bvec - \Delta \Amtrx \pvec(\para_2) -  
    \Delta \Amtrx \Delta \pvec \right)
\end{align*}
Taking the infinite norm on both sides, we then have
\begin{align*}
    \left\|\Delta \pvec\right\|_{\infty}
    & \le 
    \left\|\left(\Amtrx_{[t_1, T]}(\para_2)\right)^{-1}\right\|_{\infty}
    \left\|\Delta \bvec - \Delta \Amtrx \pvec(\para_2) -  
    \Delta \Amtrx \Delta \pvec\right\|_{\infty} \\
    & \le 
    \left\|\left(\Amtrx_{[t_1, T]}(\para_2)\right)^{-1}\right\|_{\infty}
    \left\|\Delta \bvec + \Delta \Amtrx \pvec(\para_2) \right\|_{\infty}
    + 
    \left\|\left(\Amtrx_{[t_1, T]}(\para_2)\right)^{-1}\right\|_{\infty}
    \left\|
    \Delta \Amtrx \right\|_{\infty}
    \left\| \Delta \pvec\right\|_{\infty}
\end{align*}
Rearranging the terms, we have
\begin{align}
    \label{ineq helper7}
    \left(1 - \left\|\left(\Amtrx_{[t_1, T]}(\para_2)\right)^{-1}\right\|_{\infty}
    \left\|
    \Delta \Amtrx \right\|_{\infty}\right) 
    \left\| \Delta \pvec\right\|_{\infty} \le 
    \left\|\left(\Amtrx_{[t_1, T]}(\para_2)\right)^{-1}\right\|_{\infty}
    \left\|\Delta \bvec + \Delta \Amtrx \pvec(\para_2) \right\|_{\infty}
\end{align}
Observe that 
\begin{align*}
    \left\|\Delta \bvec + \Delta \Amtrx \pvec(\para_2) \right\|_{\infty}
    & \le 
    \left\|\Delta \bvec \right\|_{\infty} + \left\|\Delta \Amtrx \pvec(\para_2) \right\|_{\infty} \\
    & \le \priceUB |y_1 - y_2| + |z_1 - z_2| + \priceUB |y_1 - y_2| \sum_{s=t_1+1}^{T} \frac{1}{s} \\
    & = 
    O\left(|z_1 - z_2| + \priceUB |y_1 - y_2| \left(1 + \ln\frac{T}{t_1}\right)\right)
\end{align*}
In below, we provide the upper bound of 
$\left\|\left(\Amtrx_{[t_1, T]}(\para_2)\right)^{-1}\right\|_{\infty}$. 

By the definition of $\pvec(\para_2)$ where
$\Amtrx_{[t_1, T]}(\para_2) 
\pvec(\para_2) = \bvec_{[t_1, T], \rePrice}(\para_2)$,  we have
\begin{align*}
    p_{t_1}(\para_2) - \sum_{s=t_1+1}^{T} \frac{y_2 p_s(\para_2)}{s} 
    = y_2 \rePrice + z_2, \quad
    - \sum_{s=t_1}^{T-1} \frac{y_2 p_s(\para_2)}{T} + p_{T}(\para_2)
    = \frac{t_1\rePrice}{T} + z_2
\end{align*}
which gives us
\begin{align*}
    y_2\rePrice + \frac{\priceUB}{2}
    & \overset{(a)}{\le} 
    p_{t_1}(\para_2) - \sum_{s=t_1+1}^{T} \frac{y_2 p_s(\para_2)}{s}
    \overset{(b)}{\le} 
    \priceUB - \sum_{s=t_1+1}^{T} \frac{y_2 p_s(\para_2)}{s} 
    \overset{(c)}{\le} 
    \priceUB - p_{T}(\para_2)\sum_{s=t_1+1}^{T} \frac{y_2 }{s}
    \\
    \frac{y_2 t_1\rePrice}{T} + \frac{\priceUB}{2}
    & \overset{(d)}{\le} 
    - \sum_{s=t_1}^{T-1} \frac{y_2 p_s(\para_2)}{T} + p_{T}(\para_2)
    \overset{(e)}{\le} 
    - \sum_{s=t_1}^{T-1} \frac{y_2 p_{T}(\para_2)}{T} + p_{T}(\para_2)
    = p_{T}(\para_2) \left(1 - \frac{y_2 (T-t_1)}{T}\right)
\end{align*}
where 
inequalities (a), (d) holds true by $z_2 \ge \frac{1}{2}\priceUB$, 
inequality (b) holds true by assumption that $p_{t_1}(\para_2) \le \priceUB$, 
and inequalities (c), (e) holds true by the observation that $p_t(\para_2) \ge p_{T}(\para_2), \forall t\in[t_1, T]$. 
From the above inequalities, we have
\begin{align}
    \label{eigen inequ 2}
    \sum_{s=t_1+1}^{T} \frac{y_2 }{s}
    \le \frac{1}{p_{T}(\para_2)}\left(
    \frac{\priceUB}{2} - y_2 \rePrice
    \right)
    \le 
    \left(1 - \frac{y_2 (T-t_1)}{T}\right)\cdot
    \frac{\frac{\priceUB}{2} - y_2 \rePrice}{\frac{\priceUB}{2} +  \frac{ y_2 t_1\rePrice}{T}}  < 1~.
\end{align}
The above inequality implies that the matrix 
$\Amtrx_{[t_1, T]}(\para_2)$ is the strictly diagonally dominant matrix. 
Thus, by Neumann Series Theorem,
we know that 
\begin{align*}
    \left(\Amtrx_{[t_1, T]}(\para_2)\right)^{-1}
    = \left(\identity - \Amtrx'_{[t_1, T], \rePrice}(\para_2)\right)^{-1} 
    = \sum_{k=0}^{\infty} \left(\Amtrx'_{[t_1, T]}(\para_2)\right)^k
\end{align*}
where $\identity$ is the identity matrix, 
and the matrix $\Amtrx'_{[t_1, T]}(\para_2)$ has all zero diagonal values and all positive values in all non-diagonal entries. 
With the above observation, we thus have
\begin{align*}
    \pvec(\para_2)
    = \left(\Amtrx_{[t_1, T]}(\para_2)\right)^{-1} \bvec_{[t_1, T], \rePrice}(\para_2)
    = \sum_{k=0}^{\infty} \left(\Amtrx'_{[t_1, T]}(\para_2)\right)^k \bvec_{[t_1, T], \rePrice}(\para_2)
\end{align*}
Notice that every entry in matrix 
$\Amtrx'_{[t_1, T]}(\para_2)$ is non-negative,
and every entry in the vector 
$\bvec_{[t_1, T], \rePrice}(\para_2)$ is no smaller than
$\sfrac{y_2\rePrice t_1}{T}+z_2$.
Thus, we have
\begin{align*}
    \priceUB \ge \left\|\pvec(\para_2)\right\|_{\infty} 
    & = 
    \left\|\left(\Amtrx_{[t_1, T]}(\para_2)\right)^{-1} \bvec_{[t_1, T], \rePrice}(\para_2)\right\|_{\infty} \\
    & = 
    \left\|\sum_{k=0}^{\infty} \left(\Amtrx'_{[t_1, T]}(\para_2)\right)^k \bvec_{[t_1, T], \rePrice}(\para_2)\right\|_{\infty} \\
    & \ge
    \left\|\sum_{k=0}^{\infty} \left(\Amtrx'_{[t_1, T]}(\para_2)\right)^k 
    \right\|_{\infty} 
    \left(\frac{y_2\rePrice t_1}{T} + z_2\right)\\
    & \ge \left\|\sum_{k=0}^{\infty} \left(\Amtrx'_{[t_1, T]}(\para_2)\right)^k 
    \right\|_{\infty} 
    \frac{\priceUB}{2}
\end{align*}
Thus, we can deduce that 
\begin{align*}
    \left\|\left(\Amtrx_{[t_1, T]}(\para_2)\right)^{-1}\right\|_{\infty}  
    = 
    \left\|\sum_{k=0}^{\infty} \left(\Amtrx'_{[t_1, T]}(\para_2)\right)^k 
    \right\|_{\infty}  
    \le 2
\end{align*}
Thus, back to \eqref{ineq helper7}, we know that 
\begin{equation}
    \begin{aligned}
    \label{eq:bound price linfty norm}
    \left\| \Delta \pvec\right\|_{\infty} 
    & \le 
    \frac{1}{\left(1 - \left\|\left(\Amtrx_{[t_1, T]}(\para_2)\right)^{-1}\right\|_{\infty}
    \left\|
    \Delta \Amtrx \right\|_{\infty}\right)}
    \left\|\left(\Amtrx_{[t_1, T]}(\para_2)\right)^{-1}\right\|_{\infty}
    \left\|\Delta \bvec + \Delta \Amtrx \pvec(\para_2) \right\|_{\infty}  \\
    & \le 
    O\left(
    \frac{2}{1 - 2 |y_1-y_2| \ln \frac{T}{t_1}}
    \right) \cdot O\left(|z_1 - z_2|  +  
    \priceUB |y_1 - y_2| \left(1 + \ln\frac{T}{t_1}\right)\right) \\
    & = 
    O\left(|z_1 - z_2|  +  
    \priceUB |y_1 - y_2| \left(1 + \ln\frac{T}{t_1}\right)\right)
    \end{aligned}
\end{equation}
Then from the above characterizations, we know that
\begin{align*}
    \left\|\pvec(\para_2) - \pvec(\para_1)\right\|
    & \le \sqrt{T-t_1+1} \left\|\pvec(\para_2) - \pvec(\para_1)\right\|_{\infty}\\
    & = 
    \sqrt{T-t_1+1} \left\|\Delta\pvec\right\|_{\infty}
    \le 
    O\left(\priceUB \left\|\para_2 - \para_1\right\|\sqrt{T-t_1}\ln \frac{T}{t_1}\right)
\end{align*}
The proof then completes.
\end{proof}

\lipschitzerrorinpolicyspace*
\begin{proof}[Proof of \Cref{lem:lipschitz error in policy space}]
Let $(\rePrice_t^*)_{t\in[t_1, T]}$ (resp.\ $(\rePrice_t)_{t\in[t_1, T]}$) be the reference price sequence under the 
pricing policy $\approxMD(\rePrice, t_1, \truePara)$ (resp.\ $\approxMD(\rePrice, t_1, \para)$). 
For pricing policy $\approxMD(\rePrice, t_1, \truePara)$ (resp.\ $\approxMD(\rePrice, t_1, \para)$),
we use $\criticalTimeOpt \triangleq \criticalTime(\truePara)$ 
(resp.\ $\criticalTimeEst \triangleq \criticalTime(\para)$) to denote the time 
index such that for every $t\in[\criticalTimeOpt, T]$ (resp.\ $t\in[\criticalTimeEst, T]$),
we have $p_t^* < \priceUB$ (resp.\ $p_t< \priceUB$). 
Let $\approxMD(\rePrice, t_1, \truePara) = \pvec(\truePara) = (p_t^*)_{t\in[t_1, T]}$, and 
$\approxMD(\rePrice, t_1, \para) = \pvec(\para) = (p_t(\para))_{t\in[t_1, T]}$.
For notation simplicity given a time $s$, we 
also define $\pvec_{s:T}^* \triangleq 
(p_t^*)_{t\in[s, T]}, \pvec_{s:T} \triangleq 
(p_t(\para))_{t\in[s, T]}$.

In below proof, we bound the Lipschitz error based on 
two possible cases:
(1) $\criticalTimeOpt = \criticalTimeEst$;
(2) $\criticalTimeOpt \neq \criticalTimeEst$.

\xhdr{Case 1 -- When $\criticalTimeOpt = \criticalTimeEst$:}
In this case, the Lipschitz error can be decomposed as follows:
\begin{align*}
    \optVal(\rePrice, t_1)
    -
    \val^{\pvec(\para)}(\rePrice, t_1)  
    \overset{(a)}{=} ~ & 
    \val^{\pvec_{\criticalTimeEst:T}^*}(\rePrice_{\criticalTimeEst}, \criticalTimeEst)~,
    -
    \val^{\pvec_{\criticalTimeEst:T}}(\rePrice_{\criticalTimeEst}, \criticalTimeEst)
\end{align*}
where equality (a) is by the fact that $\rePrice_s^* = \rePrice_s, \forall s\in [t_1, \criticalTimeEst]$ since $p_s^* = p_s = \priceUB, \forall s\in [t_1, \criticalTimeEst]$.
Notice that given the reference price $\rePrice_\criticalTimeEst$
at time $\criticalTimeEst$, the price sequence 
$p_{\criticalTimeEst:T}^*$ is the optimal pricing 
curve for the time window $[\criticalTimeEst,T]$, namely, it equals
to the pricing curve $\approxMD(\rePrice_\criticalTimeEst, \criticalTimeEst, \truePara)$.
Since $p_\criticalTimeEst^* \in [0, \priceUB)$ by definition of $\criticalTimeEst$, 
from \Cref{lem: strongly concavity},
we know that the value function 
$\val^{\pvec}(\rePrice_{\criticalTimeEst}, \criticalTimeEst)$
is strongly concave over the price sequence $\pvec$ 
and its Hessian matrix has bounded eigenvalues. 
Thus, together with \Cref{lem:lipschitz error on prices}, we have
\begin{align*}
    \optVal(\rePrice_{\criticalTimeEst}, \criticalTimeEst) - 
    \val^{\pvec_{\criticalTimeEst:T}}(\rePrice_{\criticalTimeEst}, \criticalTimeEst)
    & \le 
    O\left(\priceUB^2 \left\|\truePara - \para\right\|^2 (T-\criticalTimeEst) \left(\ln \frac{T}{\criticalTimeEst}\right)^2\right)\\
    & \le 
    O\left(\priceUB^2 \left\|\truePara - \para\right\|^2 (T-t_1) \left(\ln \frac{T}{t_1}\right)^2\right)~,
\end{align*}
where in last inequality we have $\criticalTimeEst \ge t_1$.

\xhdr{Case 2a -- When $\criticalTimeOpt < \criticalTimeEst$:}
In this case, by definition, we know that 
$p_t^* = \priceUB$ for all $t\in[t_1, \criticalTimeOpt]$,
and 
$p_t = \priceUB$ for all $t\in[t_1, \criticalTimeEst]$. Thus,
we know that $\rePrice_{\criticalTimeOpt}^* = \rePrice_{\criticalTimeOpt}$. 

Let $\solvedP_t(\para)$ be the first price of the
solution to the linear system 
$\Amtrx_{[t, T]}(\para) 
\pvec = \bvec_{[t, T], \rePrice_t}(\para)$ over the time window $[t, T]$
where $\rePrice_t = \frac{t_1\rePrice + (t-t_1)\priceUB}{t}$.
Recall that $\pvec_{\criticalTimeOpt:T}^* = (p_t^*)_{t\in [\criticalTimeOpt, T]}$
is the solution the linear system 
$\Amtrx_{[\criticalTimeOpt, T]}(\truePara) 
\pvec = \bvec_{[\criticalTimeOpt, T], \rePrice_{\criticalTimeOpt}}(\truePara)$.
Thus, follow the similar analysis in the proof of \Cref{lem:lipschitz error on prices}, i.e., \eqref{eq:bound price linfty norm}, we also have 
\begin{align*}
    \left|p_{\criticalTimeOpt}^*- \solvedP_{\criticalTimeOpt}(\para)\right|
    \le  O\left(\estCTwoErr   +  
    \priceUB \estCOneErr  \left(1 + \ln\frac{T}{\criticalTimeOpt}\right)\right)~,
\end{align*}
where $\estCOneErr$ (resp.\ $\estCTwoErr$) is the estimation error for $\trueCOne$
(resp.\ $\trueCTwo$).
When $\solvedP_{\criticalTimeOpt}(\para) < 0$, 
then from the above inequality, we know that 
\begin{align*}
    p^*_{\criticalTimeOpt} \le 
    O\left(\estCTwoErr   +  
    \priceUB \estCOneErr  \left(1 + \ln\frac{T}{\criticalTimeOpt}\right)\right)
\end{align*}
From \Cref{prop:approx opt of markdown}, we know that 
$\left\|\pvec_{[\criticalTimeOpt, T]}^*\right\|_\infty \le p^*_{\criticalTimeOpt}$. 
Consider a new pricing policy $\pvec^\ddag = (p^\ddag_t)_{t\in[\criticalTimeOpt, T]}$ where  $p^\ddag_t \equiv 0$ for all $t\in[\criticalTimeOpt, T]$. 
Then, 
\begin{align*}
    \optVal(\rePrice, t_1)
    -
    \val^{\pvec(\para)}(\rePrice, t_1)
    & = 
    \optVal(\rePrice_{\criticalTimeOpt}^*, \criticalTimeOpt)
    -
    \val^{\pvec_{\criticalTimeOpt:T}}(\rePrice_{\criticalTimeOpt}^*, \criticalTimeOpt)
    \\
    & \le  
    \optVal(\rePrice_{\criticalTimeOpt}^*, \criticalTimeOpt) \\
    & \overset{(a)}{=} 
    \optVal(\rePrice_{\criticalTimeOpt}^*, \criticalTimeOpt) -
    \val^{\pvec^\ddag}(\rePrice_{\criticalTimeOpt}^*, \criticalTimeOpt)
    \\
    & \overset{(b)}{\le}  O\left(\priceUB^2 \left\|\truePara - \para\right\|^2 (T-\criticalTimeOpt) \left(\ln \frac{T}{\criticalTimeOpt}\right)^2\right)
\end{align*}
where in equality (a) we have $\val^{\pvec^\ddag}(\rePrice_{\criticalTimeOpt}^*, \criticalTimeOpt) = 0$, and inequality (b)
is due to $\left\|\pvec_{\criticalTimeOpt:T}^* - \pvec^\ddag\right\|_\infty \le p^*_{\criticalTimeOpt}$
and the Hessian matrix of the value function 
$\val^{\pvec}(\rePrice_{\criticalTimeOpt}^*, \criticalTimeOpt)$
has bounded eigenvalues \Cref{lem: strongly concavity} and the results in \Cref{lem:lipschitz error on prices}.

When $\solvedP_{\criticalTimeOpt}(\para) > \priceUB$, then follow the similar analysis in the proof of \Cref{lem:lipschitz error on prices} for any $t\in[\criticalTimeOpt, T]$,
\begin{align*}
    \left|p_t^* - \solvedP_t(\para)\right|
    & \le O\left(\estCTwoErr + \trueCOne (\rePrice_t - \rePrice_t^*)+\estCOneErr \rePrice_t + \priceUB \estCOneErr \ln \frac{T}{t}\right) \\
    & = 
    O\left(\rePrice_t - \rePrice_t^*+ 
    \priceUB \maxEstError \ln \frac{T}{t}\right) \\
    & \overset{(a)}{=}
    O\left(\frac{(t-1) (\rePrice_{t-1} - \rePrice_{t-1}^*) + \min\{\priceUB, \solvedP_{t-1}\} - p_{t-1}^*}{t}+ 
    \priceUB \maxEstError \ln \frac{T}{t}\right) \\
    & \overset{(b)}{\le}
    O\left(\frac{(t-1) (\rePrice_{t-1} - \rePrice_{t-1}^*)}{t}+ \frac{\rePrice_{t-1} - \rePrice_{t-1}^* + \priceUB \maxEstError \ln \frac{T}{t-1}}{t} + 
    \priceUB \maxEstError \ln \frac{T}{t}\right)\\
    & = 
    O\left(\rePrice_{t-1} - \rePrice_{t-1}^*+ \frac{1}{t}\priceUB \maxEstError \ln \frac{T}{t-1} + 
    \priceUB \maxEstError \ln \frac{T}{t}\right) \\
    & \le 
    O\left(\priceUB \maxEstError \ln \frac{T}{\criticalTimeOpt} \ln \frac{t}{\criticalTimeOpt}\right)
\end{align*}
where equality (a) is due to the definition of $\rePrice_t$, inequality (b) is due to the fact that
$|\priceUB - p_{t-1}^*|\le |\solvedP_{t-1}(\para) - p_{t-1}^*|$. 
Thus, we know that $\left\|\pvec(\para) - \pvec^*\right\|_{\infty} \le O\left(\priceUB \maxEstError \ln \frac{T}{t_1}\right)$. 
Consequently, we can bound the Lipschitz error 
$\optVal(\rePrice, t_1) 
-
\val^{\approxMD(\rePrice, t_1, \para)}(\rePrice, t_1) 
= O\left(\priceUB^2 \left\|\truePara - \para\right\|^2 (T-\criticalTimeOpt) \left(\ln \frac{T}{\criticalTimeOpt}\right)^2\right)$ 
with using  \Cref{lem: strongly concavity} for 
value function 
$\val^{\pvec}(\rePrice_{\criticalTimeOpt}^*, \criticalTimeOpt)$, 
and the results in \Cref{lem:lipschitz error on prices}.

\xhdr{Case 2b -- When $\criticalTimeOpt > \criticalTimeEst$:}
In this case, 
we also define $\solvedP_t(\truePara)$ as the first price of the
solution to the linear system 
$\Amtrx_{[t, T]}(\truePara) 
\pvec = \bvec_{[t, T], \rePrice_t^*}(\truePara)$ over the time window $[t, T]$
where $\rePrice_t^* = \frac{t_1\rePrice + (t-t_1)\priceUB}{t}$.
As we have $\rePrice_{\criticalTimeEst} = \rePrice_{\criticalTimeEst}^*$,
\begin{align*}
    \left|\solvedP_\criticalTimeEst(\truePara)- p_\criticalTimeEst(\para)\right|
    & \le  
    O\left(
    \frac{2}{1 - 2 \estCOneErr \ln \frac{T}{\criticalTimeEst}}
    \right) \cdot O\left(\estCTwoErr  +  
    \priceUB \estCOneErr \left(1 + \ln\frac{T}{\criticalTimeEst}\right)\right) \\ 
    & = 
    O\left(\estCTwoErr   +  
    \priceUB \estCOneErr  \left(1 + \ln\frac{T}{\criticalTimeEst}\right)\right) 
    = O\left(\priceUB\maxEstError \ln \frac{T}{\criticalTimeEst}\right)~,
\end{align*}
and similarly, we also have for any $t\in[\criticalTimeEst, \criticalTimeOpt-1]$
\begin{align*}
    \left|\solvedP_t(\truePara)- p_t(\para)\right|
    \le  O\left(\estCTwoErr +  y\left(\rePrice_t^* - \rePrice_t\right) +  
    \priceUB \estCOneErr  \left(1 + \ln\frac{T}{\criticalTimeEst}\right)\right) 
    = O\left(\priceUB\maxEstError \left(\ln \frac{T}{\criticalTimeEst}\right)^2\right)~.
\end{align*}
We can roll out the value of $p_t(\para)$ for $t\in[\criticalTimeEst, \criticalTimeOpt-1]$
\begin{align*}
    p_t(\para) = p_{t-1}(\para) - \frac{y\rePrice_{t-1}}{t+y} 
    & = p_{\criticalTimeEst}(\para) - \sum_{s=\criticalTimeEst+1}^{t}\frac{y\rePrice_{s-1}}{s+y} \\
    & = 
    p_{\criticalTimeEst}(\para) - y \sum_{s=\criticalTimeEst+1}^{t}
    \frac{1}{s+y}\frac{\criticalTimeEst \rePrice_{\criticalTimeEst} + \sum_{i = \criticalTimeEst}^{s-1} p_i(\para)}{s-1}\\
    & \le 
    p_{\criticalTimeEst}(\para) - y \sum_{s=\criticalTimeEst+1}^{t}
    \frac{1}{s+y}\frac{\criticalTimeEst \rePrice_{\criticalTimeEst}}{s-1}\\
    & \le
    p_{\criticalTimeEst}(\para) - y \criticalTimeEst \rePrice_{\criticalTimeEst} \left(\frac{1}{\criticalTimeEst} - \frac{1}{t+1}\right)
    \le
    p_{\criticalTimeEst}(\para) - y  \rePrice_{\criticalTimeEst} \left(1 - \frac{\criticalTimeEst}{t}\right)~.
\end{align*}
Notice that at time $t = \criticalTimeOpt - 1$, we have
\begin{align*}
    \priceUB < \solvedP_{\criticalTimeOpt-1}(\truePara)
    \le p_{\criticalTimeOpt-1}(\para) 
    + O\left(\priceUB \maxEstError \left(\ln \frac{T}{\criticalTimeEst}\right)^2\right) 
    & \le 
    p_{\criticalTimeEst}(\para) - y  \rePrice_{\criticalTimeEst} \left(1 - \frac{\criticalTimeEst}{\criticalTimeOpt-1}\right)
    + O\left(\priceUB \maxEstError \left(\ln \frac{T}{\criticalTimeEst}\right)^2\right) \\
    & \le \priceUB - y  \rePrice_{\criticalTimeEst} \left(1 - \frac{\criticalTimeEst}{\criticalTimeOpt-1}\right)
    + O\left(\priceUB \maxEstError \left(\ln \frac{T}{\criticalTimeEst}\right)^2\right)~.
\end{align*}
Thus, we can deduce the following condition 
on the time round $\criticalTimeOpt$:
\begin{align}
    \label{ineq:condi on criticalTimeOpt}
    \criticalTimeOpt 
    \le O\left(\frac{y\rePrice_\criticalTimeEst\criticalTimeEst}{y\rePrice_{\criticalTimeEst} - \priceUB \maxEstError \left(\ln \frac{T}{\criticalTimeEst}\right)^2}\right)
    \le 
    O\left(\frac{\priceUB\criticalTimeEst}{\priceUB - 2\priceUB \maxEstError \left(\ln \frac{T}{\criticalTimeEst}\right)^2}\right)
    = 
    O\left(\frac{\criticalTimeEst}{1 - 2 \maxEstError \left(\ln \frac{T}{\criticalTimeEst}\right)^2}\right)
\end{align}
Now we can bound the Lipschtiz error as follows:
\begin{align*}
    & \optVal(\rePrice, t_1)
    -
    \val^{\pvec(\para)}(\rePrice, t_1)\\
    \le ~ & 
    O\left(\revUB (\criticalTimeOpt - \criticalTimeEst)\right) + 
    \optVal(\rePrice_{\criticalTimeOpt}^*, \criticalTimeOpt)
    -
    \val^{\pvec_{\criticalTimeOpt:T}}(\rePrice_{\criticalTimeOpt}, T) \\
    \overset{(a)}{\le} ~ & 
    O\left(\revUB (\criticalTimeOpt - \criticalTimeEst)\right) + 
    \optVal(\rePrice_{\criticalTimeOpt}, \criticalTimeOpt)
    -
    \val^{\pvec_{\criticalTimeOpt:T}}(\rePrice_{\criticalTimeOpt}, T)
    + O\left(\eta \priceUB \criticalTimeOpt(\rePrice_{\criticalTimeOpt}^* -\rePrice_{\criticalTimeOpt})\ln (T - \criticalTimeOpt)\right) \\
    \overset{(b)}{\le} ~ & 
    O\left(\revUB (\criticalTimeOpt - \criticalTimeEst)\right) + 
    O\left(\priceUB^2 \left\|\truePara - \para\right\|^2 (T-\criticalTimeOpt) \left(\ln \frac{T}{\criticalTimeOpt}\right)^2\right)
    + O\left(\eta \priceUB \criticalTimeOpt \cdot \frac{\priceUB(\criticalTimeOpt - \criticalTimeEst)}{\criticalTimeOpt} \cdot \ln (T - \criticalTimeOpt)\right) \\
    \overset{(c)}{\le} ~ & 
    O\left(\revUB (\criticalTimeOpt - \criticalTimeEst)\right) + 
    O\left(\priceUB^2 \left\|\truePara - \para\right\|^2 (T-\criticalTimeOpt) \left(\ln \frac{T}{\criticalTimeOpt}\right)^2\right)
    + O\left( \priceUB^2(\criticalTimeOpt - \criticalTimeEst) \cdot \ln (T - \criticalTimeOpt)\right) \\
    = ~ & 
    O\left(\priceUB^2 \left\|\truePara - \para\right\|^2 (T-t_1) \left(\ln \frac{T}{t_1}\right)^2\right)
\end{align*}
where inequality (a) is due to \Cref{lem:opt rev diff w.r.t diff ref} and 
$(p_{t, \rePrice_{\criticalTimeOpt}}^*)_{t\in[\criticalTimeOpt, T]}$ is the 
optimal pricing policy for the time window $[\criticalTimeOpt, T]$
given the initial reference price $\rePrice_{\criticalTimeOpt}$, 
inequality (b) follows similarly as in the case 1 where 
the value function $\val^{\pvec}(\rePrice_{\criticalTimeOpt}, \criticalTimeOpt)$ is a strongly-concave function, and 
the fact that $\rePrice_{\criticalTimeOpt}^* - \rePrice_{\criticalTimeOpt} \le \frac{\priceUB(\criticalTimeOpt - \criticalTimeEst)}{\criticalTimeOpt}$, 
and inequality (c) is to due to the upper bound of the 
time round $\criticalTimeOpt$ established in \eqref{ineq:condi on criticalTimeOpt} and we thus have $(\criticalTimeOpt - \criticalTimeEst)\ln (T - \criticalTimeOpt)= O(1)$.
\end{proof}